%% file: mtl_ncp.tex
\documentclass{article} 
\usepackage{iclr2027_conference,times}

\input{my_commands.tex}
\usepackage{url}
\usepackage{booktabs}
\usepackage{multirow}
\usepackage{xcolor}
\usepackage{placeins} 
\usepackage{algorithm}
\usepackage{algpseudocode}
\usepackage{wrapfig}

\newcommand{\method}{MTL-CMO}

\title{Multi-Task Learning of Conditional Mean\\ Operators: applications to dynamical systems and uncertainty quantification}

\author{Sami Chemlal$^*$ \\
CMAP, Ecole Polytechnique \\
\texttt{samychemlal@yahoo.fr}
\And
Thibaut Germain$^*$ \\
CMAP, Ecole Polytechnique \\
\texttt{thibaut.germain@polytechnique.edu}
\AND
Rémi Flamary \\
CMAP, Ecole Polytechnique \\
\texttt{remi.flamary@polytechnique.edu}
\And
Vladimir R. Kostic \\
University of Novi Sad \\
\texttt{vkostic@dmi.uns.ac.rs}
\AND
Karim Lounici \\
CMAP, Ecole Polytechnique \\
\texttt{karim.lounici@polytechnique.edu}
}

\DeclareMathOperator{\cmo}{\msE_{Y \mid X}}
\DeclareMathOperator{\cmor}{\msE^r_{Y \mid X}}
\DeclareMathOperator{\cmot}{\msE^\theta_{Y \mid X}}
\DeclareMathOperator{\cmok}{\msE^{(k)}_{Y \mid X}}

\DeclareMathOperator{\cmon}{\msE^{\text{new}}_{Y \mid X}}
\DeclareMathOperator{\dmo}{\msD_{Y \mid X}}

\DeclareMathOperator{\dmot}{\msD^{\theta}_{Y \mid X}}
\DeclareMathOperator{\dmok}{\msD^{(k)}_{Y \mid X}}
\DeclareMathOperator{\dmokr}{\msD^{(k),r_k}_{Y \mid X}}
\DeclareMathOperator{\dmokt}{\msD^{(k),\theta}_{Y \mid X}}
\DeclareMathOperator{\odmokt}{\tilde{\msD}^{(k),\theta}_{Y \mid X}}
\DeclareMathOperator{\dmokth}{\msD^{(k),\widehat{\theta}}_{Y \mid X}}
\DeclareMathOperator{\dmon}{\msD^{\text{new}}_{Y \mid X}}
\DeclareMathOperator{\px}{\mu}
\DeclareMathOperator{\py}{\nu}
\DeclareMathOperator{\pxy}{\rho}
\DeclareMathOperator{\lx}{L_\mu^2(\mcX)}
\DeclareMathOperator{\lxk}{L_{\mu_k}^2(\mcX)}
\DeclareMathOperator{\ly}{L_\nu^2(\mcY)}
\DeclareMathOperator{\lyk}{L_{\nu_k}^2(\mcY)}
\DeclareMathOperator{\rd}{\mathrm{d}}

\iclrfinalcopy 
\begin{document}

\maketitle
\def\thefootnote{*}\footnotetext{Authors with equal contributions.}

\begin{abstract}
    Estimating conditional statistics and learning representations of a population of conditional distributions are central problems in many data-driven applications, including uncertainty quantification and dynamical systems analysis. Conditional mean operators (CMOs), a class of linear operators between function spaces, resolve these objectives by providing access to a broad class of conditional statistics. However, existing methods typically estimate each CMO independently or constrain it to prespecified function spaces, thereby preventing the exploitation of shared structure across related distributions. \\
    In this work, we posit that related CMOs share finite-dimensional input and output function spaces, and are specialized for each task with a linear operator mapping these spaces. Based on this hypothesis, we introduce MTL-CMO, a multi-task framework that jointly learns shared function spaces and task-specific operators across multiple datasets. We further introduce T-CMO, a transfer learning method that reuses the shared spaces to estimate, in closed form, the operator of a new conditional distribution. We establish statistical guarantees quantifying the benefits of jointly learning the shared function spaces. Our experiments demonstrate that learning shared function spaces improves uncertainty quantification across a broad range of conditional distributions and, when applied to Langevin and plasma dynamics, yields compact representations of complex dynamics that retain physically meaningful information and enable parameter identification.

\end{abstract}

\input{sections/introduction}
\input{sections/background}
\input{sections/method}
\input{sections/numerical_experiment}
\input{sections/conclusion}

\section*{Acknowledgments}
This project received funding from the European Union’s Horizon Europe research and innovation program under grant agreement 101120237 (ELIAS), Fondation de l’Ecole Polytechnique, Hi! PARIS, the French National Research Agency (ANR) through France 2030 program (ANR-23-IACL-0005 and ANR-25-PEIA-0005), NextGenerationEU and Provincial Secretariat for Higher Education and Scientific Research of Vojvodina (Research Grant 003870560 2025 09418 003 000 000 001 04 004).

\bibliography{biblio}
\bibliographystyle{iclr2027_conference}

\appendix
\input{sections/appendix}

\end{document}

%% file: my_commands.tex
\usepackage{mathtools}
\usepackage{amsfonts}
\usepackage{amssymb}
\usepackage{amsthm}
\usepackage{bbm}
\usepackage{bm}
\usepackage{hyperref}
\usepackage{cleveref}

\newcommand{\kl}[1]{{\color{blue}{#1}}}

\providecommand{\Phic}{\Phi^{(k)}_{\theta,c}}
\providecommand{\Psic}{\Psi^{(k)}_{\theta,c}}
\providecommand{\Lk}{L^2(\mu_k\otimes\nu_k)}

\DeclareMathOperator{\spn}{span}

\DeclareMathOperator{\Id}{Id}

\DeclareMathOperator{\Tr}{Tr} 
\DeclareMathOperator{\Diag}{Diag}

\DeclarePairedDelimiterX{\innerp}[2]{\langle}{\rangle}{#1,#2}
\DeclarePairedDelimiterX{\abs}[1]{\lvert}{\rvert}{\ifblank{#1}{\:\cdot\:}{#1}}
\DeclarePairedDelimiterX{\basenorm}[1]{\lVert}{\rVert}{\ifblank{#1}{\:\cdot\:}{#1}}
\makeatletter
\newcommand{\norm}{\@ifstar{\@norm}{\@anorm}}
\newcommand{\@norm}[2][2]{\basenorm{#2}_{#1}}
\newcommand{\@anorm}[2][2]{\basenorm*{#2}_{#1}}
\makeatother

\newcommand{\bs}[1]{\boldsymbol{#1}}

\def\mcA{\mathcal{A}}
\def\mcB{\mathcal{B}}

\def\mcE{\mathcal{E}}
\def\mcF{\mathcal{F}}
\def\mcG{\mathcal{G}}
\def\mcH{\mathcal{H}}

\def\mcJ{\mathcal{J}}

\def\mcL{\mathcal{L}}

\def\mcN{\mathcal{N}}
\def\mcO{\mathcal{O}}

\def\mcR{\mathcal{R}}

\def\mcU{\mathcal{U}}

\def\mcW{\mathcal{W}}
\def\mcX{\mathcal{X}}
\def\mcY{\mathcal{Y}}
\def\mcZ{\mathcal{Z}}

\def\mb#1{\mathbf{#1}}
\def\mbA{\mathbf{A}}
\def\mbB{\mathbf{B}}
\def\mbC{\mathbf{C}}
\def\mbD{\mathbf{D}}
\def\mbE{\mathbf{E}}
\def\mbF{\mathbf{F}}
\def\mbG{\mathbf{G}}
\def\mbH{\mathbf{H}}
\def\mbI{\mathbf{I}}

\def\mbM{\mathbf{M}}
\def\mbN{\mathbf{N}}

\def\mbQ{\mathbf{Q}}
\def\mbR{\mathbf{R}}

\def\mbT{\mathbf{T}}
\def\mbU{\mathbf{U}}
\def\mbV{\mathbf{V}}
\def\mbW{\mathbf{W}}
\def\mbX{\mathbf{X}}
\def\mbY{\mathbf{Y}}
\def\mbZ{\mathbf{Z}}

\def\mba{\mathbf{a}}
\def\mbb{\mathbf{b}}
\def\mbc{\mathbf{c}}

\def\mbm{\mathbf{m}}

\def\mbr{\mathbf{r}}

\def\mbu{\mathbf{u}}
\def\mbv{\mathbf{v}}
\def\mbw{\mathbf{w}}
\def\mbx{\mathbf{x}}
\def\mby{\mathbf{y}}
\def\mbz{\mathbf{z}}

\def\msD{\mathsf{D}}
\def\msE{\mathsf{E}}

\def\msS{\mathsf{S}}

\def\mdE{\mathbb{E}}

\def\mdP{\mathbb{P}}

\def\mdR{\mathbb{R}}
\def\mdS{\mathbb{S}}

\newtheorem{thm}{Theorem}
\crefname{thm}{theorem}{theorems}
\Crefname{thm}{Theorem}{Theorems}
\newtheorem{prp}{Proposition}
\crefname{prp}{proposition}{propositions}
\Crefname{prp}{Proposition}{Propositions}
\newtheorem{rmk}{Remark}
\crefname{rmk}{remark}{remarks}
\Crefname{rmk}{Remark}{Remarks}

\crefname{dfn}{definition}{definitions}
\Crefname{dfn}{Definition}{Definitions}

\crefname{crr}{corollary}{corollaries}
\Crefname{crr}{Corollary}{Corollaries}
\newtheorem{lmm}{Lemma}
\crefname{lmm}{lemma}{lemmas}
\Crefname{lmm}{Lemma}{Lemmas}

\crefname{examp}{example}{examples}
\Crefname{examp}{Example}{Examples}
\newtheorem{assump}{Assumption}
\crefname{assump}{assumption}{assumptions}
\Crefname{assump}{Assumption}{Assumptions}

%% file: sections/introduction.tex
\section{Introduction}

Characterizing the distribution of a target variable $Y$ given a conditioning variable $X$ is key to many statistical learning problem. Beyond solely predicting the conditional expectation, access to the conditional distribution allows the computation of a broad range of conditional statistics, including moments, quantiles, or confidence regions, typically required for uncertainty quantification~\citep{ghanem2017handbook}. Similarly, for stochastic dynamical systems, the conditional distributions of future states given the current state characterizes the dynamic and provides access to physically meaningful quantities like transition probabilities, relaxation times, or invariant measures~\citep{pavliotis2014stochastic}. \\
The conditional mean operator (CMO) provides a natural operator-theoretic representation of a conditional distribution by mapping an observable $f$ of $Y$ to its conditional expectation:
\begin{equation*}
\textstyle
\left[\cmo f\right]\!(x)
= \mdE\left[f(Y)\mid X=x\right].
\end{equation*}
By varying $f$, a single CMO provides access to a broad range of conditional statistics, making it a versatile representation of conditional distributions. For dynamical systems, CMOs are particularly informative: their spectral decompositions reveal coherent dynamical modes associated with characteristic timescales and frequencies. Existing approaches to CMO estimation primarily rely on kernel methods with either fixed kernels~\citep{song2009hilbert,muandet2017kernel} or learned feature representations~\citep{shimizu2024neural,kostic2024ncp}. However, accurately estimating these operators generally requires large datasets, making their estimation challenging when only limited observations are available~\citep{hertel2026verifiable}.\\
However, in many applications, data are collected from a population of related conditional distributions arising, for instance, from different physical parameters, experimental conditions, or subjects. Although the conditional distributions vary across the population, their associated CMOs may share common input and output function spaces, while differing through the linear transformations acting between these spaces. This  structure naturally motivates a multi-task approach~\citep{caruana1997multitask}: the shared function spaces can be learned jointly across tasks, while retaining a task-specific operator for each conditional distribution. Sharing information in this manner can improve the estimation of individual CMOs~\citep{maurer2016benefit}. Moreover, the learned function spaces can be transferred to a new task, for which only the corresponding low-dimensional operator must be estimated, thereby improving data efficiency~\citep{tripuraneni2020theory}. Nevertheless, existing CMO estimators predominantly treat conditional distributions independently, while existing multi-task approaches to conditional modeling are generally designed for specific applications and do not provide a unified operator-learning framework applicable to both uncertainty quantification and dynamical systems.

\noindent
\textbf{Contributions.} In this work, \textbf{(1)} we introduce MTL-CMO, a multi-task framework that jointly learns a population of CMOs. Each task is represented by a task-specific linear operator acting between shared, learnable input/output function spaces. \textbf{(2)} We subsequently introduce T-CMO, a transfer-learning estimator that reuses the learned function spaces to estimate the CMO of a new task in closed-form. \textbf{(3)} We establish statistical guarantees quantifying the benefits of jointly learning the shared input/output function spaces for CMO estimation. \textbf{(4)} We evaluate MTL-CMO and T-CMO on uncertainty-quantification and dynamical-systems benchmarks in both multi-task and transfer-learning settings. Across these experiments, our methods are competitive or outperform state-of-the-art baselines, including methods specifically designed for the corresponding applications.

\noindent
\textbf{Paper organization.}~\Cref{sec:background} introduces the necessary background. \Cref{sec:method} presents the MTL-CMO and T-CMO frameworks, establishes their statistical guarantees, and positions them with respect to related work. \Cref{sec:experiments} presents the numerical experiments, before we conclude the paper.

%% file: sections/background.tex
\section{Background on cmo: definition, estimation and applications}
\label{sec:background}
\noindent
\textbf{Conditional Mean Operator (CMO)~\citep{fukumizu2004dimensionality}.}
Consider a pair of random variables $(X,Y)$ on the product space $\mcX \times \mcY$ with joint distribution $\rho$ and marginals $\mu$ and $\nu$ respectively. We assume that $\rho$ is absolutely continuous w.r.t. the product of marginals, i.e. $\rho \ll \mu \otimes \nu$, with density function square integrable, i.e. $\mathrm{d}\pxy/\mathrm{d}(\px \otimes\py) \in L^2_{\px \otimes \py}(\mcX \times \mcY)$. Then, the conditional mean operator $\cmo: \ly \to \lx$ is well defined and verifies for any $f \in \ly$ and $x \in \mcX$: 
\begin{equation}
    [\cmo f](x) = \mdE[f(Y)\mid X=x] = \int_\mcY f(y)p(x,y)\py(dy),
\end{equation}
where $\lx$ and $\ly$ are separable Hilbert space. 
Since $\cmo$ is an Hilbert-Schmidt operator, and $\|\cmo\|_{\mathrm{op}} = 1$ with $\cmo\mb1_\mcY = \mb1_\mcX$, it admits a singular value decomposition (SVD) such that:
\begin{equation}
    \textstyle
    \cmo = \sum_{i \geq 0} \sigma_i u_i \otimes v_i, 
    \qquad
    (\sigma_0,u_0,v_0) = (1,\mb1_\mcX, \mb1_\mcY)
\end{equation}
where $\{u_i\}_{i \geq 0}$ and $\{v_i\}_{i \geq 0}$ are singular functions forming orthonormal basis, $\{\sigma_i\}_{i\geq 0}$ the decreasing sequence of positive singular values, and $\mb1_\mcX$ an indicator function. Hence, the conditional mean operator can be approximated arbitrary well, and by the Eckart-Young-Mirsky Theorem the best rank-$r$ approximation is the truncation $\textstyle \cmor = \sum_{i =0}^{r-1} \sigma_i u_i \otimes v_i$ with approximation error $\textstyle \|\cmo - \cmor \|_{\mathrm{op}} = \sigma_r$.
It follows that for any $f \in \ly$ and $x \in \mcX$, the conditional expectation can be approximated with: 
\begin{equation}
\label{eq:cs_approx}
\textstyle
    \left[\cmo f\right]\!(x) \approx 
\left[\cmor f\right]\!(x)
=
\mdE_\nu[f(Y)]
+
\sum_{i=1}^{r-1}
\sigma_i u_i(x)\,
\mdE_\nu[v_i(Y)f(Y)],
\end{equation}

\noindent
\textbf{Operator estimation with Neural Conditional Probability.}
The Neural Conditional Probability framework (NCP) learns a finite-rank approximation of a conditional mean operator from paired observations~\citep{kostic2024ncp}. It first introduce the deflated operator $\dmo = \cmo - \mb1_\mcX \otimes \mb1_\mcY$ to remove the trivial singular component. For a latent dimension $r$, NCP parametrizes with neural networks the approximation of $\cmo$ through its truncated SVD
\begin{equation}
    \textstyle
    \cmot = \mb1_\mcX \otimes \mb1_\mcY + \dmot 
    = \mb1_\mcX \otimes \mb1_\mcY + \sum_{i =1}^{r} \sigma_i^\theta u_i^\theta \otimes v_i^\theta,
\end{equation}
with nontrivial singular values $\textstyle \sigma^\theta \in [0,1)^{r}$, and singular functions $u^\theta(x) \triangleq [u_1^\theta(x),\ldots,u_r^\theta(x)]^\top$ and $v^\theta(y) \triangleq [v_1^\theta(y),\ldots,v_r^\theta(y)]^\top$.
Denoting $\textstyle q_\theta(x,y) = \sum_{i =1}^{r} \sigma_i^\theta u_i^\theta(x)v_i^\theta(y)$, the deflated operator is estimated by minimizing a regularized problem with a data-fitting term measuring the discrepancy between the estimated and true deflated operators 
\begin{equation}
    \begin{aligned}
    \mcL(\theta) & = \|\dmot - \dmo\|_{\mathrm{HS}}^2
        - \|\dmo\|_{\mathrm{HS}}^2 \\
    & = \mdE_{\mu \otimes \nu}[q^2_\theta(X,Y)]-2(\mdE_{\rho}[q_\theta(X,Y)]-\mdE_{\mu \otimes \nu}[q_\theta(X,Y)]),   
    \end{aligned}
\end{equation}
and a regularization encouraging the functions to be centered and orthonormal in their $L^2$ spaces.

\noindent
\textbf{Application: conditional uncertainty quantification.} 
Beyond conditional expectations, the CMO provides access to conditional probabilities. Indeed, for any measurable set $B\subseteq\mcY$, the indicator function $\mb1_B \in \ly$, and the conditional probability can be evaluated with
\begin{equation}
    \textstyle
    \mdP(Y \in B \mid X=x) \approx \widehat{\mdP}(Y \in B \mid X=x) = \left[\widehat{\cmot} \mb1_B\right]\!(x).
\end{equation}
Consequently, an estimated CMO can be used to characterize the uncertainty of $Y$ conditionally on $X{=}x$, for instance through a $(1-\alpha)$ conditional confidence region $C_\alpha(x)\subseteq\mcY$ satisfying
$
\textstyle
\left[\cmo \mb1_{C_\alpha(x)}\right]\!(x)
=
\mathbb P(Y\in C_\alpha(x)\mid X=x)
\geq 1-\alpha.
$
For scalar $Y$, taking $f_t(y)=\mathbf 1_{\{y\leq t\}}$ yields
$\widehat F(t\mid x)=\widehat{\mathbb P}(Y\leq t\mid X=x)$ with conditional quantiles verifying
$Q_\alpha(x)=\inf\{t \mid \widehat F(t\mid x)\geq\alpha\}$. 

\noindent
\textbf{Application: representation of dynamical systems.}
Consider a dynamical system that is a time-homogeneous Markov process $(X_t)_{t\geq0}$ with invariant measure $\pi$. For any lag $t>0$, the transition from $X_0$ to $X_t$ is characterized by the transition operator $[A_t f](x) = \mdE[f(X_t) \mid X_0 = x]$ with $f \in L^2_\pi(\mcX)$. Hence, each $T_t$ is precisely the CMO of the pair $(X_0,X_t)$ and describes how observables of the future state depend on the initial state. Under suitable assumptions on the process, the family $\{T_t\}_{t>0}$ forms a semi-group admitting an infinitesimal generator $G = \lim_{t\to0^+}(A_tf-f)/t$ with a discrete spectrum such that for any $t>0$
\begin{equation}
    \textstyle
    [A_t f](x) = [\exp(tG)f](x) = \sum_{j \geq 0} e^{t\lambda_j}\langle f,g_j\rangle f_j(x),
\end{equation}
where $\lambda_j$ and $(f_j,g_j)$ denote the eigenvalues and corresponding right and left eigenfunctions of $G$. In particular, the transition operators $\{A_t\}_{t>0}$ share the spectral modes of the infinitesimal generator, while their eigenvalues evolve with the lag as $e^{t\lambda_j}$. Estimating these CMOs from sampled trajectories consequently provides access to the spectral structure of the underlying generator~\citep{koopman1931hamiltonian,singer1976representation,schmid2010dynamic}. The generator representation is central to data-driven dynamical systems analysis, as its eigenfunctions and eigenvalues encode characteristic spatial structures and temporal scales of the dynamics. Since errors in the estimated transition operators propagate to the generator, accurate CMO estimation is essential for reliably recovering the underlying dynamics. 
For interested readers, we provide further details on operator-theoretic representation of dynamical systems and their estimation in Appendix~\ref{app:op_ds}. 

%% file: sections/method.tex
\section{Multi-task and Transfer learning for cmo}
\label{sec:method}

\subsection{Multi-task learning of conditional mean operators} 
\label{sec:mtl-cmo}

\noindent
\textbf{MTL strategy \& assumptions.} 
We propose a multi-task approach to estimate CMOs from observations of related joint distributions. The approach learns input and output function spaces shared across tasks, while each task retains its own operator representation within those spaces. \\
Formally, we consider $K$ pairs of random variables $\textstyle (X^{(k)},Y^{(k)})\sim\rho_k$, with marginals $\mu_k$ and $\nu_k$. We denote by $\textstyle \cmok:  L^2_{\nu_k}(\mcY) \to L^2_{\mu_k}(\mcX)$ the associated conditional mean operator and $\dmok$ its corresponding deflated operator.
We make the following assumptions: \\
\textbf{(A1) Hilbert-Schmidt operator.}~For any $k \in [K]$, $\dmok$ is a Hilbert-Schmidt operator and its singular spectrum is concentrated on a small number $r_k$ of dominant components, so that it is well approximated by its rank-$r_k$ SVD truncation 
\begin{equation}
\textstyle \dmokr
=\sum_{i=1}^{r_k}
\sigma_i^{(k)}u_i^{(k)}\otimes v_i^{(k)}.
\end{equation}
\textbf{(A2) Shared functional spaces.}
We assume that the dominant singular functions of the different tasks can be represented within shared finite-dimensional functional spaces. More precisely,
there exist two $d$-dimensional functional spaces, with
$d\geq\max_k r_k$, such that for any $k \in [K]$ and $i \leq r_k$ 
\begin{equation}
u_i^{(k)}
\in
\operatorname{span}\{\phi_{1},\ldots,\phi_{d}\},
\qquad
v_i^{(k)}
\in
\operatorname{span}\{\psi_{1},\ldots,\psi_{d}\}.
\label{eq:shared_space_assumption}
\end{equation}
Thus, although each task may have a different effective rank and distinct singular functions, its dominant dependence structure is represented within the same pair of functional spaces. This shared representation provides the inductive bias that enables information to be shared across tasks.

\noindent
\textbf{Multi-task low-rank operator parametrization.}~We represent the shared functional spaces with two dictionaries of neural-networks
\begin{equation}
    \Phi_\theta(x) \triangleq \{\phi_1^{\theta}(x), \ldots,\phi_d^{\theta}(x)\}
    \qquad 
    \Psi_\theta(y) \triangleq \{\psi_1^{\theta}(y), \ldots,\psi_d^{\theta}(y)\}.
\end{equation} 
Although these dictionaries are shared across tasks, the orthogonality
constraints associated with the singular decomposition are task-dependent: for each task $k$, the left/right singular functions of the delfated operator $\dmok$ are orthonormal in $\lxk$ and $\lyk$ respectively. Consequently,
a direct parametrization of each operator through its singular value decomposition,
as in the NCP framework (see \Cref{sec:background}), would require enforcing
different orthogonality constraints for every task, making the optimization
cumbersome. 
To circumvent these task-specific constraints, we parametrize each
deflated operator through an unconstrained low-rank factorization:
\begin{equation}
    \dmokt = \msS_{\Phi_\theta}^{(k)} \mbA^{(k)}\mbB^{(k)\top}\msS_{\Psi_\theta}^{(k)*}
\end{equation}
where $(\mbA^{(k)},\mbB^{(k)}) \in (\mdR^{d \times r_k})^2$.
Here, $\textstyle \msS_{\Phi_\theta}^{(k)}:\mdR^d\to\lxk$ and
$\textstyle \msS_{\Psi_\theta}^{(k)}:\mdR^d\to\lyk$ denote the synthesis maps associated
with the shared neural dictionaries; they map any $\alpha \in \mdR^d$ to the functions 
\begin{equation}
    \label{eq:synthesis_map}
    \textstyle
    [\msS_{\Phi_\theta}^{(k)}\alpha](x)
    \triangleq
    \sum_{j=1}^d\alpha_j\phi_j^\theta(x),
    \qquad
    [\msS_{\Psi_\theta}^{(k)}\alpha](y)
    \triangleq
    \sum_{j=1}^d\alpha_j\psi_j^\theta(y).
\end{equation}
Importantly, despite the fact that synthesis maps are identical across tasks, their adjoints depend on the task-specific inner products. In particular, the adjoint of
$\msS_{\Psi_\theta}^{(k)}$ is given by
\begin{equation}
    \textstyle
    \msS_{\Psi_\theta}^{(k)*}f
    =
    \left(
        \innerp{\psi_j^\theta}{f}_{\nu_k}
    \right)_{j\in[d]},
    \qquad f\in\lyk.
\end{equation}

\noindent
\textbf{Retrieving task-specific SVD from the low-rank parametrization.}
For each task $k$, the singular value decomposition of the deflated operator can
be recovered \emph{post hoc} by centering and whitening the shared dictionaries with their
task-specific Gram matrices that can be estimated from data
\begin{equation}
    \mbG_{\Phi_{\theta\!,c}}^{(k)}
    \triangleq
    \mdE_{\mu_k}\!\left[
        \Phi_{\theta\!,c}^{(k)}(X)\Phi_{\theta\!,c}^{(k)}(X)^\top
    \right],
    \qquad
    \mbG_{\Psi_{\theta\!,c}}^{(k)}
    \triangleq
    \mdE_{\nu_k}\!\left[
        \Psi_{\theta\!,c}^{(k)}(Y)\Psi_{\theta\!,c}^{(k)}(Y)^\top
    \right],
\end{equation}
with $\Phi^{(k)}_{\theta\!,c}(x)= \Phi_{\theta} - \mdE_{\mu_k}[ \Phi_{\theta}(X)]$ and $\Psi^{(k)}_{\theta\!,c}(y)= \Psi_{\theta} - \mdE_{\nu_k}[ \Psi_{\theta}(Y)]$. Consider the matrix SVD
\begin{equation}
    \mbG_{\Phi_{\theta\!,c}}^{(k)1/2}
    \mbA^{(k)}\mbB^{(k)\top}
    \mbG_{\Psi_{\theta\!,c}}^{(k)1/2}
    =
    \mbU^{(k)}\Diag(\bs\sigma^{(k)})\mbV^{(k)\top}.
\end{equation}
Denoting $\mbu_i^{(k)}$ and $\mbv_i^{(k)}$ the $i^{th}$ columns of $\mbU^{(k)}$ and $\mbV^{(k)}$,
the singular functions are given by
\begin{equation}
    \widetilde{u}_i^{(k)}
    =
    \msS_{\Phi_{\theta\!,c}}^{(k)}
    \mbG_{\Phi_{\theta\!,c}}^{(k)-1/2}\mbu_i^{(k)},
    \qquad
    \widetilde{v}_i^{(k)}
    =
    \msS_{\Psi_{\theta\!,c}}^{(k)}
    \mbG_{\Psi_{\theta\!,c}}^{(k)-1/2}\mbv_i^{(k)},
\end{equation}
where 
$\mbG_{\Phi_{\theta\!,c}}^{(k)-1/2}$ and
$\mbG_{\Psi_{\theta\!,c}}^{(k)-1/2}$
are square root Moore--Penrose inverses.
By construction, the functions $\textstyle \{\widetilde{u}_i^{(k)}\}_{i \in [r_k]}$ and $\textstyle \{\widetilde{v}_i^{(k)}\}_{i \in [r_k]}$ are  orthonormal in their respective space $\lxk$ and $\lyk$,
and therefore yield the task-specific deflated operator with singular decomposition
\begin{equation}
    \textstyle
    \odmokt
    =
    \sum_{i=1}^{r_k}
    \sigma_i^{(k)}
    \widetilde u_i^{(k)}
    \otimes
    \widetilde v_i^{(k)}.
\end{equation}
Thus, the proposed factorization avoids imposing orthogonality constraints
during training while retaining access to an orthonormal, task-specific
spectral representation after optimization.

\noindent
\textbf{Multi-task learning protocol (MTL-CMO).}
At the population level, we jointly estimate the $K$ deflated operators by solving
\begin{equation}
    \min_{\theta, \{\mbA^{(k)},\mbB^{(k)}\}_{k=1}^K} \sum_{k=1}^K \underbrace{\|\dmokt - \dmok\|^2_{\mathrm{HS}} - \|\dmok\|^2_{\mathrm{HS}}}_{\mcL^{(k)}} + \lambda \underbrace{(\|\mbA^{(k)}\|^2_\mathrm{F} + \|\mbB^{(k)}\|^2_\mathrm{F})}_{\mcR^{(k)}}.
\end{equation}
The data-fitting term $\textstyle \mcL^{(k)}$ measures the discrepancy between the estimated deflated operator and its population counterpart, following the NCP framework (see \Cref{sec:background}). However, unlike NCP, the regularization $\textstyle \mcR^{(k)}$ does not include orthogonality penalizations, it only imposes a ridge regularization on the matrices $(\mbA^{(k)},\mbB^{(k)})$ to resolve the scaling ambiguity of the factorization.\\
From an empirical perspective, given a set of observations $\textstyle \{(x^{(k)}_i,y^{(k)}_i)\}_{i \in [n_k]}$, we denote $\Phi_\theta(\mbX^{(k)}) = \{\Phi_\theta(x^{(k)}_i)\}_{i\in[n_k]}$ and $\Psi_\theta(\mbY^{(k)}) = \{\Psi_\theta(y^{(k)}_i)\}_{i\in[n_k]}$ the feature matrices. Considering, the matrices 
\begin{equation}
    \mbQ^{(k)} = \Phi_\theta(\mbX^{(k)})\mbA^{(k)}\mbB^{(k)\top}\Psi_\theta(\mbY^{(k)})^\top
    \quad 
    \text{and}
    \quad
    \mbH^{(k)} = \mbI_{n_k} - 1/n_k \mb1\mb1^\top,
\end{equation} 
the data-fitting loss $\mcL^{(k)}$  admits an unbiased empirical estimator defined by 
\begin{equation}
    \widehat{\mcL^{(k)}} \triangleq \frac{\|\mbQ^{(k)}\|^2_\mathrm{F} - \|\Diag(\mbQ^{(k)})\|^2_\mathrm{2}}{n_k(n_k-1)} - \frac{2}{n_k-1} \Tr(\mbH^{(k)}\mbQ^{(k)}).
\end{equation}
The multi-task training procedure and the pseudo-code can be found in Appendix~\ref{app:method_app}.

\noindent
\textbf{Related work.}
Standard multi-task representation learning  methods often share an input feature map across task-specific predictors~\citep{caruana1997multitask,maurer2016benefit}. MTL-CMO extends this principle to operators by learning input and output function spaces, whose corresponding synthesis maps~\Cref{eq:synthesis_map} act in spaces defined by each task's marginals. This differs from conditional mean embedding, which typically uses a prescribed reproducing kernel Hilbert space~\citep{song2009hilbert,muandet2017kernel}, and from NCP, which estimates function spaces but for only a single CMO~\citep{kostic2024ncp}. A detailed related work is provided in Appendix~\ref{app:related_work}.

\subsection{Transfer learning of conditional mean operators in closed form (T-CMO)}
\label{sec:T-CMO}
We now suppose that the shared functional space have been learned and are now frozen: $\Phi_{\widehat{\theta}}$ and $\Psi_{\widehat{\theta}}$. 
We aim to estimate the CMO $\cmon$ associated to a new pair of random variable $(X^{\text{new}},Y^{\text{new}})$.  Rather than directly enforcing the low-rank structure, we parametrize the deflated operator with a matrix $\mbM \in \mdR^{d \times d}$ such that: 
\begin{equation}
    \dmon = \msS_{\Phi_{\widehat{\theta}\!,c}}^{\text{new}} \mbM\msS_{\Psi_{\widehat{\theta}\!,c}}^{\text{new}*}.
\end{equation}
Denoting $\textstyle \mbC_{\Phi_{\widehat{\theta}\!,c}\Psi_{\widehat{\theta}\!,c}}^{\text{new}} = \mdE_{\rho_{\text{new}}}[\Phi_{\widehat{\theta}\!,c}(X^{\text{new}})\Psi_{\widehat{\theta}\!,c}(Y^{\text{new}})^\top]$ the cross-covariance matrix, the loss
\begin{equation}
    \mcL^{\text{new}} = \Tr\left(\mbG_{\Psi_{\widehat{\theta}\!,c}}^{\text{new}}
    \mbM^\top
    \mbG_{\Phi_{\widehat{\theta}\!,c}}^{\text{new}}
    \mbM\right) - 2 \Tr\left(\mbC_{\Phi_{\widehat{\theta}\!,c}\Psi_{\widehat{\theta}\!,c}}^{\text{new}\top}\mbM\right),
\end{equation}
yields a quadratic problem in $\mbM$ with closed-form solution
\begin{equation}
    \mbM^* = 
    \mbG_{\Phi_{\widehat{\theta}\!,c}}^{\text{new}^{-1}}
    \mbC_{\Phi_{\widehat{\theta}\!,c}\Psi_{\widehat{\theta}\!,c}}^{\text{new}}
    \mbG_{\Psi_{\widehat{\theta}\!,c}}^{\text{new}^{-1}}.
\end{equation}
In a post-hoc phase, the SVD of the deflated operator can be retrieved following the procedure presented in the previous section and then truncated to retrieve its low-rank modeling.

\subsection{Statistical guarantees}

We now quantify the benefit of jointly learning the shared spaces through a statistical learning rate. We provide below a simplified result whose complete statement and proof is available in Appendix~\ref{sec:theory}, alongside additional results, including a guarantee for transfer to a new task.

\begin{thm}[Informal, see \Cref{thm:rate}]
Assume $K$ tasks with $n$ observation each, dictionary classes $(\Phi_\Theta,\Psi_\Theta)$ with $d$ bounded neural functions, and uniformly bounded task-specific matrices $\mbA^{(k)}$ and $\mbB^{(k)}$.
With high probability, the trained model $\textstyle \widehat h{=}(\widehat\theta,\{\widehat{\mbA}^{(k)},\widehat{\mbB}^{(k)}\}_{k \in [K]})$ satisfies
\begin{equation}
    \frac{1}{K} \sum_{k=1}^{K} \| \dmokth - \dmok\|^2_{\mathrm{HS}} -\mcA_K \leq \tilde{\mcO}\!\left( \frac{\sqrt{d} \big(\mathrm{Comp}(\Phi_\Theta)+\mathrm{Comp}(\Psi_\Theta)\big)}{\sqrt{nK}}+\frac{d}{\sqrt{n}}\right)
\end{equation}
where $\mcA_K$ is the irreducible risk, $\mathrm{Comp}(\cdot)$ is the Gaussian complexity of a dictionary class and $\tilde{\mcO}$ hides the logarithmic term, and a multiplicative factor depending on the boundedness constants.
\end{thm}
This result extends the statistical benefit of multi-task representation learning~\citep{maurer2016benefit} to operator estimation. The key intuition is that all tasks contribute to the learning of the shared function spaces: the cost of learning the neural dictionaries is therefore spread over all $nK$ observations. Once the spaces are learned, each operator still requires its own matrices, giving a task-specific term that depends on $d/\sqrt{n}$ alone.

%% file: sections/numerical_experiment.tex
\section{Experimental results}
\label{sec:experiments}

\subsection{Uncertainty quantification: Conditional distributions estimation}
\noindent
\paragraph{Experiment objective.} We evaluate MTL-CMO, and subsequently T-CMO, performances for uncertainty quantification on population of related synthetic conditional distributions by comparing estimated conditional cumulative distribution functions (CCDFs) to the ground truths.

\noindent
\textbf{Datasets.} 
We consider four families of conditional distributions, where the conditioning variable modifies the conditional expectation. Once conditioned, \textbf{CD1} and \textbf{CD2} are bimodal Gaussian
mixtures with task-dependent mixture weights, amplitudes and scales. \textbf{CD3} are
Student-$t$ distributions with task-dependent location, scale and
degrees of freedom. \textbf{CD4} are
skew-normal distributions with task-dependent location, scale and skewness.
While \textbf{CD1 }has an univariate conditioning others have a multivariate conditioning.
Thus, \textbf{CD1} to \textbf{CD4} progressively introduce complex conditioning,
heavy tails, and asymmetric conditional distributions. Families are detailed in Appendix~\ref{app:uq_families}.

\begin{table}[t]
\centering
\caption{ Mean 1-Wasserstein distance to the ground-truth conditional distribution averaged over 100 tasks. Scores in $<\!\text{mean}\!>\pm <\!\text{std}\!>$ over 10 training seeds. Lower is better. \textbf{First} and \underline{second}.}
\label{tab:cd_w1}
\resizebox{\linewidth}{!}{
\begin{tabular}{clcccc}
\toprule
& Method
& \textbf{CD1} & \textbf{CD2} & \textbf{CD3} & \textbf{CD4} \\
\midrule

\multirow{4}{*}{\rotatebox[origin=c]{90}{Multi-task}}
& MTL-CMO \textbf{(ours)}
& \textbf{0.088 $\pm$ 0.004}
& \textbf{0.102 $\pm$ 0.002}
& \textbf{0.106 $\pm$ 0.003}
& 0.070 $\pm$ 0.001 \\

& Pooled-NCP~\citep{kostic2024ncp}
& \underline{0.091 $\pm$ 0.002}
& 0.120 $\pm$ 0.009
& 0.139 $\pm$ 0.008
& 0.077 $\pm$ 0.008 \\

& MTL-MDN~\citep{bishop1994mixture}
& 0.110 $\pm$ 0.001
& 0.138 $\pm$ 0.005
& 0.216 $\pm$ 0.009
& 0.107 $\pm$ 0.005 \\

& DeepJMQR~\citep{rodrigues2020beyond}
& \underline{0.091 $\pm$ 0.002}
& \underline{0.114 $\pm$ 0.005}
& 0.125 $\pm$ 0.008
& \underline{0.066 $\pm$ 0.004} \\

\midrule

\multirow{8}{*}{\rotatebox[origin=c]{90}{Single-task}}
& NCP~\citep{kostic2024ncp}
& 0.096 $\pm$ 0.001
& 0.167 $\pm$ 0.037
& 0.185 $\pm$ 0.001
& 0.150 $\pm$ 0.002 \\

& CFM~\citep{lipman2023flow}
& 0.112 $\pm$ 0.001
& 0.182 $\pm$ 0.021
& 0.170 $\pm$ 0.006
& 0.150 $\pm$ 0.005 \\

& MDN~\citep{bishop1994mixture}
& 0.115 $\pm$ 0.002
& 0.259 $\pm$ 0.007
& 0.265 $\pm$ 0.008
& 0.218 $\pm$ 0.003 \\

& NSF~\citep{durkan2019neural}
& 0.109 $\pm$ 0.002
& 0.118 $\pm$ 0.013
& 0.129 $\pm$ 0.012
& 0.075 $\pm$ 0.006 \\

& Engression~\citep{shen2025engression}
& 0.143 $\pm$ 0.001
& 0.147 $\pm$ 0.001
& 0.129 $\pm$ 0.002
& 0.074 $\pm$ 0.001 \\

& Nadaraya--Watson~\citep{nadaraya1964estimating}
& 0.096 $\pm$ 0.001
& 0.152 $\pm$ 0.003
& \underline{0.109 $\pm$ 0.002}
& \textbf{0.064 $\pm$ 0.001} \\

& NGBoost~\citep{duan2020ngboost}
& 0.228 $\pm$ 0.003
& 0.284 $\pm$ 0.005
& 0.278 $\pm$ 0.004
& 0.242 $\pm$ 0.006 \\

& FlexCode~\citep{izbicki2017converting}
& 0.245 $\pm$ 0.001
& 0.225 $\pm$ 0.000
& 0.321 $\pm$ 0.003
& 0.327 $\pm$ 0.003 \\

\bottomrule
\end{tabular}
}
\end{table}
\begin{figure}[t]
    \centering
    \includegraphics[width=1\linewidth]{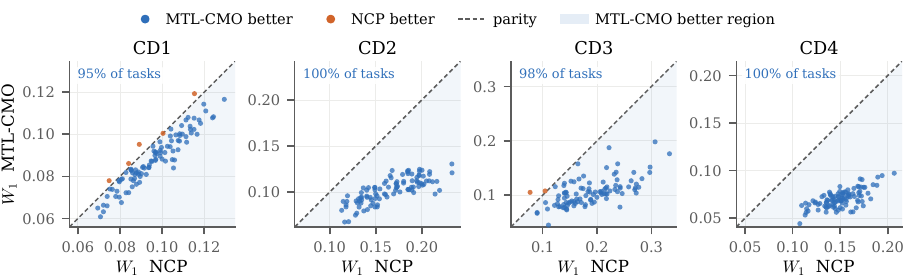}\vspace{-5mm}
        \caption{Comparison between MTL-CMO and the single-task NCP for representing CDFs on four conditional distribution families. Performances measured with the 1-Wasserstein to the ground truth.}
        \label{fig:mtl_ncp_vs_ncp}
\end{figure}
\noindent
\textbf{Multi-task learning evaluation.} 
We sample $K{=}100$ conditional distributions and generate
$n{=}400$ observations for each task and family. We compare MTL-CMO against $3$ multi-task 
and 8 single-task alternatives detailed in Appendix~\ref{app:uq_baselines}.
The experiment is
repeated over $10$ seeds and Appendix~\ref{app:uq_architecture} provides implementation details. Performance is measured with the 1-Wasserstein distance between the estimated and the ground-truth CCDFs averaged over the population and conditioning inputs. \\
Results are reported in \Cref{tab:cd_w1}. MTL-CMO outperforms all competing methods on CD1--CD3 and remains competitive on CD4. Relative to its single-task counterpart, NCP, joint representation learning reduces the Wasserstein error by $8\%$, $39\%$, $43\%$, and $53\%$ on CD1--CD4, respectively. This behavior is further illustrated in \Cref{fig:mtl_ncp_vs_ncp}, which compares the performance of MTL-CMO and to NCP across the four families for different number of training sampels per task. The improvement becomes substantially larger for the more challenging families, which involve multivariate conditioning, heavy tails, or asymmetric distributions. These results highlight the benefit of sharing information across related tasks through shared input/output function spaces, particularly when each task has not enough observations for accurate estimation.

\begin{figure}[t]
    \centering
    \includegraphics[width=\linewidth]{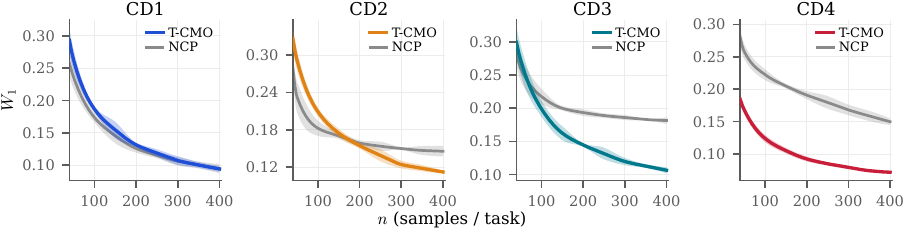}\vspace{-6mm}
    \caption{
    Transfer learning evaluation on unseen conditional distributions. Comparison in $1$-Wasserstein between T-CMO and the single-task NCP depending on the number of observations.
    }
    \label{fig:CD_transferability}
\end{figure}

\noindent
\textbf{Transfer learning evaluation.}
We next freeze the input and output function spaces learned during multi-task training. For each family and dataset sizes between $50$ and $400$, we sample observations from $100$ previously unseen conditional distributions. We estimate their conditional CDFs using the closed-form T-CMO estimator and compare it with NCP models trained independently from scratch for each task and sample size.\\
Results are shown in \Cref{fig:CD_transferability}. On the simpler CD1 family, T-CMO and NCP achieve comparable performance. As the conditional families become more complex (CD2--CD4), T-CMO increasingly outperforms NCP as additional target observations become available. The reduced performance of T-CMO in the very low-data regime for some families is consistent with the greater numerical sensitivity of the empirical Gram-matrix inversion when few samples are available. Overall, these results show that function spaces learned from a diverse population of source tasks can be effectively reused to estimate the conditional distributions of previously unseen tasks~\citep{tripuraneni2020theory}.
Implementation details and additional results are provided in Appendix~\ref{app:uq_transfer}.

\subsection{Müller-Brown Langevin dynamics}

\begin{table}[!tbp]
\centering
\caption{
Reconstruction error between estimated
and ground truth operators. Hilbert-Schmidt error
in $<\!\text{mean}\!>\pm <\!\text{std}\!>$ 
over the $150$ unseen systems. Lower is better. \textbf{First} and \underline{second}.
}
\label{tab:l2d_hs}
\resizebox{\linewidth}{!}{
\begin{tabular}{lcccccc}
\toprule
& \multicolumn{5}{c}{\textbf{Trajectory length} $N$} \\
\cmidrule(lr){2-6}
Method
& \textbf{5\,000} & \textbf{10\,000}
& \textbf{20\,000} & \textbf{40\,000} & \textbf{80\,000} \\
\midrule
T-CMO \textbf{(ours)}
& $\mathbf{0.261 \pm 0.148}$
& $\mathbf{0.190 \pm 0.108}$ & $\mathbf{0.140 \pm 0.075}$
& $\mathbf{0.099 \pm 0.055}$ & $\mathbf{0.075 \pm 0.039}$ \\
Single-task NCP \citep{kostic2024ncp}
& $0.297 \pm 0.138$
& $0.224 \pm 0.098$ & $0.176 \pm 0.068$
& $0.138 \pm 0.047$ & $0.115 \pm 0.031$ \\
Pooled NCP \citep{kostic2024ncp}
& $0.306 \pm 0.143$
& $0.219 \pm 0.102$ & $0.163 \pm 0.078$
& $0.118 \pm 0.058$ & $\underline{0.087 \pm 0.041}$ \\
RFF/EDMD~\citep{li2017extended}
& $0.295 \pm 0.140$
& $0.233 \pm 0.123$ & $0.177 \pm 0.069$
& $0.135 \pm 0.045$ & $0.109 \pm 0.030$ \\
RFF/Laplace~\citep{kostic2025laplace}
& $0.296 \pm 0.141$
& $0.221 \pm 0.100$ & $0.172 \pm 0.072$
& $0.132 \pm 0.051$ & $0.106 \pm 0.034$ \\
MetaKoopman~\citep{iwata2021meta}
& $0.288 \pm 0.142$
& $0.220 \pm 0.100$ & $0.173 \pm 0.070$
& $0.136 \pm 0.048$ & $0.115 \pm 0.034$ \\
VAMPNet~\citep{mardt2018vampnets}
& $\underline{0.273 \pm 0.145}$
& $\underline{0.205 \pm 0.104}$ & $\underline{0.154 \pm 0.071}$
& $\underline{0.110 \pm 0.051}$ & $0.088 \pm 0.034$ \\
\bottomrule
\end{tabular}
}
\end{table}

\noindent
\textbf{Müller-Brown dynamical family.}
The Müller-Brown family is a standard benchmark for slow dynamic of metastable systems and transition-path problems~\citep{muller1979location,zhang2022solving}.
It consists in overdamped Langevin dynamics with two-dimensional energy landscapes comprising three wells separated by energy barriers. Typical trajectories spend long periods within metastable basins and only rarely transition between them. Consequently, the dominant nontrivial eigenfunctions of the Infinitesimal Generator (IG, see \Cref{sec:background}) characterize the principal metastable states, while the corresponding eigenvalues encode the transition timescales.
We construct a family of such systems by varying the orientation and depth of one potential well. These variations alter the metastable regions and transition timescales, leading to task-specific spectral decompositions of the corresponding infinitesimal generators.

\noindent
\textbf{Experimental protocol.} The experiment comprises three stages: (i) learning the shared functional space with MTL-CMO from $750$ systems, (ii) finetuning the hyperparameters of the IG estimator with $50$ dynamics, and (iii) estimating the IG of $150$ unseen dynamics by transferring the function space with T-CMO. Each system provides a trajectory of $80{,}000$ observations. Transfer performance is evaluated from truncated trajectories from $2k$ to $80k$ observations. Since Langevin dynamics are time-reversible, the IGs are self-adjoint and we consider a single function dictionary, i.e. $\Phi_\theta{=}\Psi_\theta$. The IGs are estimated with the Laplace estimator~\citep{kostic2025laplace}. Stage (ii) shows that the first three non-trivial components capture $95\%$ of the spectral energy, motivating our focus on rank-$3$ spectral recovery. For transfer learning, we compare T-CMO/Laplace against: single-task NCP/Laplace, pooled NCP/Laplace, RFF/Laplace, RFF/EDMD, MetaKoopman, and VampNet. Implementation details are provided in Appendix~\ref{app:langevin_2d}.

\begin{figure}[t]
    \centering
    \includegraphics[width=1\linewidth]{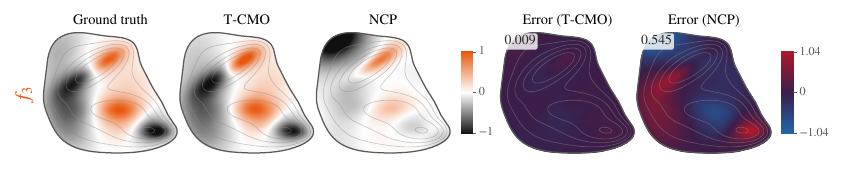}\vspace{-5mm}
    \caption{
    Recovery of $3^{rd}$ eigenfunction of an unseen Müller-Brown system with T-CMO/Laplace and its single-task counterpart NCP/Laplace. Global error in cosine distance.
    }
    \label{fig:mb_psi3_fields}
\end{figure}

\begin{figure}[t]
    \centering
    \includegraphics[width=1\linewidth]{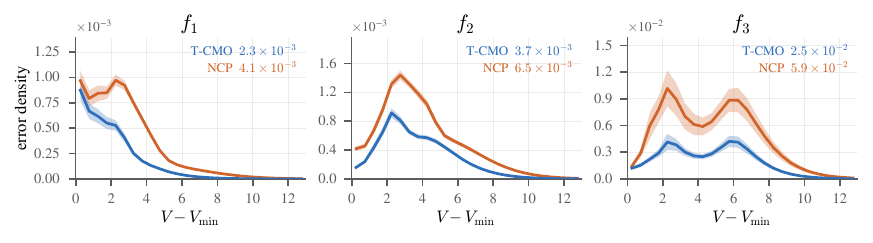}\vspace{-8mm}
    \caption{
    Reconstruction error across eigenfunctions and potential values between T-CMO/Laplace and NCP/Laplace. Error averaged over the $150$ unseen systems. Global error in cosine distance.
    }
    \label{fig:mb_eigenfunction_errors}
\end{figure}

\noindent
\textbf{Results.} 
The multi-task results reported in Appendix~\ref{app:mb_operator_learning} show that MTL-CMO improves operator reconstruction over its single-task counterpart, NCP.
In transfer learning, T-CMO/Laplace achieves the lowest operator reconstruction error (\Cref{tab:l2d_hs}) and the most accurate recovery of the individual spectral components (Appendix~\ref{app:mb_transfer}) across all trajectory lengths. More specifically, \Cref{fig:mb_psi3_fields} shows that T-CMO accurately recovers the spatial structure of the third eigenfunction, whereas NCP, exhibits substantially larger localized errors. Figure \ref{fig:mb_eigenfunction_errors} shows a consistent improvement of T-CMO across all three eigenfunctions and potential values, reducing the overall error by approximately a factor of two. These results demonstrate that transferring the learned function space is particularly beneficial for recovering challenging spectral components of unseen dynamics from limited trajectory data. Additional results can be found in Appendix~\ref{app:mb_transfer}.

\FloatBarrier
\subsection{Turbulent plasma dynamics}

\noindent
\textbf{Plasma dynamics and reduced-model identification.}
Turbulent plasmas exhibit strongly nonlinear and multiscale interactions that make high-fidelity simulations computationally demanding. They are therefore often approximated by reduced physical models that retain the dominant mechanisms while enabling faster simulations, which is valuable for applications requiring rapid predictions, such as plasma monitoring and control. Their practical use, however, requires inferring model parameters from observed trajectories, a challenging problem because turbulence and instabilities can obscure their effects on the dynamics~\citep{boeuf2023exb,coosemans2021bayesian}. This has motivated data-driven operator approaches for extracting physical information from plasma measurements and simulations~\citep{faraji2023dmd,taylor2017dmd}.

\begin{figure}[t]
    \centering
    \includegraphics[width=1\linewidth]{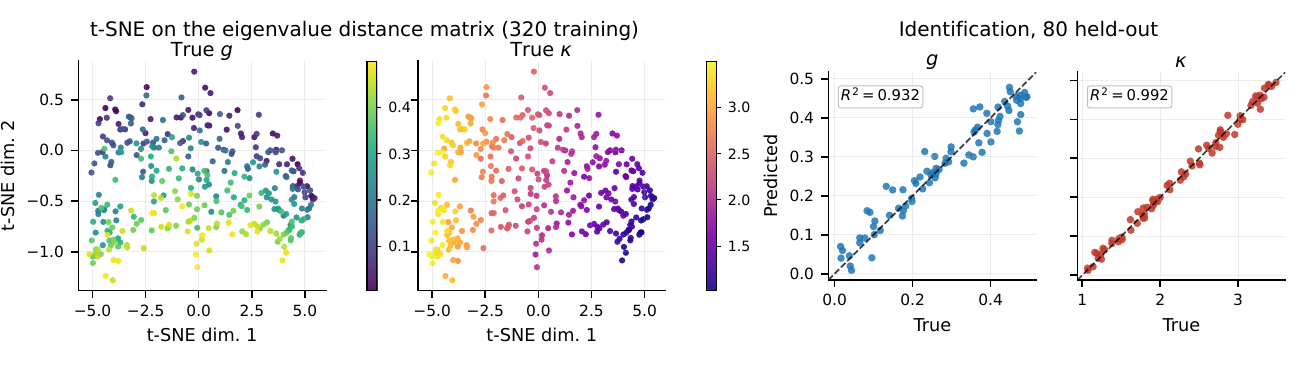}\vspace{-8mm}
\caption{
Left: t-SNE on the distance matrix between the $320$ training operator spectra colored by $g$ and $\kappa$. Right: parameter identification with kernel regression evaluated on the $80$ held-out systems.
}
    \label{fig:plasma_tsne_regression}
\end{figure}

\noindent
\textbf{Experiment objective.}
We consider Tokam2D~\citep{tokam2d}, a plasma simulator based on a reduced two-field model describing the evolution of plasma density and electrostatic potential in a two-dimensional plane. The dynamics are parametrized by $g$, which controls the interchange coupling between density and vorticity, and $\kappa$, which controls the imposed density-gradient drive~\citep{ghendrih2018sol,ghendrih2022avalanche}. Varying these parameters produces distinct turbulent regimes. We investigate whether $g$ and $\kappa$ can be recovered from the spectra of IGs estimated through T-CMO. Appendix~\ref{sec:plasma_physics} provides a detailed descriptions.

\begin{wrapfigure}{r}{0.38\textwidth}
    \centering
    \includegraphics[width=0.95\linewidth]{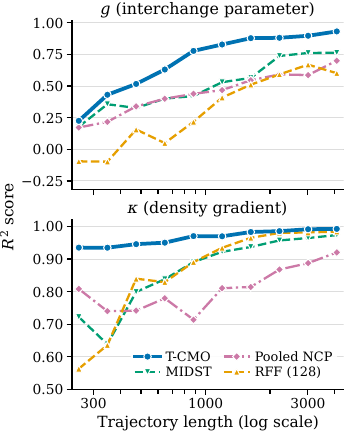}\vspace{-4mm}
    \caption{Identification of the parameters $g$ and $\kappa$ versus trajectory length for T-CMO and competing methods.}\vspace{-4mm}
    \label{fig:plasma_horizon1}
\end{wrapfigure}

\noindent
\textbf{Experimental protocol.}
We simulate $400$ plasma trajectories of $4{,}093$ observations on a $64\times64$ spatial grid, with $(g,\kappa)$ sampled uniformly from $[0,0.5]\times[1,3.5]$. We train MTL-CMO on $320$ trajectories using a convolutional ResNet that produces a dictionary of $128$ functions. For IG estimation, we combine T-CMO with the Laplace estimator~\citep{kostic2025laplace}. We regress $g$ and $\kappa$ from the estimated IG eigenvalues on the training set using kernel ridge regression. On the $80$ held-out trajectories of increasing length, we compare our approach against eigenvalues obtained with MIDST~\citep{elul2024data}, pooled NCP/Laplace~\citep{kostic2024ncp}, and RFF/Laplace~\citep{brault2016random}. Implementation details are provided in Appendix~\ref{sec:plasma_representation}.

\noindent
\textbf{Results.}
\Cref{fig:plasma_tsne_regression}-left shows a t-SNE embedding of the IG eigenvalue signatures estimated by MTL-CMO on the $320$ training systems. Although $g$ and $\kappa$ are not provided during training, the resulting representations vary smoothly with both parameters. On the $80$ held-out systems and using full trajectories, T-CMO/Laplace can be used to predict with $R^2{=}0.932$ for $g$ and $R^2{=}0.992$ for $\kappa$ (\Cref{fig:plasma_tsne_regression}-right). Moreover, \Cref{fig:plasma_horizon1} shows that it outperforms all baselines across all trajectory lengths with an early near-perfect recovery of $\kappa$ and rapidly improving the estimation of $g$ compared to other baselines. These results show the benefit of jointly learning and transferring functional spaces for parameter identification in plasma dynamics. In appendices~\ref{sec:plasma_operator} and \ref{sec:plasma_identification} we provide additional results.

%% file: sections/conclusion.tex
\section{Conclusion}
In this work, we introduced MTL-CMO, a multi-task approach that jointly learns conditional mean operators through shared function spaces and task-specific linear operators. We then proposed T-CMO, which reuses the learned function spaces to estimate the operator of a new distribution in closed form. Our statistical results and experiments on uncertainty quantification and complex dynamical systems demonstrate the benefits of jointly learning the shared function spaces.\\
In future work we will investigate new strategies to improve the estimation of the task-specific low-rank operators by leveraging the population manifold structure of their matrix representations.

%% file: sections/appendix.tex
\section{Operator representation of dynamical systems \& estimation}
\label{app:op_ds}

\subsection{Operator representation of dynamical systems \& spectral decomposition}
\label{app:op_ds_rpr}

\paragraph{Markov semigroup and transfer operator.}
Let $(X_t)_{t \geq 0}$ be a time-homogeneous Markov process on a measurable state space $\mcX$, and invariant measure $\pi$. Let $\mcF \subset \mcL^2_\pi(\mcX) $ be a space of observables. The evolution of the system is described by the Markov semigroup $\{A_t\}_{t \geq 0}$ acting on $\mcF$, such as
\begin{equation}
    [A_t f](x) = \mathbb{E}[f(X_t)\mid X_0 = x], \quad f \in \mcF.
\end{equation}

\paragraph{Infinitesimal generator.}
The family $\{A_t\}_{t\geq 0}$ is a strongly continuous semigroup satisfying
\[
A_0 = \mathrm{Id}, \qquad A_{t+s} = A_t A_s, \quad \forall t,s \geq 0.
\]
The infinitesimal generator $G$ associated with $\{A_t\}_{t\geq 0}$ is defined as
\begin{equation}
    G f \triangleq \lim_{t \to 0^+} \frac{A_t f - f}{t},
    \label{eq:generator_appendix}
\end{equation}
with domain $\mathcal{D}(G) \subset \mcF$. The generator is, in general, an unbounded linear operator on $\mcF$ and characterizes the semigroup through $A_t = \exp(tG)$.

\paragraph{Spectral decomposition via compact resolvent.}
The spectral analysis of $G$ requires additional structure due to its infinite-dimensional and unbounded nature. Let $\pi$ be an invariant measure of the process and consider the Hilbert space $\mcL^2_\pi(\mcX)$.

We assume that $G$ has compact resolvent, i.e., there exists $\mu_0$ in the spectrum of $G$ such that $(\mu_0 \Id - G)^{-1}$ is compact on $\mcL^2_\pi(\mcX)$. Under this assumption, the spectrum of $G$ is discrete and consists of isolated eigenvalues $(\lambda_i)_{i\in\mathbb{N}}$ with finite multiplicity and no accumulation point except possibly at infinity.

In addition, there exists an bi-orthogonal basis $(f_i,g_i)_{i\in\mathbb{N}}$ of $\mcL^2_\pi(\mcX)$ such that
\begin{equation}
    G f_i = \lambda_i f_i  \qquad G^*g_i = \overline{\lambda_i} g_i \qquad \innerp{f_i}{g_j}_{\mcL^2_\pi(\mcX)} = \delta_{ij} 
\end{equation}

For each isolated eigenvalue $\lambda_i$, the associated spectral projector $P_i$ is defined as the projector onto the corresponding eigenspace in $\mcL^2_\pi(\mcX)$.

In the simple eigenvalue case, we have
\begin{equation}
    P_i f = \langle f, g_i\rangle_{\mcL^2_\pi(\mcX)} f_i .
\end{equation}

More generally, if $\lambda_i$ has multiplicity $m_i$ and $(f_{i,k},g_{i,k})_{k=1}^{m_i}$ is an bi-orthogonal basis of the eigenspace, then
\begin{equation}
    P_i f
    =
    \sum_{k=1}^{m_i}
    \langle f, g_{i,k}\rangle_{\mcL^2_\pi(\mcX)} f_{i,k}.
\end{equation}

The projectors satisfy
\begin{equation}
    P_i P_j = \delta_{ij} P_i,
    \qquad
    \sum_{i} P_i = \Id
\end{equation}
on the spectral subspace.

With this notation, the generator admits the spectral representation
\begin{equation}
    G f = \sum_{i} \lambda_i P_i f,
    \label{eq:spectral_appendix}
\end{equation}
for all $f \in \mathcal{D}(G)$ for which the expansion is well-defined, and the semigroup diagonalizes as
\begin{equation}
    A_t f
    =
    \sum_{i}
    e^{\lambda_i t} P_i f .
\end{equation}

\subsection{Estimation through generator resolvent}
\label{app:resolvent}

This section presents the methodology introduced in ~\citet{kostic2025laplace}.

\paragraph{Learning initial generators via spectral filtering.}

The first step consists in learning a collection of initial generator atoms directly from trajectory data by estimating spectral components of the infinitesimal generator \(G\). 
Rather than approximating \(G\) through finite differences of the form \((A_{\Delta t}-I)/\Delta t\), which is unstable at small time scales, we rely on spectral filtering based on Toeplitz representations of analytic functions of the generator.

Let \(A_{\Delta t} = e^{\Delta t G}\) denote the transfer operator at time step \(\Delta t\). 
We consider a Toeplitz symbol
\begin{equation}
    T(z) = \sum_{j\in\mathbb{Z}} a_j z^j,
    \label{eq:toeplitz_symbol_step1}
\end{equation}
and define the associated filtered operator
\begin{equation}
    F(G) = T(A_{\Delta t}).
    \label{eq:functional_calculus_step1}
\end{equation}

Since \(A_{j\Delta t} = A_{\Delta t}^j\), the operator \(F(G)\) admits the representation
\begin{equation}
    F(G) = \sum_{j\in\mathbb{Z}} a_j A_{\Delta t}^j,
    \label{eq:toeplitz_operator_expansion_step1}
\end{equation}
which can be approximated from data using time-lagged observations. In practice, we use a truncated expansion
\begin{equation}
    F_\ell(G) = \sum_{|j|\leq \ell} a_j A_{\Delta t}^j,
    \label{eq:truncated_toeplitz_step1}
\end{equation}
whose empirical action is represented by a banded Toeplitz matrix acting on the time-ordered data sequence. 
Thus, analytic functional calculus of the generator reduces to structured linear algebra on trajectory data.

\noindent
\textbf{Resolvent and exponential filters.}
To probe the spectrum of the generator \(G\), we use families of analytic filters.

The key object is the generator resolvent
\begin{equation}
    R_\mu = (\mu I - G)^{-1}
    = \int_0^\infty e^{t G} e^{-\mu t}\, dt,
    \label{eq:resolvent_definition_step1}
\end{equation}
which we approximate by a Toeplitz-weighted sum of transfer operators,
\begin{equation}
    R_{\mu,\ell} = \sum_{j=0}^{\ell} a_j A_{\Delta t}^j,
    \qquad
    a_j \approx \Delta t\, e^{-\mu j\Delta t},
    \label{eq:resolvent_toeplitz_step1}
\end{equation}
corresponding to a discretization of the Laplace transform.

Equivalently, using the transfer-operator resolvent, we have
\begin{equation}
    (e^\mu I - A_{\Delta t})^{-1}
    =
    \sum_{j=0}^{\infty} e^{-(j+1)\mu} A_{\Delta t}^j,
    \label{eq:transfer_resolvent_step1}
\end{equation}
which naturally yields Toeplitz coefficients. These filters concentrate spectral information near the shift parameter \(\mu\), allowing us to localize different regions of the spectrum.

In addition, exponential and trigonometric filters (e.g. \(e^{\Delta t G}\), \(\cosh(\Delta t G)\), \(\sinh(\Delta t G)\)) can be used to emphasize specific spectral structures depending on the nature of the dynamics.

\paragraph{Estimation from data.}
\Cref{alg:tcmo-generator} describes the procedure to estimate the infinitesimal generator according to~\citep{kostic2025laplace}.

\subsection{From T-CMO to estimation of operator representation of dynamical system.}
\label{app:tcmo-op-est}
Estimation of a dynamical system's operator representation requires a function space (observable space) that best captures the dynamics as depicted in \Cref{sec:background} and Appendix~\ref{app:op_ds_rpr}. By construction, MTL-CMO, introduced in \Cref{sec:method}, captures function spaces that describe the dominant mode of variation in a population of dynamical systems. These learned function spaces can be adapted via transfer with T-CMO to estimate the operator representation of a new, related dynamical system.

Formally, consider the learned features maps  $(\Phi_{\widehat{\theta}},\Psi_{\widehat{\theta}})$ with MTL-CMO on a population of related dynamical systems. Suppose a sampled trajectory $\{x_i^{\text{new}}\}_{i \in [T]} \subset \mcX$ from a new dynamical system with invariant measure $\pi_{\text{new}}$. Then, 

\begin{enumerate}
    \item With T-CMO estimate the new operator $\dmon$, see \Cref{sec:T-CMO} and \Cref{alg:tcmo}.
    \item Derive its singular value decomposition following \Cref{sec:mtl-cmo} (SVD paragraph) and keep the orthonormal function basis $U^{\text{new}} = \{u_i^{\text{new}}\}_{i \in [r_\text{new}]}$.
    \item Estimate the infinitesimal generator with the function space $\text{span}(U^{\text{new}}) \subset \mcL^2_{\pi_{\text{new}}}(\mcX)$ with any infinitesimal generator estimator.
\end{enumerate}
Various infinitesimal generator estimators exist, including: \citep{schmid2010dynamic,li2017extended,kostic2022learning,kostic2025laplace}.
In particular, we describe~\citet{kostic2025laplace} in Appendix~\ref{app:resolvent}.
\begin{algorithm}[t]
\caption{Generator estimation from T-CMO transferred observables}
\label{alg:tcmo-generator}
\small
\begin{algorithmic}[1]

\Require Trajectory
$\{x_t^\star\}_{t=0}^{T-1}$,
rank-$r$ observable basis
$U^\star=\{u_j^\star\}_{j=1}^{r}$ returned by T-CMO,
sampling step $\Delta t$, resolvent shift $\mu$,
maximum lag $\ell$, rank $q$, regularization $\gamma$

\State Evaluate
$\mbz_t^\star
\gets
\big(
u_1^\star(x_t^\star),\ldots,u_r^\star(x_t^\star)
\big)^\top
\in\mathbb R^r$

\State Form
$\mbZ^\star
\gets
(\mbz_0^\star,\ldots,\mbz_{T-1}^\star)^\top
\in\mathbb R^{T\times r}$

\State Construct the truncated Laplace weights
$\displaystyle
a_k
\gets
\mu\Delta t\,e^{-\mu k\Delta t},
\quad k=0,\ldots,\ell$

\State Form the associated Toeplitz operator
$\mbT_\mu$ and compute
$\displaystyle
\mbZ_\mu^\star
\gets
\mbT_\mu\mbZ^\star$

\State Compute
$\displaystyle
\widehat{\mbC}^\star
\gets
\frac{1}{T}\mbZ^{\star\top}\mbZ^\star$,
$\displaystyle
\widehat{\mbH}_\mu^\star
\gets
\frac{1}{T}\mbZ^{\star\top}\mbZ_\mu^\star$

\State Regularize
$\displaystyle
\widehat{\mbC}_\gamma^\star
\gets
\widehat{\mbC}^\star+\gamma\mbI_r$

\State Compute the $q$ dominant generalized eigenvectors of
$\displaystyle
\widehat{\mbH}_\mu^\star
\widehat{\mbH}_\mu^{\star\top}\mbv
=
\sigma^2
\widehat{\mbC}_\gamma^\star\mbv$
and collect the normalized directions in $\mbV_q^\star$

\State Form the reduced resolvent representation
$\displaystyle
\widehat{\mbR}_\mu^\star
\gets
\mbV_q^{\star\top}
\widehat{\mbH}_\mu^\star
\mbV_q^\star$

\State Compute
$\displaystyle
\widehat{\mbR}_\mu^\star\mbr_i^\star
=
\nu_i^\star\mbr_i^\star$

\State Recover
$\displaystyle
\mbv_i^\star
\gets
\mbV_q^\star\mbr_i^\star$

\State Recover the generator eigenvalues
$\displaystyle
\widehat{\lambda}_i^\star
\gets
\mu\left(1-\frac{1}{\nu_i^\star}\right)$

\State Define the corresponding eigenfunctions
$\displaystyle
\widehat f_i^\star(x)
\gets
\sum_{j=1}^{r}
[\mbv_i^\star]_j\,u_j^\star(x)$

\State \Return
$\{\widehat{\lambda}_i^\star\}_{i=1}^{q}$ and
$\{\widehat f_i^\star\}_{i=1}^{q}$

\end{algorithmic}
\end{algorithm}

\section{Multi-task learning for conditional mean operator: Appendix}
\label{app:method_app}
\label{app:mtlcmo_derivations}

We give additional details on the finite-dimensional representation of the
task-specific operators, the population objective optimized by MTL-CMO, and
the resulting transfer estimator.

\noindent
\paragraph{Post-hoc Singular Value Decomposition of deflated operators.}
For each task $k$, the singular value decomposition of the deflated operator can
be recovered \emph{post hoc} by centering and whitening the shared dictionaries. Formally, consider the orthogonal projectors 
\begin{equation*}
    P_{\mu_k} \triangleq I_d - \mb1_\mcX \otimes_{\mu_k} \mb1_\mcX \qquad P_{\nu_k} \triangleq I_d - \mb1_\mcY \otimes_{\nu_k} \mb1_\mcY,
\end{equation*}
where $\mb1_\mcX$ and $\mb1_\mcY$ are indicator functions. These projectors center functions, i.e. 
\begin{equation*}
    P_{\mu_k}f = f - \mdE_{\mu_k}[f(X)]
    \qquad 
    P_{\nu_k}g = g - \mdE_{\nu_k}[g(Y)].
\end{equation*}
Assuming an unconstrained operator $\dmokt$, we can define the centered deflated operator: 
\begin{equation*}
    \odmokt \triangleq P_{\mu_k}\dmokt P_{\mu_k} = \msS_{\Phi_{\theta\!,c}}^{(k)}\mbA^{(k)}\mbB^{(k)}\msS_{\Psi_{\theta\!,c}}^{(k)*},
\end{equation*}
with 
$\Phi_{\theta\!,c}^{(k)}(x)
=
\Phi_\theta(x)-\mdE_{\mu_k}\!\left[\Phi_\theta(X)\right]$ and $
\Psi_{\theta\!,c}^{(k)}(y)
=
\Psi_\theta(y)-\mdE_{\nu_k}\!\left[\Psi_\theta(Y)\right]$. In fact, the centered deflated operator is a better estimator of $\dmok$. Indeed,
\begin{equation*}
    \left\|\odmokt -\dmok \right\|_{\mathrm{HS}} \leq \left\|\dmokt -\dmok \right\|_{\mathrm{HS}},
\end{equation*}
since $P_{\mu_k}\dmok P_{\nu_k} = \dmok$ and projectors are contractant maps. Assume the centered Gram-matrices 
\begin{equation*}
    \mbG_{\Phi_{\theta\!,c}}^{(k)}
    \triangleq
    \mdE_{\mu_k}\!\left[
        \Phi_{\theta\!,c}^{(k)}(X)\Phi_{\theta\!,c}^{(k)}(X)^\top
    \right],
    \qquad
    \mbG_{\Psi_{\theta\!,c}}^{(k)}
    \triangleq
    \mdE_{\nu_k}\!\left[
        \Psi_{\theta\!,c}^{(k)}(Y)\Psi_{\theta\!,c}^{(k)}(Y)^\top
    \right],
\end{equation*}
Consider the matrix SVD
\begin{equation*}
    \mbG_{\Phi_{\theta\!,c}}^{(k)1/2}
    \mbA^{(k)}\mbB^{(k)\top}
    \mbG_{\Psi_{\theta\!,c}}^{(k)1/2}
    =
    \mbU^{(k)}\Diag(\bs\sigma^{(k)})\mbV^{(k)\top}.
\end{equation*}
Since
\[
\mbU^{(k)\top}\mbU^{(k)}
=
\mbI_{r_k},
\qquad
\mbV^{(k)\top}\mbV^{(k)}
=
\mbI_{r_k},
\]
the coefficient matrices
\[
\mbC^{(k)}
=
\mbG_{\Phi_\theta}^{(k)-1/2}\mbU^{(k)},
\qquad
\mbD^{(k)}
=
\mbG_{\Psi_\theta}^{(k)-1/2}\mbV^{(k)}
\]
satisfy
\[
\mbC^{(k)\top}
\mbG_{\Phi_\theta}^{(k)}
\mbC^{(k)}
=
\mbI_{r_k},
\qquad
\mbD^{(k)\top}
\mbG_{\Psi_\theta}^{(k)}
\mbD^{(k)}
=
\mbI_{r_k}.
\]

Denoting $\mbu_i^{(k)}$ and $\mbv_i^{(k)}$ the $i^{th}$ columns of $\mbU^{(k)}$ and $\mbV^{(k)}$,
the singular functions are given by
\begin{equation*}
    \widetilde{u}_i^{(k)}
    =
    \msS_{\Phi_{\theta\!,c}}^{(k)}
    \mbG_{\Phi_{\theta\!,c}}^{(k)-1/2}\mbu_i^{(k)},
    \qquad
    \widetilde{v}_i^{(k)}
    =
    \msS_{\Psi_{\theta\!,c}}^{(k)}
    \mbG_{\Psi_{\theta\!,c}}^{(k)-1/2}\mbv_i^{(k)},
\end{equation*}
where 
$\mbG_{\Phi_{\theta\!,c}}^{(k)-1/2}$ and
$\mbG_{\Psi_{\theta\!,c}}^{(k)-1/2}$
are square root inverses.
By construction, the functions $\textstyle \{\widetilde{u}_i^{(k)}\}_{i \in [r_k]}$ and $\textstyle \{\widetilde{v}_i^{(k)}\}_{i \in [r_k]}$ are  orthonormal in their respective space $\lxk$ and $\lyk$,
and therefore yield the task-specific singular decomposition
\begin{equation}
    \textstyle
    \odmokt
    =
    \sum_{i=1}^{r_k}
    \sigma_i^{(k)}
    \widetilde u_i^{(k)}
    \otimes
    \widetilde v_i^{(k)}.
\end{equation}
Thus, the proposed factorization avoids imposing orthogonality constraints
during training while retaining access to an orthonormal, task-specific
spectral representation after optimization.

\paragraph{Derivation task-specific theoretical and empirical losses.} We start be deriving the loss for each task.
Consider the task $k$, its loss is the discrepancy between the $\dmokt$ and $\dmok$: 
\[
\begin{aligned}
    \mcL^{(k)} &= \|\dmokt - \dmok\|_{\mathrm{HS}}^2 - \|\dmok\|_{\mathrm{HS}}^2 \\
    & = \|\dmokt\|_{\mathrm{HS}}^2 - 2\innerp{\dmokt}{\dmok}_{\mathrm{HS}}.
\end{aligned}
\]
Let
\(
q_k(x,y)
=
\sum_{i,j}
M_{ij}^{(k)}
\phi_{i}^{(k)}(x)
\psi_{j}^{(k)}(y)
\), with $\mbM^{(k)} = \mbA^{(k)}\mbB^{(k)\top}$,
be the kernel of $\dmokt$. Since $\dmokt:L^2_{\nu_k}\to L^2_{\mu_k}$ is an integral operator,
\[
\begin{aligned}
    \|\dmokt\|_{\mathrm{HS}}^2
    & = \int q_k(x,y)^2\,d\mu_k(x)d\nu_k(y) \\
    & = \mdE_{\mu_k \otimes \nu_k}[q_k(X,Y)^2] \\
    & = \Tr\!\left(
\mbG_{\Psi_\theta}^{(k)}
\mbM^{(k)\top}
\mbG_{\Phi_\theta}^{(k)}
\mbM^{(k)}
\right),
\end{aligned}
\]
with 
\begin{equation*}
    \mbG_{\Phi_{\theta}}^{(k)}
    \triangleq
    \mdE_{\mu_k}\!\left[
        \Phi_{\theta}^{(k)}(X)\Phi_{\theta}^{(k)}(X)^\top
    \right]
    \qquad
    \mbG_{\Psi_{\theta}}^{(k)}
    \triangleq
    \mdE_{\nu_k}\!\left[
        \Psi_{\theta}^{(k)}(Y)\Psi_{\theta}^{(k)}(Y)^\top
    \right].
\end{equation*}
Recall that
\[
\dmok
=
\cmok
-
\mathbf 1_{\mcX}\otimes\mathbf 1_{\mcY},
\qquad
[\cmok f](x)=\mdE[f(Y)\mid X=x],
\]
so that
\[
\innerp{\dmokt}{\dmok}_{\mathrm{HS}}
=
\innerp{\dmokt}{\cmok}_{\mathrm{HS}}
-
\innerp{
\dmokt
}{
\mathbf 1_{\mcX}\otimes\mathbf 1_{\mcY}
}_{\mathrm{HS}}.
\]
For $u\in L^2_{\mu_k}$ and $v\in L^2_{\nu_k}$,
\[
\innerp{u\otimes v}{\cmok}_{\mathrm{HS}}
=
\innerp{u}{\cmok v}_{\mu_k},
\]
and,
\[
\begin{aligned}
\innerp{u}{\cmok v}_{\mu_k}
&=
\int
u(x)\mdE[v(Y)\mid X=x]\,d\mu_k(x) \\
&=
\mdE\!\left[
u(X)\mdE[v(Y)\mid X]
\right] \\
&=
\mdE\!\left[
\mdE[u(X)v(Y)\mid X]
\right] \\
&=
\mdE_{\rho_k}[u(X)v(Y)] \\
& = \int u(x)v(y)\,d\rho_k(x,y).
\end{aligned}
\]
Therefore, 
\[
\begin{aligned}
    \innerp{\dmokt}{\cmok}_{\mathrm{HS}} & = \int q_k(x,y)\,d\rho_k(x,y) \\
    &  = \mdE_{\rho_k}\left[q_k(X,Y)\right]
\end{aligned}
\]
Similarly,
\[
\innerp{
u\otimes v
}{
\mathbf 1_{\mcX}\otimes\mathbf 1_{\mcY}
}_{\mathrm{HS}}
=
\innerp{u}{\mathbf 1_{\mcX}}_{\mu_k}
\innerp{v}{\mathbf 1_{\mcY}}_{\nu_k},
\]
which gives
\[
\begin{aligned}
    \innerp{
\dmokt
}{
\mathbf 1_{\mcX}\otimes\mathbf 1_{\mcY}
}_{\mathrm{HS}}
& =
\int q_k(x,y)\,d\mu_k(x)d\nu_k(y) \\
& = \mdE_{\mu_k \otimes \nu_k}\left[q_k(X,Y)\right]
\end{aligned}
\]
Hence
\[
\begin{aligned}
    \innerp{\dmokt}{\dmok}_{\mathrm{HS}}
 & =
\int q_k(x,y)\,d\rho_k(x,y)
-
\int q_k(x,y)\,d\mu_k(x)d\nu_k(y)\\
& = \mdE_{\rho_k}\left[q_k(X,Y)\right] - \mdE_{\mu_k \otimes \nu_k}\left[q_k(X,Y)\right].
\end{aligned}
\]
Defining
\[
\begin{aligned}
    \mbC_{\Phi_{\theta\!,c}\Psi_{\theta\!,c}}^{(k)}
    = \mdE_{\rho_k}\left[\Phi_{\theta,c}^{(k)}(X) \Psi_{\theta,c}^{(k)}(Y)^\top\right],
\end{aligned}
\]
we obtain
\[
\innerp{\dmokt}{\dmok}_{\mathrm{HS}}
=
\Tr\!\left(
\mbC_{\Phi_\theta\Psi_\theta}^{(k)\top}
\mbM^{(k)}
\right).
\]
It leads to the theoretical population loss 
\[
\mcL^{(k)}
=
\Tr\!\left(
\mbG_{\Psi_\theta}^{(k)}
\mbM^{(k)\top}
\mbG_{\Phi_\theta}^{(k)}
\mbM^{(k)}
\right)
-
2\Tr\!\left(
\mbC_{\Phi_\theta\Psi_\theta}^{(k)\top}
\mbM^{(k)}
\right).
\]

From an empirical perspective, given a set of observations $\textstyle \{(x^{(k)}_i,y^{(k)}_i)\}_{i \in [n_k]}$, we denote $\Phi_\theta(\mbX^{(k)}) = \{\Phi_\theta(x^{(k)}_i)\}_{i\in[n_k]}$ and $\Psi_\theta(\mbY^{(k)}) = \{\Psi_\theta(y^{(k)}_i)\}_{i\in[n_k]}$ the feature matrices. Considering, the matrices 
\begin{equation}
    \mbQ^{(k)} = \Phi_\theta(\mbX^{(k)})\mbA^{(k)}\mbB^{(k)\top}\Psi_\theta(\mbY^{(k)})^\top
    \quad 
    \text{and}
    \quad
    \mbH^{(k)} = \mbI_{n_k} - 1/n_k \mb1\mb1^\top,
\end{equation} 
the data-fitting loss $\mcL^{(k)}$  admits an unbiased empirical estimator defined by 
\begin{equation}
    \widehat{\mcL^{(k)}} \triangleq \frac{\|\mbQ^{(k)}\|^2_\mathrm{F} - \|\Diag(\mbQ^{(k)})\|^2_\mathrm{2}}{n_k(n_k-1)} - \frac{2}{n_k-1} \Tr(\mbH^{(k)}\mbQ^{(k)}).
\end{equation}
Here $\textstyle (\|\mbQ^{(k)}\|^2_\mathrm{F} - \|\Diag(\mbQ^{(k)})\|^2_\mathrm{2})/(n_k(n_k-1))$ is the unbiased U-statistic to estimate $\|\dmokt\|^2_{\mathrm{HS}}$.

\paragraph{Multi-task optimization.}
The factorization
$\mbM^{(k)}=\mbA^{(k)}\mbB^{(k)\top}$
is scale non-identifiable, since for every $c\neq0$,
\(
\mbA^{(k)}\mbB^{(k)\top}
=
(c\mbA^{(k)})(c^{-1}\mbB^{(k)})^\top.
\)
We therefore regularize both factors and optimize
\[
\min_{\theta,\{\mbA^{(k)},\mbB^{(k)}\}_{k=1}^K}
\sum_{k=1}^{K}
\left[
\widehat{\mcL^{(k)}}
+
\lambda
\left(
\|\mbA^{(k)}\|_{\mathrm F}^2
+
\|\mbB^{(k)}\|_{\mathrm F}^2
\right)
\right].
\]
In practice, this objective is optimized stochastically at two levels:
each update samples a mini-batch of systems, and then $L$ paired observations
within each selected system.

\begin{algorithm}[t]
\caption{MTL-CMO learning procedure}
\label{alg:mtlcmo}
\small
\begin{algorithmic}[1]

\Require Task datasets $\{\mathcal D_k\}_{k=1}^K$, system batch size $B$,
within-task sample size $L$, ranks $\{r_k\}_{k=1}^K$, regularization $\lambda$

\State Initialize shared dictionary parameters $\theta$ and
$\mbA^{(k)},\mbB^{(k)}\in\mathbb R^{d\times r_k}$, $k=1,\ldots,K$

\For{each optimization step}
    \State Sample $\mathcal B\subset\{1,\ldots,K\}$ with $|\mathcal B|=B$

    \For{each $k\in\mathcal B$}
        \State Sample $L$ paired observations
        $\{(x_i^{(k)},y_i^{(k)})\}_{i=1}^{L}$ from $\mathcal D_k$

        \State Form
        $\mbX^{(k)}\gets(x_i^{(k)})_{i=1}^{L}$ and
        $\mbY^{(k)}\gets(y_i^{(k)})_{i=1}^{L}$

        \State Evaluate
        $\mbW^{(k)}\gets\Phi_\theta(\mbX^{(k)})$ and
        $\mbZ^{(k)}\gets\Psi_\theta(\mbY^{(k)})$

        \State Set
        $\mbH_L\gets\mbI_L-\frac{1}{L}\mb1\mb1^\top$

        \State From 
        $\mbQ^{(k)}
        \gets
        \mbW^{(k)}
        \mbA^{(k)}\mbB^{(k)\top}
        \mbZ^{(k)\top}$

        \State Compute
        $\displaystyle
        \widehat{\mcL}^{(k)}
        \gets
        \frac{
        \|\mbQ^{(k)}\|_{\mathrm F}^{2}
        -
        \|\Diag(\mbQ^{(k)})\|_{2}^{2}
        }{L(L-1)}
        -
        \frac{2}{L-1}\Tr(\mbH_L\mbQ^{(k)})$
    \EndFor

    \State Aggregate
    $\displaystyle
    \widehat{\mcJ}_{\mathcal B}
    \gets
    \frac{1}{B}
    \sum_{k\in\mathcal B}
    \left[
    \widehat{\mcL}^{(k)}
    +
    \lambda
    \left(
    \|\mbA^{(k)}\|_{\mathrm F}^{2}
    +
    \|\mbB^{(k)}\|_{\mathrm F}^{2}
    \right)
    \right]$

    \State Back-propagate $\widehat{\mcJ}_{\mathcal B}$ and update
    $\theta$ and $\{\mbA^{(k)},\mbB^{(k)}\}_{k\in\mathcal B}$
\EndFor

\State \Return
$\Phi_\theta,\Psi_\theta$ and
$\{\mbM^{(k)}=\mbA^{(k)}\mbB^{(k)\top}\}_{k=1}^{K}$

\end{algorithmic}
\end{algorithm}

The matrix $\mbQ^{(k)}$ contains all input--output pairings within task $k$.
Its diagonal entries correspond to matched pairs from $\rho_k$, whereas its
off-diagonal entries cross distinct observations and provide the empirical
product-of-marginals term $\mu_k\otimes\nu_k$. For dynamical systems, the
diagonal retains the time-lagged transition pairing while the off-diagonal
terms break it.

\paragraph{Closed-form transfer.}
Once the shared dictionaries are fixed, transfer only requires optimizing
with respect to $\mbM$:
\[
\mcL(\mbM)
=
\Tr\!\left(
\mbG_{\Psi}
\mbM^\top
\mbG_{\Phi}
\mbM
\right)
-
2\Tr\!\left(
\mbC^\top\mbM
\right).
\]
Assume that $\mbG_{\Phi}\succ0$ and $\mbG_{\Psi}\succ0$. Since
\[
\Tr\!\left(
\mbG_{\Psi}
\mbM^\top
\mbG_{\Phi}
\mbM
\right)
=
\left\|
\mbG_{\Phi}^{1/2}
\mbM
\mbG_{\Psi}^{1/2}
\right\|_{\mathrm F}^2,
\]
we have
\[
\mcL(\mbM)
\geq
\lambda_{\min}(\mbG_{\Phi})
\lambda_{\min}(\mbG_{\Psi})
\|\mbM\|_{\mathrm F}^2
-
2\|\mbC\|_{\mathrm F}\|\mbM\|_{\mathrm F}.
\]
Thus $\mcL(\mbM)\to+\infty$ as $\|\mbM\|_{\mathrm F}\to\infty$.
The objective is continuous and coercive, hence admits a global minimizer.
Moreover, with $\mbm=\operatorname{vec}(\mbM)$,
\[
\nabla_{\mbm}^2\mcL
=
2\left(
\mbG_{\Psi}\otimes\mbG_{\Phi}
\right)\succ0,
\]
so the minimizer is unique. Finally,
\[
\nabla_{\mbM}\mcL
=
2\mbG_{\Phi}\mbM\mbG_{\Psi}
-
2\mbC,
\]
and the first-order condition gives
\[
\mbG_{\Phi}\mbM^\star\mbG_{\Psi}
=
\mbC,
\qquad
\boxed{
\mbM^\star
=
\mbG_{\Phi}^{-1}
\mbC
\mbG_{\Psi}^{-1}.
}
\]

For a new task, T-CMO therefore estimates only the target-specific operator;
the shared dictionaries are not retrained.

\begin{algorithm}[t]
\caption{T-CMO transfer procedure}
\label{alg:tcmo}
\small
\begin{algorithmic}[1]

\Require Target dataset
$\mathcal D^\star=\{(x_i^\star,y_i^\star)\}_{i=1}^{n}$,
pretrained dictionaries $\Phi_\theta,\Psi_\theta$,
target rank $r$, Gram regularization $\varepsilon$

\State Form
$\mbX^\star=(x_i^\star)_{i=1}^{n}$ and
$\mbY^\star=(y_i^\star)_{i=1}^{n}$

\State Evaluate
$\mbW^\star\gets\Phi_\theta(\mbX^\star)\in\mathbb R^{n\times d}$ and
$\mbZ^\star\gets\Psi_\theta(\mbY^\star)\in\mathbb R^{n\times d}$

\State Set
$\mbH_n\gets\mbI_n-\frac{1}{n}\mb1\mb1^\top$

\State Center
$\mbW_c^\star\gets\mbH_n\mbW^\star$ and
$\mbZ_c^\star\gets\mbH_n\mbZ^\star$

\State Compute
$\displaystyle
\widehat{\mbG}_{\Phi}^\star
\gets
\frac{1}{n}\mbW_c^{\star\top}\mbW_c^\star$,
$\displaystyle
\widehat{\mbG}_{\Psi}^\star
\gets
\frac{1}{n}\mbZ_c^{\star\top}\mbZ_c^\star$

\State Compute
$\displaystyle
\widehat{\mbC}_{\Phi\Psi}^\star
\gets
\frac{1}{n}\mbW_c^{\star\top}\mbZ_c^\star$

\State Regularize
$\widehat{\mbG}_{\Phi,\varepsilon}^\star
\gets
\widehat{\mbG}_{\Phi}^\star+\varepsilon\mbI_d$ and
$\widehat{\mbG}_{\Psi,\varepsilon}^\star
\gets
\widehat{\mbG}_{\Psi}^\star+\varepsilon\mbI_d$

\State Recover
$\displaystyle
\widehat{\mbM}^\star
\gets
\widehat{\mbG}_{\Phi,\varepsilon}^{\star-1}
\widehat{\mbC}_{\Phi\Psi}^\star
\widehat{\mbG}_{\Psi,\varepsilon}^{\star-1}$

\State Whiten
$\displaystyle
\widehat{\mbM}^\star_\mbW
\gets
\widehat{\mbG}_{\Phi,\varepsilon}^{\star 1/2}
\widehat{\mbM}^\star
\widehat{\mbG}_{\Psi,\varepsilon}^{\star 1/2}$

\State Compute the rank-$r$ SVD
$\widehat{\mbM}^\star_\mbW
=
\mbU_r^\star
\Diag(\bs\sigma^\star)
\mbV_r^{\star\top}$

\State Recover the singular functions from
$\widehat{\mbG}_{\Phi,\varepsilon}^{\star-1/2}\mbU_r^\star$
and
$\widehat{\mbG}_{\Psi,\varepsilon}^{\star-1/2}\mbV_r^\star$

\State \Return
$\widehat{\mbM}^\star_\mbW$ and its rank-$r$ singular representation

\end{algorithmic}
\end{algorithm}

\FloatBarrier
\input{sections/related_work}

\FloatBarrier
\section{Uncertainty Quantification Experiments}
\label{app:uq}

\subsection{Conditional-distribution families}
\label{app:uq_families}

We consider four families of related conditional distributions. For each
family, a task $k$ corresponds to a conditional law
$P_k(\mathrm y\mid \mbx)$, with
$\mbX\sim\mcU([-2,2]^D)$. Task dependence enters through the parameters of
the conditional law.

\noindent\textbf{(CD1) Univariate bimodal Gaussian mixture.}
For $D=1$,
\begin{equation*}
    Y^{(k)}\mid X=x
    \sim
    p_k\,\mcN\!\left(a_k\sin x,\sigma_k^2\right)
    +(1-p_k)\,\mcN\!\left(-a_k\sin x,\sigma_k^2\right),
\end{equation*}
where
\[
p_k\sim\mcU([0.2,0.8]),
\qquad
a_k\sim\mcU([0.6,1.0]),
\qquad
\sigma_k\sim\mcU([0.5,0.8]).
\]
Equivalently, the mixture component can be represented through a Bernoulli
variable $B^{(k)}\sim\mcB(p_k)$ and
$Z^{(k)}=2B^{(k)}-1\in\{-1,1\}$.

\noindent\textbf{(CD2) Multivariate bimodal Gaussian mixture.}
For $D=10$,
\begin{equation*}
    Y^{(k)}\mid \mbX=\mbx
    \sim
    p_k\,\mcN\!\left(
        a_k\sin(\mbw^\top\mbx),\sigma_k^2
    \right)
    +(1-p_k)\,\mcN\!\left(
        -a_k\sin(\mbw^\top\mbx),\sigma_k^2
    \right),
\end{equation*}
with the same distributions for $(p_k,a_k,\sigma_k)$ as in \textbf{CD1},
$\mbw\in\mdS^9$, and $\mbw$ shared across tasks.

\noindent\textbf{(CD3) Multivariate Student distribution.}
For $D=10$, let
\[
\mu_k(\mbx)
=
a_k\sin(\mbw^\top\mbx)
+
b_k\cos(\mbw^\top\mbx).
\]
We define
\begin{equation*}
    Y^{(k)}\mid\mbX=\mbx
    \sim
    t_{\nu_k}\!\left(\mu_k(\mbx),\sigma_k\right),
\end{equation*}
whose conditional density is
\begin{equation*}
p_k(y\mid\mbx)
=
\frac{
    \Gamma\!\left((\nu_k+1)/2\right)
}{
    \Gamma\!\left(\nu_k/2\right)
    \sqrt{\nu_k\pi}\,\sigma_k
}
\left[
1+
\frac{1}{\nu_k}
\left(
\frac{y-\mu_k(\mbx)}{\sigma_k}
\right)^2
\right]^{-(\nu_k+1)/2}.
\end{equation*}
The task parameters are
\[
a_k\sim\mcN(1,0.3^2),
\qquad
b_k\sim\mcN(0.3,0.2^2),
\qquad
\sigma_k\sim\mcN(0.5,0.1^2),
\qquad
\nu_k\sim\mcU([3,10]),
\]
 The direction
$\mbw\in\mdS^9$ is shared across tasks.

\noindent\textbf{(CD4) Multivariate skew-normal distribution.}
We use the same form
\[
\mu_k(\mbx)
=
a_k\sin(\mbw^\top\mbx)
+
b_k\cos(\mbw^\top\mbx),
\]
and define
\begin{equation*}
    Y^{(k)}\mid\mbX=\mbx
    \sim
    \operatorname{SN}\!\left(
        \mu_k(\mbx),\sigma_k,\alpha_k
    \right).
\end{equation*}
Writing $\phi$ and $\Phi$ for the standard Gaussian density and CDF,
respectively, its conditional density is
\begin{equation*}
p_k(y\mid\mbx)
=
\frac{2}{\sigma_k}
\phi\!\left(
    \frac{y-\mu_k(\mbx)}{\sigma_k}
\right)
\Phi\!\left(
    \alpha_k
    \frac{y-\mu_k(\mbx)}{\sigma_k}
\right).
\end{equation*}
We sample
\[
a_k\sim\mcN(1,0.15^2),
\qquad
b_k\sim\mcN(0.5,0.15^2),
\qquad
\sigma_k\sim\mcN(0.45,0.08^2),
\qquad
\alpha_k\sim\mcN(1.5,0.4^2),
\]
with $\sigma_k$ truncated below at $0.1$ and
$\mbw\in\mdS^9$ shared across tasks.

\Cref{fig:CD_density} displays representative conditional densities from
\textbf{CD1}, \textbf{CD3}, and \textbf{CD4}. We omit \textbf{CD2}, which
differs from \textbf{CD1} only through the higher-dimensional input and the
projection $\mbw^\top\mbx$.

\begin{figure}[t]
    \centering
    \includegraphics[width=\linewidth]{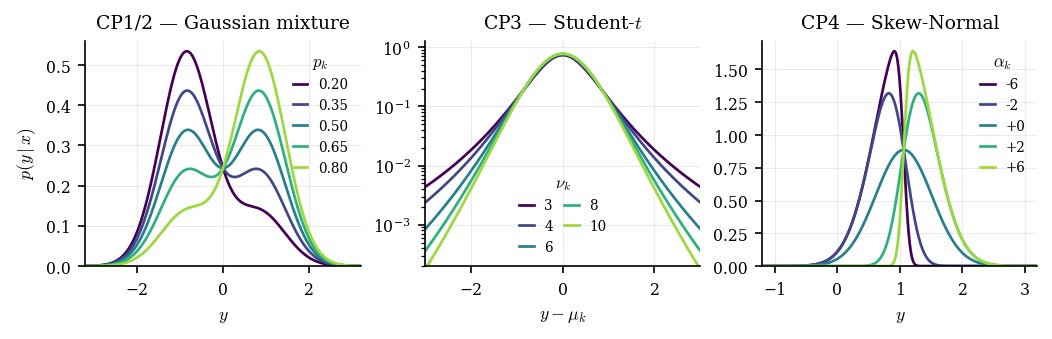}
    \caption{
    Representative conditional densities from \textbf{CD1}, \textbf{CD3},
    and \textbf{CD4} for different conditioning inputs and task parameters.
    }
    \label{fig:CD_density}
\end{figure}

\paragraph{Experimental protocol.}
For each family, we generate $K=100$ tasks with $n=400$ observations per
task. Task parameters are sampled once and kept fixed across repetitions,
while the observations are independently resampled. Results are averaged over
$10$ independent data realizations.

For transfer, we generate $100$ additional unseen tasks by independently
sampling new task parameters. For each target task, we vary the number of
available observations .
Transfer results are averaged over $10$ independent realizations of the target
data.
\subsection{Multi-task CMO architecture and training}
\label{app:uq_architecture}

For all four conditional-distribution families, we use the multi-task CMO
parametrization introduced . The two shared functional
spaces are represented by the neural dictionaries
\[
\Phi_\theta(x)
=
\{\phi_1^\theta(x),\ldots,\phi_d^\theta(x)\},
\qquad
\Psi_\theta(y)
=
\{\psi_1^\theta(y),\ldots,\psi_d^\theta(y)\}.
\]

For each task $k$, the corresponding deflated conditional operator is
parametrized as
\[
\dmokt
=
\msS_{\Phi_\theta}^{(k)}
\mbA^{(k)}\mbB^{(k)\top}
\msS_{\Psi_\theta}^{(k)*},
\qquad
\mbA^{(k)},\mbB^{(k)}
\in\mathbb R^{d\times r_k}.
\]
Hence, $\mbA^{(k)}\mbB^{(k)\top}$ defines a rank-$r_k$ factorization of the
task-specific operator in the shared functional coordinates. The neural
dictionaries are common to all tasks, while the low-rank factors
$\mbA^{(k)}$ and $\mbB^{(k)}$ capture the task-specific dependence structure.

\paragraph{Neural parametrization.}
The two dictionaries are parametrized by separate MLPs, both shared across
tasks. The input dictionary takes $x\in\mathbb R^D$, whereas the output
dictionary takes the scalar response $y\in\mathbb R$. Each hidden block is
composed of a linear layer followed by the activation function and dropout;
The dimension $d$ of the learned functional
space is given by the width of the last hidden layer. No decoder is used:
conditional statistics are obtained directly from the learned CMO.
The architecture is adapted to each conditional-distribution family. Table
\ref{tab:uq_architecture} reports the configurations used in the experiments.

\begin{table}[!htbp]
\centering
\footnotesize
\setlength{\tabcolsep}{5pt}
\renewcommand{\arraystretch}{0.95}
\caption{Neural architecture and operator rank used for the four
conditional-distribution families.}
\label{tab:uq_architecture}
\begin{tabular}{lccccc}
\toprule
& Hidden layers & $d$ & $r_k$ & Activation & Dropout \\
\midrule
CD1 & $3\times64$  & $64$  & $20$ & GELU & $0.15$ \\
CD2 & $2\times64$ & $64$ & $7$  & Tanh & $0.2$ \\
CD3 & $4\times64$  & $64$  & $18$ & GELU & $0.14$ \\
CD4 & $4\times64$  & $64$  & $13$ & Tanh & $0.16$ \\
\bottomrule
\end{tabular}
\end{table}

\paragraph{Optimization.}
For each task, the response variable $Y$ is standardized using its empirical
mean and standard deviation, while the input variables are left unchanged.
The shared dictionaries and the task-specific low-rank factors are optimized
jointly with AdamW. We use separate learning rates and weight-decay
coefficients for the shared neural parameters and for
$\{\mbA^{(k)},\mbB^{(k)}\}_{k=1}^K$.

\begin{table}[!htbp]
\centering
\footnotesize
\setlength{\tabcolsep}{4pt}
\renewcommand{\arraystretch}{0.95}
\caption{Optimization parameters for the uncertainty-quantification experiments.}
\label{tab:uq_training}
\resizebox{\linewidth}{!}{
\begin{tabular}{lcccccccc}
\toprule
& Epochs & Batch size
& $\mathrm{lr}_{\mathrm{shared}}$
& $\mathrm{lr}_{\mathrm{specific}}$
& $\mathrm{wd}_{\mathrm{shared}}$
& $\mathrm{wd}_{\mathrm{specific}}$
& Grad. clip
& Scheduler \\
\midrule
CD1
& $5400$ & $32$
& $5.8\times10^{-5}$
& $1.4\times10^{-4}$
& $2.2\times10^{-3}$
& $2.6\times10^{-2}$
& $5.0$
& None
\\
CD2
& $4200$ & $64$
& $2.7\times10^{-4}$
& $6.4\times10^{-5}$
& $3.9\times10^{-3}$
& $2.6\times10^{-4}$
& $2.0$
& None
\\
CD3
& $3600$ & $64$
& $1.2\times10^{-4}$
& $1.8\times10^{-3}$
& $2.2\times10^{-4}$
& $3.4\times10^{-6}$
& $1.0$
& None
\\
CD4
& $2600$ & $64$
& $3.0\times10^{-3}$
& $2.9\times10^{-3}$
& $9.1\times10^{-6}$
& $4.4\times10^{-3}$
& $2.0$
& Cosine
\\
\bottomrule
\end{tabular}
}
\end{table}

\subsection{Conditional CDF reconstruction}
\label{app:uq_cdf}

We now detail how a conditional CDF is recovered from the learned CMO.
Consider a task $k$ with observations
$\mathcal D_k=\{(x_j^{(k)},y_j^{(k)})\}_{j=1}^{n_k}$.

\paragraph{Centering and task-specific SVD.}
The outputs of the two neural dictionaries are centered independently for each
task. We define
\[
\overline{\Phi}^{(k)}
=
\frac{1}{n_k}\sum_{j=1}^{n_k}\Phi_\theta(x_j^{(k)}),
\qquad
\overline{\Psi}^{(k)}
=
\frac{1}{n_k}\sum_{j=1}^{n_k}\Psi_\theta(y_j^{(k)}),
\]
and
\[
\Phi_{\theta,c}^{(k)}(x)
=
\Phi_\theta(x)-\overline{\Phi}^{(k)},
\qquad
\Psi_{\theta,c}^{(k)}(y)
=
\Psi_\theta(y)-\overline{\Psi}^{(k)}.
\]
The corresponding empirical Gram matrices are
\[
\widehat{\mbG}_{\Phi_\theta}^{(k)}
=
\frac{1}{n_k}
\sum_{j=1}^{n_k}
\Phi_{\theta,c}^{(k)}(x_j^{(k)})^\top
\Phi_{\theta,c}^{(k)}(x_j^{(k)}),
\qquad
\widehat{\mbG}_{\Psi_\theta}^{(k)}
=
\frac{1}{n_k}
\sum_{j=1}^{n_k}
\Psi_{\theta,c}^{(k)}(y_j^{(k)})^\top
\Psi_{\theta,c}^{(k)}(y_j^{(k)}).
\]

During multi-task learning, task $k$ is represented in the shared functional
coordinates by the low-rank matrix $\mbA^{(k)}\mbB^{(k)\top}$. Its
orthonormal singular representation is recovered by whitening with the
task-specific Gram matrices and computing
\[
\widehat{\mbG}_{\Phi_\theta}^{(k)1/2}
\mbA^{(k)}\mbB^{(k)\top}
\widehat{\mbG}_{\Psi_\theta}^{(k)1/2}
=
\mbU^{(k)}
\Diag(\bs\sigma^{(k)})
\mbV^{(k)\top}.
\]
Denoting by $\mbu_i^{(k)}$ and $\mbv_i^{(k)}$ the $i$th columns of
$\mbU^{(k)}$ and $\mbV^{(k)}$, the task-specific singular functions are
\[
\widetilde u_i^{(k)}(x)
=
\Phi_{\theta,c}^{(k)}(x)
\widehat{\mbG}_{\Phi_\theta}^{(k)-1/2}
\mbu_i^{(k)},
\qquad
\widetilde v_i^{(k)}(y)
=
\Psi_{\theta,c}^{(k)}(y)
\widehat{\mbG}_{\Psi_\theta}^{(k)-1/2}
\mbv_i^{(k)}.
\]
The resulting rank-$r_k$ deflated operator is
\[
\dmokt
=
\sum_{i=1}^{r_k}
\sigma_i^{(k)}
\widetilde u_i^{(k)}
\otimes
\widetilde v_i^{(k)}.
\]

\paragraph{Conditional CDF.}
For scalar $Y$, we recover the conditional CDF by taking
$f_t(y)=\mathbf 1_{\{y\leq t\}}$, so that
\[
F_k(t\mid x)
=
\mathbb E\!\left[
f_t(Y^{(k)})\mid X^{(k)}=x
\right].
\]
We approximate expectations with respect to the output marginal $\nu_k$ using
the empirical measure
\[
\widehat{\nu}_k
=
\frac{1}{n_k}
\sum_{j=1}^{n_k}\delta_{y_j^{(k)}}.
\]
Applying the recovered CMO to $f_t$ gives
\[
\widehat F_k(t\mid x)
=
\frac{1}{n_k}
\sum_{j=1}^{n_k}
\mathbf 1_{\{y_j^{(k)}\leq t\}}
+
\frac{1}{n_k}
\sum_{j:y_j^{(k)}\leq t}
\sum_{i=1}^{r_k}
\sigma_i^{(k)}
\widetilde u_i^{(k)}(x)
\widetilde v_i^{(k)}(y_j^{(k)}).
\]
The first term is the empirical marginal CDF of $Y^{(k)}$, while the second
is the correction induced by the deflated conditional operator and introduces
the dependence on $x$.

For computation, we sort the outputs
$y_{(1)}^{(k)}\leq\cdots\leq y_{(n_k)}^{(k)}$ and define
\[
c_j^{(k)}(x)
=
\sum_{i=1}^{r_k}
\sigma_i^{(k)}
\widetilde u_i^{(k)}(x)
\widetilde v_i^{(k)}(y_{(j)}^{(k)}).
\]
At the $m$th ordered output,
\[
\widehat F_k(y_{(m)}^{(k)}\mid x)
=
\frac{m}{n_k}
+
\frac{1}{n_k}
\sum_{j=1}^{m}c_j^{(k)}(x).
\]
The full conditional CDF is therefore obtained by a cumulative sum over the
ordered outputs. For an arbitrary threshold $t$, all observations satisfying
$y_j^{(k)}\leq t$ are retained.

In practice, small negative empirical masses may arise from numerical estimation. We clip these masses to zero and renormalize them before taking the cumulative sum, ensuring that the resulting estimate is a valid conditional CDF.

\paragraph{Wasserstein evaluation.}
For each task $k$, we compare the estimated conditional CDF
$\widehat F_k(\cdot\mid x)$ with the analytical ground-truth CDF of the
corresponding data-generating process using the one-dimensional
Wasserstein-$1$ distance. Both CDFs are evaluated on a uniform grid of
$1000$ output values covering the relevant support of $Y$. For each task,
the distance is computed at $40$ conditioning values uniformly distributed
over the input range and averaged over these conditioning points. The
resulting scores are then averaged across tasks. All experiments are repeated
over $10$ independent random seeds.

\subsection{Transfer to unseen conditional distributions}
\label{app:uq_transfer}

We next evaluate whether the functional spaces learned jointly across source
tasks can be reused to estimate conditional distributions that were not
observed during multi-task training. For each conditional-distribution family,
we generate a new collection of target tasks independently of the source tasks
used to learn the shared dictionaries. The parameters of the target
distributions are drawn from the same task-generating distribution as during
multi-task training.

\paragraph{Transfer protocol.}
Given the learned dictionaries
$\Phi_{\hat\theta}$ and $\Psi_{\hat\theta}$, all their parameters are frozen.
For a new target task with
\[
\mathcal D_{\mathrm{tr}}
=
\{(x_j,y_j)\}_{j=1}^{n_{\mathrm{tr}}},
\]
we estimate only the operator associated with this task. The dictionary
outputs are centered using the target observations, and we compute the
empirical Gram and cross-covariance matrices
\[
\widehat{\mbG}_{\Phi}
=
\frac{1}{n_{\mathrm{tr}}}
\sum_{j=1}^{n_{\mathrm{tr}}}
\Phi_c(x_j)^\top\Phi_c(x_j),
\qquad
\widehat{\mbG}_{\Psi}
=
\frac{1}{n_{\mathrm{tr}}}
\sum_{j=1}^{n_{\mathrm{tr}}}
\Psi_c(y_j)^\top\Psi_c(y_j),
\]
and
\[
\widehat{\mbC}
=
\frac{1}{n_{\mathrm{tr}}}
\sum_{j=1}^{n_{\mathrm{tr}}}
\Phi_c(x_j)^\top\Psi_c(y_j).
\]
The target operator is then recovered in closed form as
\[
\mbM^\star
=
\widehat{\mbG}_{\Phi}^{-1}
\widehat{\mbC}
\widehat{\mbG}_{\Psi}^{-1}.
\]
We subsequently whiten $\mbM^\star$ with the target Gram matrices and compute
its truncated SVD, using the same rank as in the corresponding source model.
The resulting singular functions are used to reconstruct the target
conditional CDF as described in Section~\ref{app:uq_cdf}. No gradient-based
optimization or fine-tuning of the shared dictionaries is performed at
transfer time.

\paragraph{Transfer performance.}
Table~\ref{tab:uq_transfer} reports the Wasserstein-$1$ error obtained on
$100$ unseen target tasks as the number of target observations increases.
Results are averaged over $10$ independent target-data seeds.

\begin{table*}[t]
\centering
\caption{
Transfer to unseen conditional distributions.
Wasserstein-$1$ distance to the ground-truth conditional distribution
(mean $\pm$ standard deviation over $10$ target-data seeds), averaged over
$100$ unseen tasks. Lower is better. The last row of each block reports the
ratio between single-task NCP and T-CMO.
}
\label{tab:uq_transfer}

\resizebox{\linewidth}{!}{%
\begin{tabular}{lcccccccc}
\toprule
$n_{\mathrm{tr}}$
& 50 & 100 & 150 & 200 & 250 & 300 & 350 & 400 \\
\midrule

\multicolumn{9}{l}{\textbf{CD1}}\\
NCP
& \textbf{0.2362 $\pm$ 0.0059}
& \textbf{0.1739 $\pm$ 0.0025}
& \textbf{0.1455 $\pm$ 0.0038}
& \textbf{0.1258 $\pm$ 0.0021}
& \textbf{0.1152 $\pm$ 0.0015}
& \textbf{0.1062 $\pm$ 0.0021}
& 0.1004 $\pm$ 0.0018
& 0.0948 $\pm$ 0.0024 \\
T-CMO
& 0.2629 $\pm$ 0.0053
& 0.1869 $\pm$ 0.0015
& 0.1546 $\pm$ 0.0042
& 0.1315 $\pm$ 0.0012
& 0.1188 $\pm$ 0.0020
& 0.1072 $\pm$ 0.0022
& 0.1004 $\pm$ 0.0015
& \textbf{0.0943 $\pm$ 0.0015} \\
NCP / T-CMO
& 0.90 & 0.93 & 0.94 & 0.96 & 0.97 & 0.99 & 1.00 & 1.01 \\

\midrule
\multicolumn{9}{l}{\textbf{CD2}}\\
NCP
& \textbf{0.2326 $\pm$ 0.0082}
& \textbf{0.1826 $\pm$ 0.0041}
& \textbf{0.1704 $\pm$ 0.0016}
& 0.1588 $\pm$ 0.0015
& 0.1538 $\pm$ 0.0024
& 0.1500 $\pm$ 0.0007
& 0.1470 $\pm$ 0.0024
& 0.1458 $\pm$ 0.0028 \\
T-CMO
& 0.2913 $\pm$ 0.0052
& 0.2089 $\pm$ 0.0022
& 0.1747 $\pm$ 0.0023
& \textbf{0.1547 $\pm$ 0.0016}
& \textbf{0.1391 $\pm$ 0.0032}
& \textbf{0.1244 $\pm$ 0.0021}
& \textbf{0.1186 $\pm$ 0.0017}
& \textbf{0.1125 $\pm$ 0.0007} \\
NCP / T-CMO
& 0.80 & 0.87 & 0.98 & 1.03 & 1.11 & 1.21 & 1.24 & 1.30 \\

\midrule
\multicolumn{9}{l}{\textbf{CD3}}\\
NCP
& \textbf{0.2603 $\pm$ 0.0062}
& 0.2176 $\pm$ 0.0038
& 0.1999 $\pm$ 0.0011
& 0.1934 $\pm$ 0.0016
& 0.1889 $\pm$ 0.0012
& 0.1860 $\pm$ 0.0012
& 0.1830 $\pm$ 0.0005
& 0.1815 $\pm$ 0.0017 \\
T-CMO
& 0.2728 $\pm$ 0.0074
& \textbf{0.1990 $\pm$ 0.0039}
& \textbf{0.1624 $\pm$ 0.0034}
& \textbf{0.1447 $\pm$ 0.0009}
& \textbf{0.1317 $\pm$ 0.0029}
& \textbf{0.1198 $\pm$ 0.0018}
& \textbf{0.1128 $\pm$ 0.0009}
& \textbf{0.1067 $\pm$ 0.0015} \\
NCP / T-CMO
& 0.95 & 1.09 & 1.23 & 1.34 & 1.43 & 1.55 & 1.62 & 1.70 \\

\midrule
\multicolumn{9}{l}{\textbf{CD4}}\\
NCP
& 0.2584 $\pm$ 0.0038
& 0.2225 $\pm$ 0.0028
& 0.2049 $\pm$ 0.0012
& 0.1909 $\pm$ 0.0020
& 0.1797 $\pm$ 0.0033
& 0.1681 $\pm$ 0.0027
& 0.1587 $\pm$ 0.0026
& 0.1498 $\pm$ 0.0017 \\
T-CMO
& \textbf{0.1683 $\pm$ 0.0019}
& \textbf{0.1240 $\pm$ 0.0019}
& \textbf{0.1052 $\pm$ 0.0014}
& \textbf{0.0923 $\pm$ 0.0012}
& \textbf{0.0849 $\pm$ 0.0009}
& \textbf{0.0794 $\pm$ 0.0008}
& \textbf{0.0742 $\pm$ 0.0001}
& \textbf{0.0719 $\pm$ 0.0005} \\
NCP / T-CMO
& 1.54 & 1.79 & 1.95 & 2.07 & 2.12 & 2.12 & 2.14 & 2.08 \\

\bottomrule
\end{tabular}%
}
\end{table*}
The benefit of transferring the shared functional spaces depends strongly on
the conditional-distribution family. For CP1, whose conditional structure is
comparatively simple, T-CMO remains close to single-task NCP and the two
methods reach essentially the same accuracy for large target sample sizes.
The advantage of transfer becomes more pronounced for the higher-dimensional
and more structured families. On CP2, T-CMO improves over single-task NCP from
$n_{\mathrm{tr}}=200$ onward, while on CP3 the improvement appears already at
$n_{\mathrm{tr}}=100$ and increases steadily with the amount of target data.
The strongest effect is observed for CP4, where T-CMO consistently
outperforms single-task NCP and achieves approximately a two-fold reduction
in Wasserstein error for moderate and large target sample sizes.

These results suggest that the shared functional spaces are particularly
useful when the conditional distribution is difficult to recover from an
individual task alone. Multi-task learning provides a representation adapted
to the common structure of the task family, so that a new conditional
distribution can be identified by estimating only its task-specific operator.
For simpler conditional families, where a single-task model can already learn
an accurate representation from relatively few observations, the benefit of
sharing across tasks is correspondingly smaller.

\paragraph{Sensitivity to the number of tasks and samples.}
We finally investigate how transfer performance depends on the amount of
information available during multi-task learning. We focus on \textbf{CD4}
and vary either the number of observations per source task $n$, while fixing
$K=200$, or the number of source tasks $K$, while fixing $n=200$.
Figure~\ref{fig:uq_transfer_sweeps} shows the corresponding trends, while
Table~\ref{tab:CD4_sweeps} reports the same results numerically.

\begin{figure}[t]
    \centering
    \includegraphics[width=\linewidth]{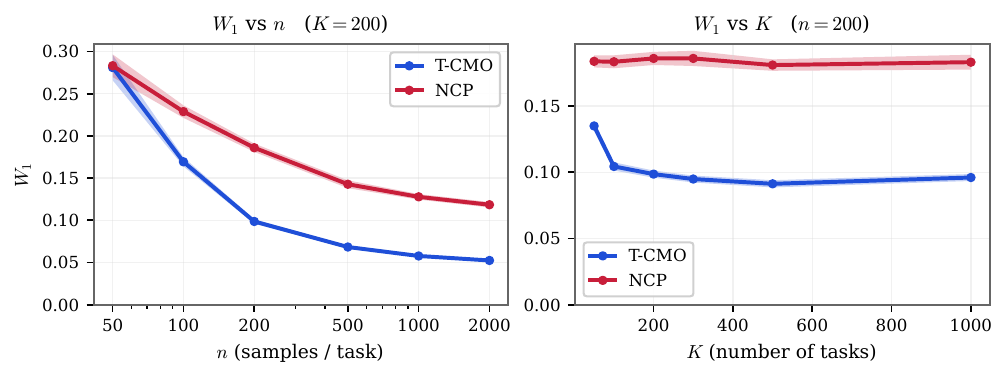}
    \caption{
    Sensitivity of transfer performance on \textbf{CD4}.
    \textbf{Left:} Wasserstein-$1$ error as a function of the number of
    observations per source task $n$, with $K=200$.
    \textbf{Right:} Wasserstein-$1$ error as a function of the number of
    source tasks $K$, with $n=200$.
    }
    \label{fig:uq_transfer_sweeps}
\end{figure}

\begin{table}[t]
\centering
\caption{
Sensitivity of transfer performance on \textbf{CD4}.
Wasserstein-$1$ distance to the ground-truth conditional distribution.
\emph{Gain} denotes the relative reduction in $W_1$ of T-CMO with respect
to single-task NCP. Lower is better.
}
\label{tab:CD4_sweeps}
\begin{tabular}{lccc}
\toprule
 & T-CMO \textbf{(ours)} & NCP & Gain \\
\midrule
\multicolumn{4}{l}{\textit{Varying $n$ (samples per task), $K=200$}} \\
$n=50$
& 0.281 $\pm$ 0.015
& 0.283 $\pm$ 0.013
& 1\% \\
$n=100$
& 0.169 $\pm$ 0.005
& 0.229 $\pm$ 0.007
& 26\% \\
$n=200$
& 0.099 $\pm$ 0.003
& 0.186 $\pm$ 0.005
& 47\% \\
$n=500$
& 0.068 $\pm$ 0.002
& 0.143 $\pm$ 0.005
& 52\% \\
$n=1000$
& 0.058 $\pm$ 0.001
& 0.128 $\pm$ 0.003
& 55\% \\
$n=2000$
& 0.052 $\pm$ 0.001
& 0.118 $\pm$ 0.003
& 56\% \\
\midrule
\multicolumn{4}{l}{\textit{Varying $K$ (number of tasks), $n=200$}} \\
$K=50$
& 0.135 $\pm$ 0.004
& 0.184 $\pm$ 0.005
& 27\% \\
$K=100$
& 0.104 $\pm$ 0.003
& 0.183 $\pm$ 0.005
& 43\% \\
$K=200$
& 0.099 $\pm$ 0.003
& 0.186 $\pm$ 0.005
& 47\% \\
$K=300$
& 0.095 $\pm$ 0.003
& 0.186 $\pm$ 0.006
& 49\% \\
$K=500$
& 0.091 $\pm$ 0.003
& 0.181 $\pm$ 0.004
& 50\% \\
$K=1000$
& 0.096 $\pm$ 0.003
& 0.183 $\pm$ 0.006
& 48\% \\
\bottomrule
\end{tabular}
\end{table}

Increasing the number of observations per source task substantially improves
T-CMO. At $n=50$, T-CMO and single-task NCP perform almost identically,
whereas the relative reduction in Wasserstein error reaches $47\%$ at
$n=200$ and $56\%$ at $n=2000$. Increasing the number of source tasks has a
similar effect: the gain rises from $27\%$ at $K=50$ to approximately $50\%$
for $K\geq300$, after which the performance largely saturates. In contrast,
single-task NCP remains essentially insensitive to $K$, since it does not
exploit information from the additional source tasks. Overall, these results
show that transfer improves as the shared functional spaces are learned from
either more observations per task or a richer collection of related
conditional distributions, with diminishing returns once the common
representation is sufficiently well identified.

\subsection{Baselines}
\label{app:uq_baselines}

We compare MTL-CMO against both multi-task and single-task conditional
distribution estimators. All methods are trained from the same observations
and evaluated using the same conditioning inputs, output grid, and
Wasserstein-$1$ metric described above.

\paragraph{Multi-task baselines.}
\textbf{Pooled-NCP} is a pooled version of Neural Conditional Probability
(NCP)~\citep{kostic2024ncp}. Samples from all source tasks are aggregated and
a single NCP model is trained without access to the task identity. This
baseline therefore exploits the larger pooled dataset, but does not model
task-specific operators.

\textbf{MTL-MDN} combines hard parameter sharing
\citep{caruana1997multitask} with Mixture Density Networks
\citep{bishop1994mixture}. A common neural backbone is shared by all tasks and
each task has its own Gaussian-mixture output head. We use two mixture
components for CD1 and four components for CD2--CD4.

\textbf{DeepJMQR} follows the joint multi-quantile regression approach of
~\citep{rodrigues2020beyond}. We use a shared neural representation together
with a task embedding and predict $99$ conditional quantiles, trained with
the pinball loss. The conditional CDF used for evaluation is recovered from
the predicted quantiles.

\paragraph{Single-task baselines.}
\textbf{NCP}~\citep{kostic2024ncp} is trained independently on every task
and therefore provides the direct task-specific counterpart of our
multi-task operator model. Each model only observes samples from its
corresponding task.

\textbf{Conditional Flow Matching (CFM)} is the flow-matching
objective of \citep{lipman2023flow}. A conditional velocity field is trained
independently for each task and used to transport samples from the reference
distribution to the target conditional distribution. Conditional CDFs are
obtained from generated samples.

\textbf{Mixture Density Network (MDN)}~\citep{bishop1994mixture} directly
parameterizes $p(y\mid x)$ as a Gaussian mixture whose parameters are
predicted by a neural network. As for MTL-MDN, we use two Gaussian components
for CD1 and four components for CD2--CD4.

\textbf{Neural Spline Flow (NSF)}~\citep{durkan2019neural} models the
conditional distribution using monotone rational-quadratic spline
transformations. We use three spline transforms with eight bins and recover
the conditional CDF from samples generated by the fitted flow.

\textbf{Engression}~\citep{shen2025engression} is a neural distributional
regression method trained using the energy score. We use the authors'
implementation with input/output standardization and obtain conditional CDFs
from samples generated by the fitted model.

\textbf{NGBoost}~\citep{duan2020ngboost} performs probabilistic gradient
boosting using the natural gradient. We use a Gaussian predictive
distribution, $100$ boosting estimators, learning rate $0.1$, and
depth-$2$ regression trees as base learners.

\textbf{FlexCode}~\citep{izbicki2017converting} represents the conditional
density in an orthogonal series whose coefficients are estimated through
regression. We use a Fourier basis, with Lasso regression for CD1--CD2 and
Random Forest regression for CD3--CD4.

\textbf{Nadaraya--Watson (NW)} is a non-parametric kernel baseline based on
the classical Nadaraya--Watson estimator
\citep{nadaraya1964estimating,watson1964smooth}. 

\paragraph{Model selection and evaluation.}
Hyperparameters are selected independently for each conditional-distribution
family using held-out validation data; target evaluation tasks are not used
for model selection. For neural baselines, architecture, learning rate,
regularization, and method-specific parameters are selected using the same
validation criterion. For Nadaraya--Watson, only the input-kernel bandwidth
is selected. Generative baselines are evaluated from their generated
conditional samples, whereas methods providing an analytical or empirical
CDF are evaluated directly. All methods are finally compared on exactly the
same conditioning points and output grid.

\FloatBarrier

\section{Langevin Experiments}
\label{app:langevin}

\subsection{Two-dimensional Müller--Brown family}
\label{app:langevin_2d}

\paragraph{Potential family.}
We consider the parametric Müller--Brown family introduced in the main text,
derived from the standard Müller--Brown potential, a widely used benchmark for
metastable dynamics and transition-path problems
~\citep{devergne2024biased,letreut2025markov}.It corresponds to a two-dimensional overdamped Langevin dynamics,
\begin{equation}
d\mbX_t
=
-\nabla V_{r,\theta}(\mbX_t)\,dt
+
\sqrt{2\beta^{-1}}\,d\mbW_t,
\qquad \beta=1,
\label{eq:mb_langevin}
\end{equation}
where $\mbW_t$ is a standard two-dimensional Brownian motion. For
$\mbx=(x_1,x_2)\in\mathbb R^2$, the potential is
\begin{equation}
V_{r,\theta}(\mbx)
=
\frac{1}{s}
\sum_{j=1}^{4}
w_j(r)A_j
\exp\!\left[
(\mbx-\bm\mu_j)^\top
\mbH_j^{(\theta)}
(\mbx-\bm\mu_j)
\right]
+
c_{\mathrm{wall}}
\left[(x_1-x_c)^4+(x_2-y_c)^4\right].
\end{equation}
Here $\mbw(r)=(1,1,r,1)$,
$\mbH_3^{(\theta)}=\mbR_\theta\mbH_3\mbR_\theta^\top$, and
$\mbH_j^{(\theta)}=\mbH_j$ for $j\neq3$, with
$\mbH_j=\left(\begin{smallmatrix}a_j & b_j/2\\ b_j/2 & c_j\end{smallmatrix}\right)$.
The canonical Müller--Brown coefficients are
$\mbA=(-200,-100,-170,15)$,
$\bm\mu_1=(1,0)$, $\bm\mu_2=(0,0.5)$,
$\bm\mu_3=(-0.5,1.5)$, $\bm\mu_4=(-1,1)$,
$\mba=(-1,-1,-6.5,0.7)$,
$\mbb=(0,0,11,0.6)$, and
$\mbc=(-10,-10,-6.5,0.7)$.

We use $s=35$ and sample
\(
r\sim\mathcal U([0.75,1.15]),
\
\theta\sim\mathcal U([-\pi/4,\pi/12]).
\)
The parameter $r$ modifies the amplitude of the third Müller--Brown component,
whereas $\theta$ rotates its local quadratic form over a $60^\circ$ range,
while the other three components remain fixed.

A fixed quartic wall is added for numerical confinement, with
$c_{\mathrm{wall}}=0.5$ and
$(x_c,y_c)=(-0.25,0.875)$. Its contribution is negligible in the metastable
basins and increases rapidly away from the region of interest, preventing
trajectories from exploring numerically irrelevant regions.

\begin{figure}[!htbp]
    \centering
    \includegraphics[width=0.95\linewidth]{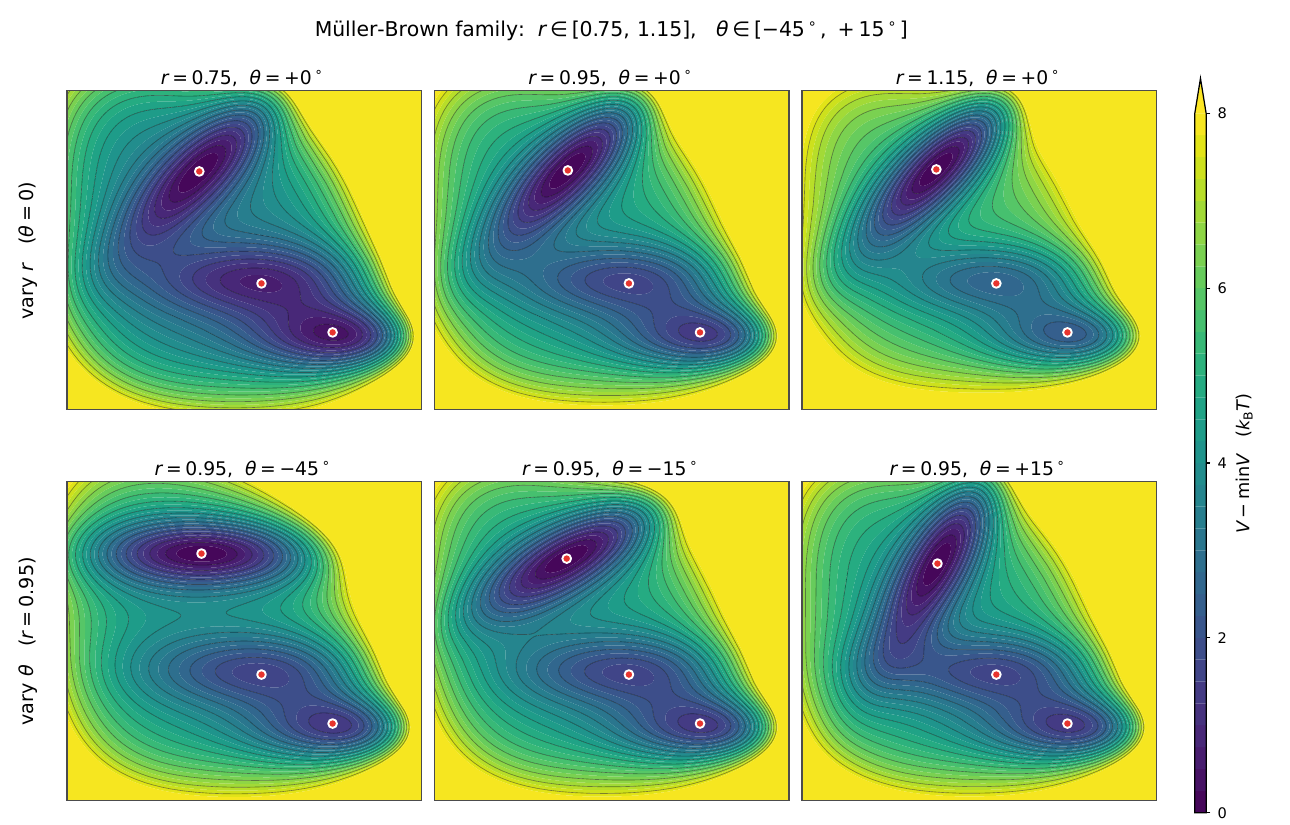}
    \caption{
    \textbf{Parametric Müller--Brown family.}
    Top: varying $r$ at fixed $\theta=0^\circ$ modifies the relative depth of
    the third metastable basin while largely preserving its orientation.
    Bottom: varying $\theta$ at fixed $r=0.95$ rotates the local geometry of
    the same basin over a $60^\circ$ range.
    }
    \label{fig:mb_family}
\end{figure}

Figure~\ref{fig:mb_family} illustrates the complementary effects of the two
task parameters. Varying $r$ primarily changes the energetic importance of the
third basin relative to the other metastable states. In contrast, varying
$\theta$ rotates its anisotropic quadratic structure, modifying the
orientation of the local level sets and of the corresponding drift field.
The family therefore preserves the same global three-basin topology while
inducing controlled task-dependent changes in the dynamics and associated
spectral structure.

\noindent
\textbf{Simulation and dataset.}
For each $(r,\theta)$, we simulate the overdamped Langevin dynamics at a sampling rate $\Delta t_{\mathrm{sim}}=10^{-3}$ and down-sampled the trajectories to the sampling rate $\Delta t_{\mathrm{obs}}=10^{-2}$. We discard the first $5{,}000$ integration steps and retain $80{,}000$
observations for each system.

We generate $950$ systems in total: $750$ source tasks are used for multi-task
learning, $50$ tasks are reserved for model selection, and the remaining
$150$ systems are kept unseen for transfer evaluation. Performance is
evaluated from trajectory prefixes
\(
N\in\{2,5,10,20,40,80\}\times10^3.
\)
All model-selection choices for the competing methods are made on the
$50$ validation tasks, separately for each trajectory length $N$.

\noindent
\textbf{Reference spectrum.}
For each task, reference eigenpairs are obtained by finite-difference
discretization of the backward Langevin generator on $[-3,3]^2$ using a
$240\times240$ grid, corresponding to $\Delta x=0.025$.
We retain the first three non-trivial eigenpairs
$(\lambda_j,\psi_j)_{j=1}^{3}$.
To avoid numerical instabilities in regions far outside the relevant part of
the energy landscape, the potential used to construct the discrete generator
is clipped at $V_{\min}+300$.

The stationary density is
\(
\pi_{r,\theta}(\mathbf{x})
\propto
\exp\!\left[-\beta V_{r,\theta}(\mathbf{x})\right].
\)
Reference and estimated eigenfunctions are normalized in
$L^2(\pi_{r,\theta})$, and mode-wise recovery is measured using the
sign-invariant cosine error
$1-\cos_{\pi}(\widehat{\psi}_j,\psi_j)$.
We additionally report eigenspace and eigenvalue diagnostics below.

\subsection{Shared representation and operator estimation}
\label{app:mb_operator_learning}

\paragraph{Multi-task representation learning.}
We instantiate the multi-task learning procedure of
\Cref{sec:method} with a single shared map
$\Phi_\theta=\Psi_\omega=\phi_\theta$, exploiting the reversibility of the
Langevin dynamics. The map
$\phi_\theta:\mathbb R^2\rightarrow\mathbb R^{64}$ is parameterized by a
three-layer MLP with hidden width $576$ and SiLU activations. During the
multi-task stage, each system is associated with a rank-$20$ operator
\(
    M^{(k)} = A_kB_k^\top,
    \
    A_k,B_k\in\mathbb R^{64\times20}.
\)
The final shared representation is learned on the $800$ training systems.
Training uses consecutive states, corresponding to $\tau=0.01$, with temporal
windows of length $2176$. The remaining optimization settings are summarized
in Table~\ref{tab:mb_training}.

\noindent
\textbf{System-specific operator estimation.}
For each unseen system, we apply T-CMO using the closed-form transfer procedure
of \Cref{alg:tcmo}, with adaptation rank $20$ and no gradient updates.
The resulting representation is then passed to the resolvent estimator
described in Appendix~\ref{app:resolvent}, from which we retain the three
slowest non-trivial modes for evaluation.

\noindent
\textbf{Spectral-rank diagnostic.}
We examine the reduced-rank reconstruction error of the whitened lagged
operator as a function of the retained rank. As shown in
Fig.~\ref{fig:mb_rrr_rank}, the error decreases sharply over the first three
components and then rapidly saturates. This provides an empirical diagnostic
for the low-dimensional spectral structure captured by the first three modes.

\begin{figure}[ht]
    \centering
    \includegraphics[width=0.52\linewidth]{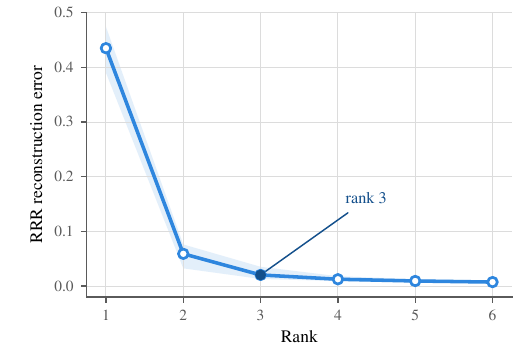}
    \caption{
    \textbf{Spectral-rank diagnostic.}
    Mean reduced-rank reconstruction error of the whitened lagged operator as
    a function of the retained rank. Most of the reduction occurs within the
    first three components, after which the error rapidly saturates.
    }
    \label{fig:mb_rrr_rank}
\end{figure}

\paragraph{Geometry of the learned operators.}
We next examine whether the task-specific operators learned in the shared
functional space retain the structure of the underlying Müller--Brown family.
Each recovered rank-$3$ operator is represented in the common
$\phi_\theta$ basis, and pairwise operator distances are computed using
SGOT~\citep{germain2026spectral}. We apply t-SNE to the resulting distance
matrix and visualize the induced geometry in
\Cref{fig:mb_sgot_tsne}. Importantly, neither $r$ nor $\theta$ is provided
during learning. Nevertheless, the embedding exhibits smooth gradients with
respect to both parameters: operators associated with nearby values of $r$ or
$\theta$ occupy neighboring regions. This shows that the learned operators
retain the structured variation of the underlying dynamical family without
supervision from its physical parameters.

\begin{figure}[!t]
    \centering
    \includegraphics[width=0.72\linewidth]{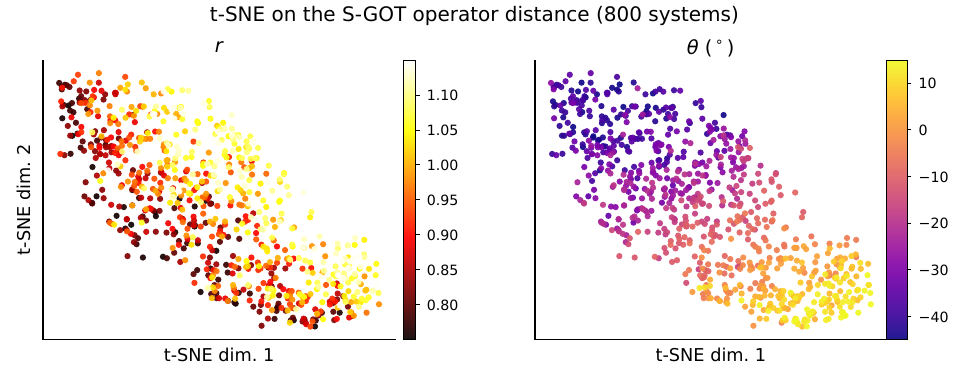}
    \caption{
    \textbf{Geometry of the learned Müller--Brown operators.}
    t-SNE of pairwise S-GOT distances between operators learned during the
    multi-task stage, colored by $r$ (left) and $\theta$ (right).
    }
    \label{fig:mb_sgot_tsne}
\end{figure}

\paragraph{Operator estimation accuracy.}
We next assess whether sharing the functional representation improves
operator estimation on the systems used during multi-task learning.
\Cref{fig:mb_mtl_vs_single} compares MTL-CMO with independently trained
single-task CMO models. Most systems lie below the diagonal, showing that the
shared representation improves eigenfunction recovery across the task family
rather than on only a small subset of systems. The improvement is observed
across all three modes and is strongest for the more difficult third mode
$\psi_3$.

\begin{figure}[!t]
    \centering
    \includegraphics[width=\linewidth]{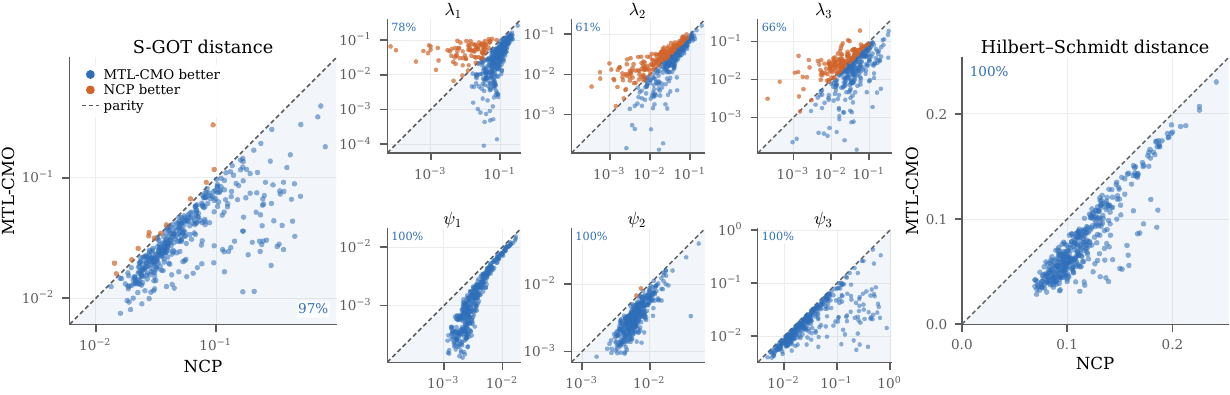}
    \caption{
    \textbf{Multi-task versus single-task CMO learning.}
    Eigenfunction errors of MTL-CMO versus independently trained single-task
    CMO models. Points below the diagonal favor MTL-CMO.
    }
    \label{fig:mb_mtl_vs_single}
\end{figure}
\begin{table}[!htbp]
\centering
\caption{
Architecture, training, and spectral-estimation configuration for the
Müller--Brown experiment.
}
\label{tab:mb_training}
\small
\begin{tabular}{lc}
\toprule
\textbf{Configuration} & \textbf{Value} \\
\midrule
\multicolumn{2}{l}{\textit{Architecture}} \\
Input dimension & $2$ \\
Shared feature dimension $d$ & $64$ \\
Hidden layers & $3\times576$ \\
Activation & SiLU \\
Task-specific rank & $20$ \\
\midrule
\multicolumn{2}{l}{\textit{Multi-task training}} \\
Systems & $800$ \\
Batch size & $40$ \\
Window length & $2176$ \\
Lag & $1$ step  \\
Epochs & $450$ \\
Optimizer & AdamW \\
Network learning rate & $3.117\times10^{-4}$ \\
Operator learning rate & $1.780\times10^{-4}$ \\
Weight decay & $4.49\times10^{-7}$ \\
$\lambda$ & $3.952\times10^{-3}$ \\
\midrule
\multicolumn{2}{l}{\textit{Spectral estimation}} \\
Adaptation rank & $20$ \\
Spectral rank & $3$ \\
Resolvent shift $\mu$ & $4.2$ \\
Maximum lag & $10{,}000$ \\
Regularization & $3.339\times10^{-8}$ \\
\bottomrule
\end{tabular}
\end{table}

\FloatBarrier

\FloatBarrier
\subsection{Transfer learning on unseen systems}
\label{app:mb_transfer}

\paragraph{Experimental protocol.}
We evaluate spectral recovery on the $150$ unseen Müller--Brown systems.
T-CMO adapts the shared representation to each target system using the
closed-form procedure of \Cref{alg:tcmo}, with adaptation rank $20$ and no
gradient updates. Spectral quantities are then extracted with the
resolvent-based estimator, and we retain the three slowest non-trivial modes, see Appendix~\ref{app:tcmo-op-est}.
Performance is evaluated from trajectory prefixes
$N\in\{2,5,10,20,40,80\}\times10^3$.
All methods receive exactly the same target prefixes; methods requiring
task-specific estimation are refitted from the corresponding prefix, whereas
transfer and meta-learning methods reuse their source-stage representation
according to their respective procedures. Hyperparameters are selected on the
$50$ validation systems separately for each trajectory length.

\paragraph{Baselines.}
We compare against six complementary approaches.
\textbf{Single-task NCP}~\citep{kostic2024ncp} learns an independent
representation from each target trajectory and therefore uses no information
from the source systems.
\textbf{Pooled NCP}~\citep{kostic2024ncp} trains a single model on the
concatenated source trajectories without task identities, testing whether
pooling alone provides a transferable representation.
\textbf{RFF-EDMD} combines fixed random Fourier features
\citep{rahimi2007random} with direct EDMD-type operator estimation
\citep{kostic2022learning}, while \textbf{RFF + resolvent} uses the same
features with the resolvent-based spectral estimator
\citep{kostic2025laplace}.
\textbf{MetaKoopman}~\citep{iwata2021meta} explicitly meta-learns Koopman
spectral information across related short time series, and
\textbf{VAMPNet}~\citep{mardt2018vampnets} learns dominant dynamical modes
through a neural variational objective.

\paragraph{Quantitative spectral recovery.}
The mode-wise results show where this improvement arises.
For the eigenfunctions (\Cref{tab:l2d_psi}), T-CMO attains the lowest mean
error in $17$ of the $18$ mode--trajectory combinations. The only exception
is $\psi_3$ at $N=5{,}000$, where it is second to VAMPNet by $3.1\times10^{-3}$,
well within the reported across-system variability. The first two modes become
accurate for several methods as $N$ increases, whereas $\psi_3$ remains
substantially harder and provides the clearest separation between approaches.
For the eigenvalues (\Cref{tab:l2d_lambda}), T-CMO achieves the lowest reported
mean error in every column, up to a tie at displayed precision for
$\lambda_1$ at $N=40{,}000$. Thus, transfer improves both the spatial
eigenfunctions and the associated relaxation rates.

\begin{table}[t]
\centering
\scriptsize
\caption{
Eigenfunction errors for the first three non-trivial modes on the $150$ unseen
Müller--Brown systems. Values are mean $\pm$ standard deviation.
Lower is better; \textbf{best} and \underline{second best} per column.
}
\label{tab:l2d_psi}
\setlength{\tabcolsep}{3.5pt}
\renewcommand{\arraystretch}{0.94}
\resizebox{\linewidth}{!}{
\begin{tabular}{lcccccc}
\toprule
& \multicolumn{6}{c}{\textbf{Trajectory length} $N$} \\
\cmidrule(lr){2-7}
Method
& \textbf{2\,000} & \textbf{5\,000} & \textbf{10\,000}
& \textbf{20\,000} & \textbf{40\,000} & \textbf{80\,000} \\
\midrule

\multicolumn{7}{l}{\textit{$\psi_1$}} \\
T-CMO \textbf{(ours)}
& $\mathbf{0.0731 \pm 0.0966}$
& $\mathbf{0.0294 \pm 0.0414}$
& $\mathbf{0.0151 \pm 0.0206}$
& $\mathbf{0.0079 \pm 0.0095}$
& $\mathbf{0.0040 \pm 0.0053}$
& $\mathbf{0.0021 \pm 0.0026}$ \\
RFF-EDMD
& $0.0982 \pm 0.1537$ & $0.0333 \pm 0.0427$
& $0.0248 \pm 0.0821$ & $0.0098 \pm 0.0098$
& $0.0058 \pm 0.0054$ & $0.0037 \pm 0.0026$ \\
MetaKoopman
& $0.0837 \pm 0.1154$ & $0.0327 \pm 0.0424$
& $0.0174 \pm 0.0209$ & $0.0099 \pm 0.0098$
& $0.0058 \pm 0.0055$ & $0.0040 \pm 0.0029$ \\
Single-task NCP
& $0.0839 \pm 0.1071$ & $0.0332 \pm 0.0415$
& $0.0177 \pm 0.0206$ & $0.0101 \pm 0.0095$
& $0.0059 \pm 0.0054$ & $0.0041 \pm 0.0026$ \\
RFF + resolvent
& $\underline{0.0781 \pm 0.0896}$ & $0.0329 \pm 0.0421$
& $0.0170 \pm 0.0206$ & $0.0093 \pm 0.0095$
& $0.0052 \pm 0.0054$ & $0.0033 \pm 0.0026$ \\
Pooled NCP
& $0.1013 \pm 0.1298$ & $0.0332 \pm 0.0425$
& $\underline{0.0163 \pm 0.0206}$
& $\underline{0.0083 \pm 0.0095}$
& $\underline{0.0041 \pm 0.0054}$
& $\underline{0.0022 \pm 0.0026}$ \\
VAMPNet
& $0.0900 \pm 0.1310$ & $\underline{0.0308 \pm 0.0426}$
& $0.0164 \pm 0.0213$ & $0.0084 \pm 0.0098$
& $0.0042 \pm 0.0054$ & $0.0023 \pm 0.0026$ \\

\midrule
\multicolumn{7}{l}{\textit{$\psi_2$}} \\
T-CMO \textbf{(ours)}
& $\mathbf{0.1174 \pm 0.1652}$
& $\mathbf{0.0352 \pm 0.0350}$
& $\mathbf{0.0173 \pm 0.0133}$
& $\mathbf{0.0089 \pm 0.0063}$
& $\mathbf{0.0046 \pm 0.0032}$
& $\mathbf{0.0022 \pm 0.0014}$ \\
RFF-EDMD
& $0.1607 \pm 0.1744$ & $0.0710 \pm 0.1234$
& $0.0453 \pm 0.1162$ & $0.0224 \pm 0.0567$
& $0.0085 \pm 0.0046$ & $0.0054 \pm 0.0021$ \\
MetaKoopman
& $0.1722 \pm 0.2028$ & $0.0606 \pm 0.0855$
& $0.0361 \pm 0.0745$ & $0.0289 \pm 0.0972$
& $0.0144 \pm 0.0323$ & $0.0079 \pm 0.0037$ \\
Single-task NCP
& $0.1890 \pm 0.2045$ & $0.0723 \pm 0.1072$
& $0.0378 \pm 0.0738$ & $0.0235 \pm 0.0734$
& $0.0149 \pm 0.0656$ & $0.0065 \pm 0.0025$ \\
RFF + resolvent
& $0.2013 \pm 0.2181$ & $0.0761 \pm 0.1102$
& $0.0426 \pm 0.0887$ & $0.0295 \pm 0.1061$
& $0.0195 \pm 0.0932$ & $0.0053 \pm 0.0023$ \\
Pooled NCP
& $0.2567 \pm 0.2026$ & $0.0890 \pm 0.1145$
& $0.0478 \pm 0.0982$ & $0.0349 \pm 0.1270$
& $0.0191 \pm 0.1039$ & $\underline{0.0036 \pm 0.0023}$ \\
VAMPNet
& $\underline{0.1559 \pm 0.2062}$
& $\underline{0.0446 \pm 0.0417}$
& $\underline{0.0235 \pm 0.0158}$
& $\underline{0.0144 \pm 0.0084}$
& $\underline{0.0081 \pm 0.0044}$
& $0.0061 \pm 0.0027$ \\

\midrule
\multicolumn{7}{l}{\textit{$\psi_3$}} \\
T-CMO \textbf{(ours)}
& $\mathbf{0.2555 \pm 0.2206}$
& $\underline{0.1378 \pm 0.1490}$
& $\mathbf{0.0756 \pm 0.0766}$
& $\mathbf{0.0410 \pm 0.0397}$
& $\mathbf{0.0263 \pm 0.0283}$
& $\mathbf{0.0265 \pm 0.0186}$ \\
RFF-EDMD
& $0.3527 \pm 0.2495$ & $0.2155 \pm 0.2180$
& $0.1392 \pm 0.1820$ & $0.1216 \pm 0.1600$
& $0.0595 \pm 0.1098$ & $0.0407 \pm 0.0805$ \\
MetaKoopman
& $0.3405 \pm 0.2370$ & $0.2073 \pm 0.1903$
& $0.1416 \pm 0.1534$ & $0.1161 \pm 0.1555$
& $0.0805 \pm 0.1244$ & $0.0565 \pm 0.0950$ \\
Single-task NCP
& $0.3682 \pm 0.2335$ & $0.2267 \pm 0.2020$
& $0.1593 \pm 0.1865$ & $0.1267 \pm 0.1765$
& $0.0939 \pm 0.1615$ & $0.0593 \pm 0.1160$ \\
RFF + resolvent
& $0.3821 \pm 0.2407$ & $0.2376 \pm 0.2083$
& $0.1775 \pm 0.2036$ & $0.1363 \pm 0.1851$
& $0.1055 \pm 0.1793$ & $0.0678 \pm 0.1376$ \\
Pooled NCP
& $0.4675 \pm 0.2236$ & $0.2781 \pm 0.2165$
& $0.2045 \pm 0.2158$ & $0.1502 \pm 0.1988$
& $0.1228 \pm 0.2024$ & $0.0739 \pm 0.1482$ \\
VAMPNet
& $\underline{0.3118 \pm 0.2496}$
& $\mathbf{0.1347 \pm 0.1388}$
& $\underline{0.0829 \pm 0.0813}$
& $\underline{0.0623 \pm 0.0623}$
& $\underline{0.0410 \pm 0.0533}$
& $\underline{0.0290 \pm 0.0321}$ \\
\bottomrule
\end{tabular}
}
\end{table}

\begin{table}[t]
\centering
\scriptsize
\caption{
Relative eigenvalue errors for the first three non-trivial modes on the
$150$ unseen Müller--Brown systems. Values are mean $\pm$ standard deviation.
Lower is better; \textbf{best} and \underline{second best} per column.
}
\label{tab:l2d_lambda}
\setlength{\tabcolsep}{3.5pt}
\renewcommand{\arraystretch}{0.94}
\resizebox{\linewidth}{!}{
\begin{tabular}{lcccccc}
\toprule
& \multicolumn{6}{c}{\textbf{Trajectory length} $N$} \\
\cmidrule(lr){2-7}
Method
& \textbf{2\,000} & \textbf{5\,000} & \textbf{10\,000}
& \textbf{20\,000} & \textbf{40\,000} & \textbf{80\,000} \\
\midrule

\multicolumn{7}{l}{\textit{$\lambda_1$}} \\
T-CMO \textbf{(ours)}
& $\mathbf{0.8288 \pm 2.8035}$ & $\mathbf{0.1919 \pm 0.2668}$
& $\mathbf{0.1241 \pm 0.1057}$ & $\mathbf{0.0916 \pm 0.0824}$
& $\mathbf{0.0643 \pm 0.0569}$ & $\mathbf{0.0443 \pm 0.0367}$ \\
RFF-EDMD
& $0.9852 \pm 2.8284$ & $0.3160 \pm 0.3197$
& $0.2710 \pm 0.1682$ & $0.1236 \pm 0.1025$
& $0.2239 \pm 0.0935$ & $0.1620 \pm 0.0706$ \\
MetaKoopman
& $1.0705 \pm 3.0177$ & $0.2779 \pm 0.4465$
& $0.1842 \pm 0.1614$ & $0.1356 \pm 0.1260$
& $0.1043 \pm 0.0853$ & $0.0870 \pm 0.0689$ \\
Single-task NCP
& $1.0360 \pm 2.9714$ & $0.2714 \pm 0.4275$
& $0.1829 \pm 0.1489$ & $0.1346 \pm 0.1168$
& $0.1026 \pm 0.0793$ & $0.0885 \pm 0.0590$ \\
RFF + resolvent
& $1.0123 \pm 2.9517$ & $0.2567 \pm 0.4141$
& $0.1676 \pm 0.1395$ & $0.1202 \pm 0.1095$
& $0.0878 \pm 0.0731$ & $0.0695 \pm 0.0530$ \\
Pooled NCP
& $1.0015 \pm 2.9503$ & $0.2490 \pm 0.4097$
& $0.1570 \pm 0.1313$ & $0.1088 \pm 0.1026$
& $0.0766 \pm 0.0635$ & $0.0512 \pm 0.0458$ \\
VAMPNet
& $\underline{0.9758 \pm 3.0234}$
& $\underline{0.2143 \pm 0.3208}$
& $\underline{0.1348 \pm 0.1203}$
& $\underline{0.0951 \pm 0.0863}$
& $\underline{0.0643 \pm 0.0570}$
& $\underline{0.0452 \pm 0.0386}$ \\

\midrule
\multicolumn{7}{l}{\textit{$\lambda_2$}} \\
T-CMO \textbf{(ours)}
& $\mathbf{0.2605 \pm 0.4016}$ & $\mathbf{0.1140 \pm 0.0923}$
& $\mathbf{0.0764 \pm 0.0659}$ & $\mathbf{0.0520 \pm 0.0393}$
& $\mathbf{0.0354 \pm 0.0304}$ & $\mathbf{0.0240 \pm 0.0202}$ \\
RFF-EDMD
& $\underline{0.2830 \pm 0.3944}$ & $\underline{0.1269 \pm 0.1085}$
& $0.0922 \pm 0.0951$ & $\underline{0.0564 \pm 0.0464}$
& $0.0544 \pm 0.0382$ & $0.0395 \pm 0.0269$ \\
MetaKoopman
& $0.3265 \pm 0.5164$ & $0.1515 \pm 0.1390$
& $0.1077 \pm 0.0933$ & $0.0801 \pm 0.0647$
& $0.0649 \pm 0.0514$ & $0.0569 \pm 0.0388$ \\
Single-task NCP
& $0.3154 \pm 0.4623$ & $0.1498 \pm 0.1325$
& $0.1035 \pm 0.0890$ & $0.0751 \pm 0.0617$
& $0.0582 \pm 0.0478$ & $0.0482 \pm 0.0348$ \\
RFF + resolvent
& $0.3122 \pm 0.4538$ & $0.1484 \pm 0.1314$
& $0.1021 \pm 0.0866$ & $0.0739 \pm 0.0619$
& $0.0559 \pm 0.0466$ & $0.0440 \pm 0.0325$ \\
Pooled NCP
& $0.3111 \pm 0.4416$ & $0.1483 \pm 0.1302$
& $0.1019 \pm 0.0842$ & $0.0714 \pm 0.0623$
& $0.0525 \pm 0.0454$ & $0.0388 \pm 0.0294$ \\
VAMPNet
& $0.3214 \pm 0.5216$ & $0.1308 \pm 0.1264$
& $\underline{0.0891 \pm 0.0946}$ & $0.0566 \pm 0.0492$
& $\underline{0.0417 \pm 0.0356}$ & $\underline{0.0310 \pm 0.0239}$ \\

\midrule
\multicolumn{7}{l}{\textit{$\lambda_3$}} \\
T-CMO \textbf{(ours)}
& $\mathbf{0.2021 \pm 0.2943}$ & $\mathbf{0.0962 \pm 0.0798}$
& $\mathbf{0.0646 \pm 0.0545}$ & $\mathbf{0.0445 \pm 0.0373}$
& $\mathbf{0.0338 \pm 0.0264}$ & $\mathbf{0.0249 \pm 0.0215}$ \\
RFF-EDMD
& $\underline{0.2055 \pm 0.2940}$ & $0.1103 \pm 0.1061$
& $0.0837 \pm 0.0871$ & $0.0705 \pm 0.0791$
& $0.0467 \pm 0.0494$ & $\underline{0.0313 \pm 0.0337}$ \\
MetaKoopman
& $0.2753 \pm 0.3363$ & $0.1369 \pm 0.1217$
& $0.1157 \pm 0.0882$ & $0.0913 \pm 0.0713$
& $0.0781 \pm 0.0621$ & $0.0731 \pm 0.0415$ \\
Single-task NCP
& $0.2640 \pm 0.3232$ & $0.1353 \pm 0.1120$
& $0.1059 \pm 0.0865$ & $0.0848 \pm 0.0686$
& $0.0681 \pm 0.0585$ & $0.0557 \pm 0.0428$ \\
RFF + resolvent
& $0.2551 \pm 0.3070$ & $0.1360 \pm 0.1100$
& $0.1090 \pm 0.0903$ & $0.0899 \pm 0.0808$
& $0.0708 \pm 0.0640$ & $0.0541 \pm 0.0477$ \\
Pooled NCP
& $0.2538 \pm 0.3021$ & $0.1380 \pm 0.1151$
& $0.1133 \pm 0.0921$ & $0.0921 \pm 0.0858$
& $0.0734 \pm 0.0669$ & $0.0539 \pm 0.0498$ \\
VAMPNet
& $0.2370 \pm 0.4022$ & $\underline{0.1096 \pm 0.1212}$
& $\underline{0.0769 \pm 0.0622}$ & $\underline{0.0544 \pm 0.0448}$
& $\underline{0.0446 \pm 0.0361}$ & $0.0318 \pm 0.0270$ \\
\bottomrule
\end{tabular}
}
\end{table}

\FloatBarrier

\paragraph{Transfer versus single-task estimation.}
To isolate the contribution of the shared representation, we compare T-CMO
directly with single-task NCP, which estimates each unseen system independently
from the same target trajectory. Across the quantitative tables above, T-CMO
has lower mean eigenfunction and eigenvalue errors than the single-task model
for every reported mode and trajectory length. The improvement is modest for
the easiest mode once trajectories become long, but becomes substantially
larger for the harder second and third modes, showing that transfer is most
useful when the target spectral structure is difficult to estimate from a
single trajectory.

\begin{figure}[!htbp]
    \centering
    \includegraphics[width=0.92\linewidth]{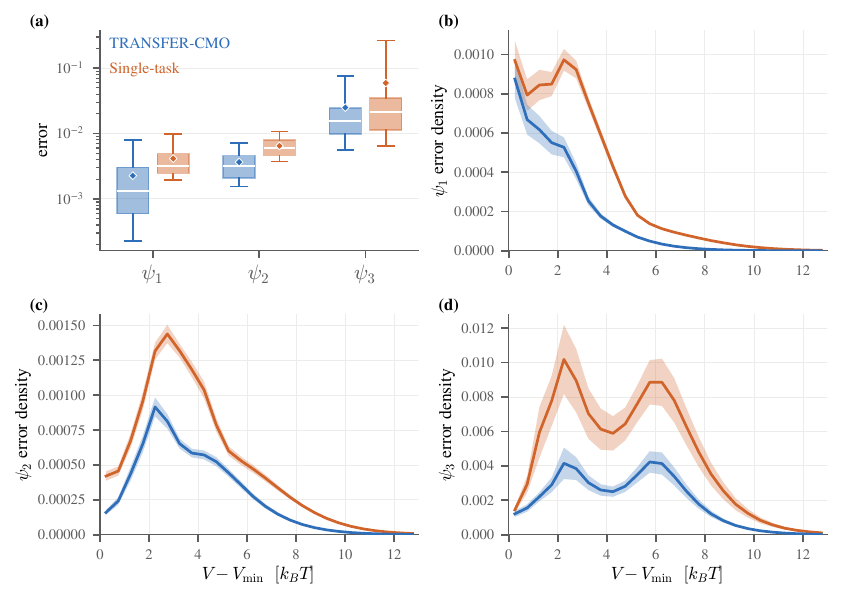}
    \caption{
    \textbf{Transfer versus single-task spectral recovery.}
    \textbf{(a)} Distribution of mode-wise
    $1-\cos_{\pi}(\widehat{\psi}_j,\psi_j)$ errors over unseen systems.
    \textbf{(b--d)} Error distribution across potential-energy levels for
    $\psi_1$, $\psi_2$, and $\psi_3$.
    }
    \label{fig:mb_transfer_modes}
\end{figure}

\Cref{fig:mb_transfer_modes} further reveals a clear hierarchy in spectral
difficulty. Both methods recover the first two modes relatively accurately,
whereas $\psi_3$ exhibits larger errors and stronger variability across
systems. T-CMO reduces the error for all three modes, with the clearest
separation on $\psi_3$. The energy-resolved curves show that the residual
error is not spread uniformly over the state space: it is concentrated over
specific non-minimal energy ranges associated with the more difficult parts
of the landscape. T-CMO suppresses these errors across the energy range rather
than improving only isolated regions.

\begin{figure}[t]
    \centering
    \includegraphics[width=0.95\linewidth]{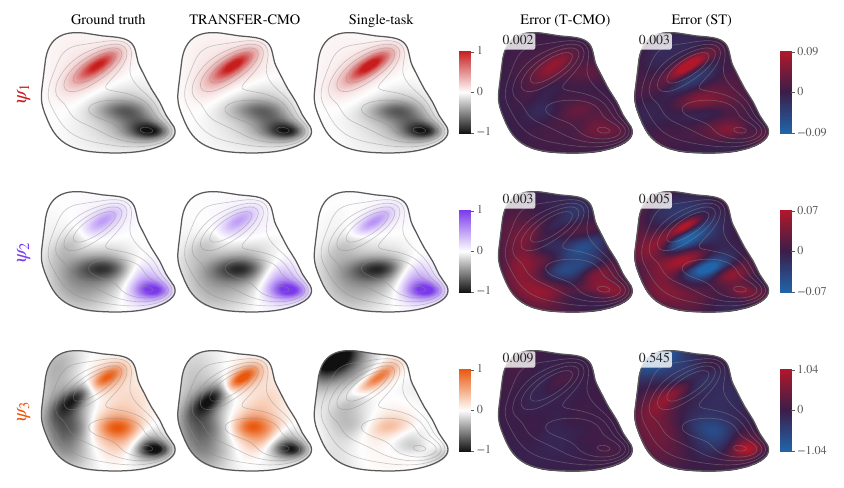}
    \caption{
    \textbf{Spatial recovery of the first three eigenfunctions on a representative
    unseen system.}
    From left to right: finite-difference reference, T-CMO, single-task NCP,
    and the corresponding pointwise errors. Eigenfunctions are sign-aligned
    and normalized in $L^2(\pi)$.
    }
    \label{fig:mb_transfer_fields}
\end{figure}

\Cref{fig:mb_transfer_fields} provides a spatial view of the same effect.
The displayed eigenfunctions are multiplied by $\sqrt{\pi}$, so that
$\|f\|_{L^2(\pi)}^2=\int |\sqrt{\pi(\mbx)}f(\mbx)|^2\,d\mbx$ and visible
spatial discrepancies directly reflect their contribution to the evaluation
metric. For $\psi_1$ and $\psi_2$, both approaches reproduce the dominant
structure, although T-CMO leaves smaller residual errors. The difference is
much stronger for $\psi_3$: T-CMO remains close to the finite-difference
reference, whereas the independently learned estimate exhibits substantial
spatial distortion. Together with the quantitative results above, this shows
that reusing the multi-task representation improves both global operator
recovery and the more difficult task-specific spectral components.

\FloatBarrier

\section{Plasma experiments}
\label{app:plasma}

This experiment tests whether functional spaces learned from turbulent plasma
trajectories, without access to the physical control parameters, recover the
physical organization of the Tokam2D family. MTL-CMO first learns the shared
representation across source systems; T-CMO then estimates a system-specific
operator in closed form, from which we extract a finite-dimensional generator
spectrum. The physical parameters $(g,\kappa)$ are used only afterward as
diagnostic variables. The complete experiment is run on a single NVIDIA RTX
A6000 GPU.

\subsection{Physical setting and plasma data}
\label{sec:plasma_physics}

\paragraph{Physical model.}
The benchmark is based on a reduced two-field model for edge turbulence,
describing the coupled evolution of density fluctuations $n(x,y,t)$ and the
electrostatic potential $\phi(x,y,t)$
\citep{hasegawa1983plasma,ghendrih2018sol,ghendrih2022avalanche}.
The potential determines the perpendicular $\mathbf E\times\mathbf B$ flow and
the vorticity,
\[
\mathbf V_E=(-\partial_y\phi,\partial_x\phi),
\qquad
\Omega=\Delta_\perp\phi,
\qquad
\Delta_\perp=\partial_{xx}+\partial_{yy}.
\]
Since $\mathbf V_E$ is divergence free, nonlinear transport can be written as
$\mathbf V_E\!\cdot\nabla f=[\phi,f]$, with
$[\phi,f]=\partial_x\phi\,\partial_y f-\partial_y\phi\,\partial_x f$.

We consider the gradient-driven configuration, with inverse density-gradient
length $\kappa=1/L_n$. The reduced dynamics are
\[
\left\{
\begin{aligned}
\partial_t n
+\kappa\,\partial_y\phi
+[\phi,n]
-D_n\Delta_\perp n
&=
-\sigma_n n+\sigma_{n\phi}\phi,
\\
\partial_t\Omega
+g\,\partial_y n
+[\phi,\Omega]
-D_\phi\Delta_\perp\Omega
&=
-\sigma_{\phi n}n+\sigma_\phi\phi,
\qquad
\Omega=\Delta_\perp\phi .
\end{aligned}
\right.
\]
The two parameters act through distinct mechanisms.
The term $\kappa\,\partial_y\phi$ couples the imposed background density
gradient to the radial $\mathbf E\times\mathbf B$ motion and drives density
fluctuations, whereas $g\,\partial_y n$ couples density perturbations back into
the vorticity equation and controls the interchange mechanism. Their nonlinear
interaction closes the feedback loop
$n\rightarrow\Omega\rightarrow\phi\rightarrow\mathbf V_E\rightarrow n$.
Diffusion and linear loss terms are fixed throughout the dataset.

A physically important observable is the radial particle flux
$\Gamma_x=nV_{E,x}=-n\,\partial_y\phi$, which measures radial transport induced
by the correlation between density fluctuations and the
$\mathbf E\times\mathbf B$ flow.

\paragraph{Simulation family.}
We generate $M=400$ simulations using Tokam2D~\citep{tokam2d}, with
$g\sim\mathcal U([0,0.5])$ and
$\kappa\sim\mathcal U([1,3.5])$. All remaining physical and numerical
parameters are fixed. Simulations are performed on a $128\times128$ spatial
grid over a square domain of side length $64$, with snapshots separated by
$\Delta t=0.05$. Each trajectory contains $4093$ stored states. The first
$320$ systems are used for multi-task learning and the remaining $80$ are held
out for evaluation.

\begin{figure}[t]
    \centering
    \includegraphics[width=\linewidth]{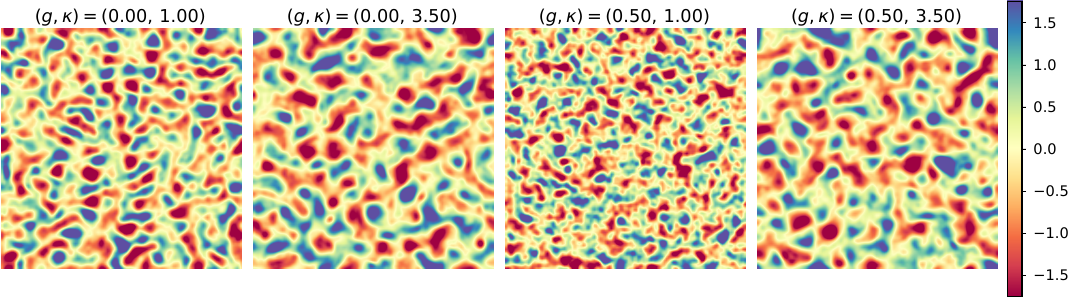}
    \vspace{0.25em}
    \includegraphics[width=\linewidth]{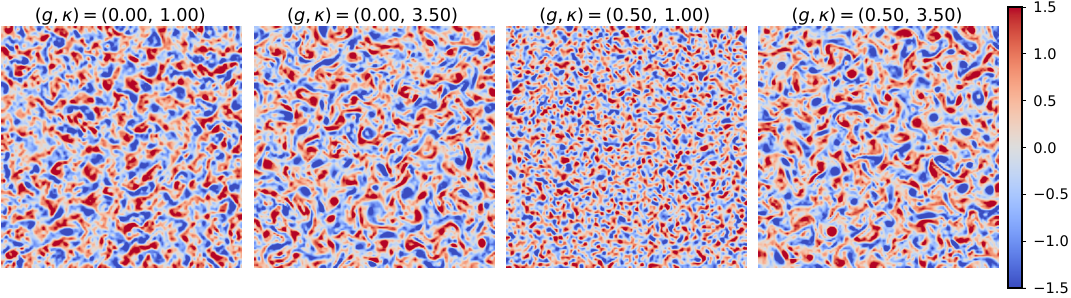}
    \caption{
    Representative electrostatic potential $\phi$ (top) and vorticity
    $\Omega=\Delta_\perp\phi$ (bottom) across the $(g,\kappa)$ domain.
    Each field is standardized independently for visualization.
    }
    \label{fig:app_plasma_potential_vorticity}
\end{figure}

\Cref{fig:app_plasma_potential_vorticity} illustrates the distinct effects of
the two physical parameters. Increasing $\kappa$ strengthens the
density-gradient drive and substantially changes the characteristic spatial
scale of the potential and vorticity fields, with high-$\kappa$ regimes
exhibiting larger and more coherent structures. The effect of $g$ is
qualitatively different: at low $\kappa$, increasing $g$ produces a finer and
more fragmented vorticity field, consistently with its direct coupling to
density perturbations through $g\,\partial_y n$. At larger $\kappa$, this
effect interacts with the stronger density-gradient forcing. The dependence
on $(g,\kappa)$ is therefore not simply additive; both parameters jointly
control the spatial organization of the turbulent flow.

\paragraph{Observed fields and preprocessing.}
Each state contains three physical channels: density $n$, vorticity $\Omega$,
and radial particle flux $\Gamma_x$. The original
$3\times128\times128$ fields are downsampled by a factor of two before being
passed to the network, giving
$\mathbf x_t\in\mathbb R^{3\times64\times64}$.
These observables respectively describe transported density, rotational flow
structure, and radial turbulent transport. The electrostatic potential shown
in \Cref{fig:app_plasma_potential_vorticity} is used only for physical
interpretation and is not provided directly to MTL-CMO.

\FloatBarrier

\subsection{Shared functional representation}
\label{sec:plasma_representation}

MTL-CMO parameterizes the shared dictionaries $\Phi_\theta$ and $\Psi_\theta$
with a convolutional ResNet. A common encoder maps the three input fields to a
$256$-dimensional latent representation, followed by two output heads producing
$d=128$ functional features. The encoder contains three residual stages with
channel dimensions $(64,128,256)$ and $(2,2,2)$ residual blocks, using
GroupNorm, SiLU activations, and dropout. For each source system, the
task-specific operator has rank $80$ and is parameterized as
$\mbM^{(k)}=\mbA^{(k)}\mbB^{(k)\top}$.

Training follows the MTL-CMO procedure of \Cref{alg:mtlcmo} using consecutive
states $(x_t^{(k)},x_{t+1}^{(k)})$. Each optimization step samples $8$ systems
and $256$ transitions per system. Input channels are standardized using
statistics computed from the $320$ source systems only, while dictionary
outputs are centered within each sampled system window before evaluating the
multi-task objective.

\begin{table}[t]
\centering
\footnotesize
\setlength{\tabcolsep}{4pt}
\renewcommand{\arraystretch}{0.95}
\caption{Architecture and optimization configuration for the plasma experiment.}
\label{tab:plasma_training}
\begin{tabular}{lc}
\toprule
Configuration & Value \\
\midrule
\multicolumn{2}{l}{\textit{Architecture}} \\
Input size & $3\times64\times64$ \\
Shared feature dimension $d$ & $128$ \\
Task-specific rank & $80$ \\
Encoder & Convolutional ResNet \\
Residual stages & $(64,128,256)$ \\
Blocks per stage & $(2,2,2)$ \\
Encoder output dimension & $256$ \\
Output heads & $2\times(256\rightarrow128)$ \\
Normalization & GroupNorm (8 groups) \\
Activation & SiLU \\
Dropout & $0.2$ \\
\midrule
\multicolumn{2}{l}{\textit{Optimization}} \\
Systems per mini-batch & $8$ \\
Transitions per system & $256$ \\
Epochs & $120$ \\
Optimizer & AdamW \\
Learning rate & $3\times10^{-4}$ \\
Minimum learning rate & $10^{-6}$ \\
Weight decay & $10^{-4}$ \\
Learning-rate schedule & Cosine \\
Warm-up & $10\%$ \\
Gradient clipping & $1$ \\
\bottomrule
\end{tabular}
\end{table}

\begin{figure}[t]
    \centering
    \includegraphics[width=0.72\linewidth]{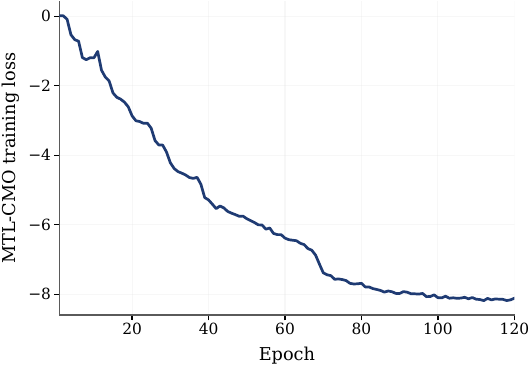}
    \caption{Evolution of the MTL-CMO training objective over $120$ epochs.}
    \label{fig:plasma_loss}
\end{figure}

\FloatBarrier

\subsection{System-specific transfer and spectral estimation}
\label{sec:plasma_operator}

\paragraph{Closed-form adaptation.}
For each plasma system, we apply the T-CMO procedure of
\Cref{alg:tcmo} to the available trajectory. The shared dictionaries are reused
without gradient updates, and only the system-specific operator is estimated.
The same procedure is applied to source systems when constructing diagnostic
representations and to held-out systems for transfer evaluation.

\paragraph{Adaptation rank.}
We select the transfer rank using a reduced-rank regression diagnostic. For
each candidate rank, a one-step predictor is estimated independently on each
of the $320$ source systems and the reconstruction error is averaged across
systems. As shown in \Cref{fig:plasma_rrr_rank}, the error decreases rapidly
and largely saturates around $r=32$. We therefore use adaptation rank $32$
throughout the plasma experiments. This is distinct from the rank-$80$
task-specific operators used during MTL-CMO training.

\begin{figure}[t]
    \centering
    \includegraphics[width=0.62\linewidth]{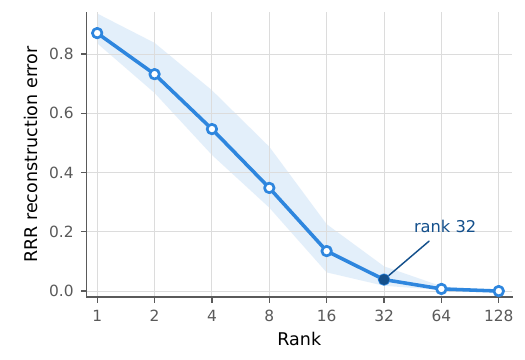}
    \caption{
    Mean one-step RRR reconstruction error across the $320$ source systems
    versus retained rank. The transfer rank is fixed to $32$.
    }
    \label{fig:plasma_rrr_rank}
\end{figure}

\paragraph{Generator spectrum.}
As explained in Appendix~\ref{app:tcmo-op-est}, the resulting $32$-dimensional adapted trajectories are passed to the
generator-resolvent estimator described in Appendix~\ref{app:resolvent}. We
retain the $31$ slowest recovered generator eigenvalues and represent each
system by their real and absolute imaginary parts, yielding a
$62$-dimensional spectral descriptor. We use resolvent shift $s=10$
throughout the plasma experiments.

\FloatBarrier

\subsection{Physical identification and data efficiency}
\label{sec:plasma_identification}

\paragraph{Compared representations.}
We compare four dynamical representations.
\textbf{MTL-CMO/T-CMO} uses the shared ResNet learned across the source family,
followed by closed-form T-CMO adaptation and the generator-resolvent pipeline,
yielding $62$ spectral features.
\textbf{Pooled NCP} uses the same neural architecture and source data but learns
a single pooled representation without task-specific operators; the same
adaptation and spectral-estimation pipeline is used afterward.
\textbf{RFF} replaces the learned functional representation by $128$ Gaussian
random Fourier features, with bandwidth determined from the source systems,
and uses the same downstream spectral estimator.
Finally, \textbf{MIDST} learns a shared latent dynamical model with a
$128$-dimensional system-specific representation; for an unseen system its
shared parameters are frozen and the system-specific coefficients are adapted
by gradient descent. We use lag $10$, selected from the source systems.

For physical identification, the spectral coordinates of MTL-CMO/T-CMO,
Pooled NCP, and RFF are the real and absolute imaginary parts of the $31$
slowest recovered eigenvalues. MIDST uses its native $128$-dimensional
system-specific representation. Features are standardized using the $320$
source systems only. We fit separate RBF kernel-ridge regressors for $g$ and
$\kappa$, using $\alpha=10^{-2}$ and $\gamma=10^{-3}$ for every method, with
no method-specific tuning. Performance is measured by $R^2$ on the $80$
held-out systems.

\paragraph{Geometry of the learned representations.}
We first examine whether the unsupervised representations organize the source
systems according to the underlying physics. For each method, we compute
pairwise Euclidean distances between standardized system descriptors and apply
classical multidimensional scaling. The values of $g$ and $\kappa$ are used
only afterward to color the embeddings.

\begin{figure}[t]
    \centering
    \includegraphics[width=\linewidth]{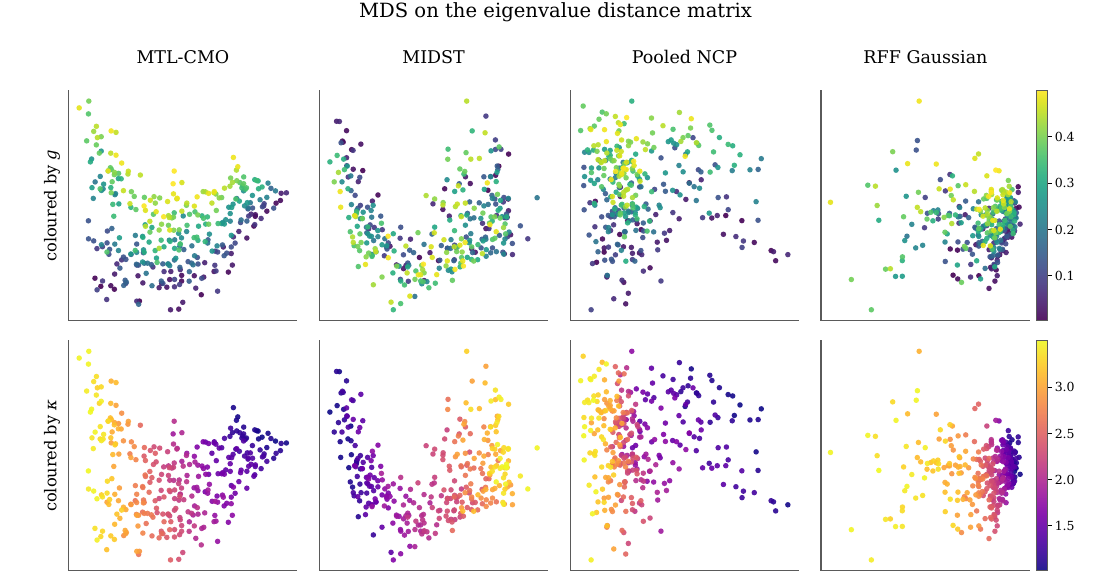}
    \caption{
    MDS embeddings of the dynamical representations of the $320$ source
    systems, colored by $g$ (top) and $\kappa$ (bottom).
    }
    \label{fig:plasma_mds}
\end{figure}

The MTL-CMO representation exhibits a continuous organization with respect to
both physical parameters. The ordering is particularly pronounced for
$\kappa$, while $g$ induces a weaker but still structured variation across the
embedding. Since neither parameter is observed during representation learning
or operator estimation, this organization indicates that the learned spectrum
captures physically meaningful variation across turbulent regimes rather than
direct supervision from $(g,\kappa)$.

\paragraph{Identification from short trajectories.}
We next recompute every system-dependent representation from increasingly short
trajectory prefixes. The source/held-out split and downstream regression
protocol remain unchanged. \Cref{fig:plasma_horizon_app} and
 show that T-CMO yields the highest $R^2$ for both
parameters at every trajectory length.

For $\kappa$, the spectral representation is already highly informative from
only $256$ frames, with $R^2=0.935$, and reaches $0.992$ on the complete
trajectory. Recovering $g$ is substantially harder: T-CMO increases from
$R^2=0.225$ at $256$ frames to $0.778$ at $878$ and $0.932$ on the full
trajectory. This difference is consistent with the spectral geometry above:
the density-gradient parameter is strongly expressed in the leading spectral
structure, whereas the interchange parameter is encoded through a more
distributed signature and benefits more from additional temporal information.

\begin{figure}[t]
    \centering
    \includegraphics[width=\linewidth]{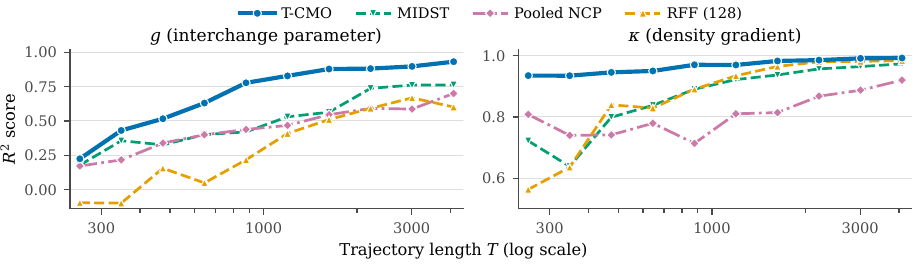}
    \caption{
    Physical-parameter recovery on held-out systems versus trajectory length
    for T-CMO and competing representations.
    }
    \label{fig:plasma_horizon_app}
\end{figure}

\begin{table}[t]
\centering
\footnotesize
\setlength{\tabcolsep}{4pt}
\renewcommand{\arraystretch}{0.95}
\caption{
Held-out $R^2$ for physical-parameter recovery versus available trajectory
length. Lower-dimensional spectral methods use $62$ features; MIDST uses its
native $128$-dimensional representation.
}
\label{tab:plasma_r2}
\begin{tabular}{lcccccc}
\toprule
Trajectory length $T$
& $256$ & $474$ & $878$ & $1\,625$ & $3\,008$ & $4\,093$ \\
\midrule
\multicolumn{7}{l}{\textit{Interchange parameter} $g$} \\
T-CMO \textbf{(ours)}
& $\mathbf{0.225}$ & $\mathbf{0.516}$ & $\mathbf{0.778}$
& $\mathbf{0.878}$ & $\mathbf{0.897}$ & $\mathbf{0.932}$ \\
MIDST
& $0.178$ & $0.326$ & $0.421$ & $0.564$ & $0.761$ & $0.763$ \\
Pooled NCP
& $0.172$ & $0.339$ & $0.438$ & $0.546$ & $0.587$ & $0.700$ \\
RFF (128)
& $-0.096$ & $0.156$ & $0.217$ & $0.511$ & $0.669$ & $0.692$ \\
\midrule
\multicolumn{7}{l}{\textit{Density-gradient parameter} $\kappa$} \\
T-CMO \textbf{(ours)}
& $\mathbf{0.935}$ & $\mathbf{0.946}$ & $\mathbf{0.970}$
& $\mathbf{0.983}$ & $\mathbf{0.991}$ & $\mathbf{0.992}$ \\
MIDST
& $0.722$ & $0.799$ & $0.891$ & $0.937$ & $0.964$ & $0.973$ \\
Pooled NCP
& $0.808$ & $0.741$ & $0.714$ & $0.814$ & $0.887$ & $0.920$ \\
RFF (128)
& $0.563$ & $0.840$ & $0.890$ & $0.965$ & $0.982$ & $0.984$ \\
\bottomrule
\end{tabular}
\end{table}

The comparisons separate the contributions of representation learning and
multi-task structure. RFF keeps the downstream spectral estimator but removes
the learned functional space, while Pooled NCP retains the neural architecture
but removes the task-specific multi-task factorization. Their lower performance,
especially for $g$, shows that the gain cannot be explained by the downstream
regressor or spectral estimator alone. MIDST benefits from cross-system
learning but remains less accurate, particularly for short trajectories and
for recovery of the interchange parameter.

\paragraph{Number of retained spectral modes.}
We finally vary the number $n$ of slow generator modes supplied to the
diagnostic regressor. Each mode contributes two coordinates,
$\operatorname{Re}\lambda_j$ and $|\operatorname{Im}\lambda_j|$.
\Cref{fig:plasma_nmodes} shows that $\kappa$ is largely encoded by the leading
modes and saturates rapidly. In contrast, $g$ benefits from a broader part of
the slow spectrum and largely plateaus only after roughly a dozen modes. We
retain all $31$ modes in the main evaluation, yielding $62$ spectral features.

\begin{figure}[t]
    \centering
    \includegraphics[width=0.62\linewidth]{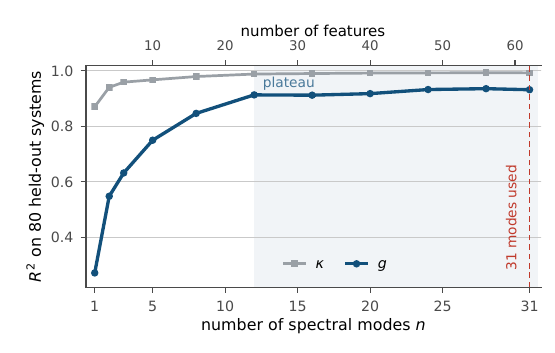}
    \caption{
    Held-out $R^2$ versus the number of slow generator modes retained from the
    T-CMO representation. Each mode contributes two spectral coordinates.
    }
    \label{fig:plasma_nmodes}
\end{figure}

\paragraph{RFF dimension sensitivity.}
The Gaussian RFF baseline exhibits a clear optimum around $D=128$ on complete
trajectories. Increasing the number of random features beyond this value does
not improve physical identification. 
\begin{table}[t]
\centering
\footnotesize
\setlength{\tabcolsep}{5pt}
\renewcommand{\arraystretch}{0.95}
\caption{Gaussian RFF sensitivity to the random-feature dimension $D$ on complete trajectories.}
\label{tab:rff_dims}
\begin{tabular}{lcccccc}
\toprule
Features $D$ & $32$ & $64$ & $\mathbf{128}$ & $256$ & $512$ & $1024$ \\
\midrule
Interchange $g$
& $0.388$ & $0.662$ & $\mathbf{0.714}$ & $0.641$ & $0.639$ & $0.639$ \\
Density gradient $\kappa$
& $0.968$ & $0.979$ & $\mathbf{0.986}$ & $0.985$ & $0.983$ & $0.982$ \\
\bottomrule
\end{tabular}
\end{table}

\begin{table}[t]
\centering
\footnotesize
\setlength{\tabcolsep}{5pt}
\renewcommand{\arraystretch}{0.95}
\caption{Physical-identification evaluation settings.}
\label{tab:plasma_eval_settings}
\begin{tabular}{lc}
\toprule
Parameter & Value \\
\midrule
Source systems & $320$ \\
Held-out systems & $80$ \\
T-CMO / Pooled NCP spectral features & $62$ \\
RFF dictionary dimension & $128$ \\
RFF spectral features & $62$ \\
MIDST features & $128$ \\
MIDST lag & $10$ \\
Regression model & Kernel ridge, RBF \\
Feature preprocessing & StandardScaler \\
KRR regularization $\alpha$ & $10^{-2}$ \\
RBF parameter $\gamma$ & $10^{-3}$ \\
Hyperparameter selection & Fixed, no cross-validation \\
Metric & Held-out $R^2$ \\
\bottomrule
\end{tabular}
\end{table}

\FloatBarrier

\input{sections/Stat_analysis}

%% file: sections/related_work.tex
\section{Detailed related work}
\label{app:related_work}

The proposed approach, MTL-CMO, jointly estimates conditional mean operators (CMOs) from data associated with related joint distributions (see \Cref{sec:method}). Because CMOs describe conditional distributions through their action on observables, relevant work spans both general-purpose operator estimation and application-specific methods. We discuss four connections: (i) multi-task representation learning, (ii) single-task CMO estimation, (iii) uncertainty quantification, and (iv) operator-based learning for dynamical systems.

\noindent
\textbf{Positioning within multi-task representation learning.}
MTL-CMO can be considered as a two-sided extension of the one-sided shared representation paradigm of multi-task learning~\citet{caruana1997multitask}. A standard formulation of this framework stipulates that each task has a predictor
\begin{equation*}
    \textstyle
    f_k(x) = w_k^\top h_\theta(x),
\end{equation*}
where $h_\theta$ is a shared feature map and $w_k$ the task-specific head. The one-sided problem has been studied extensively \citep{evgeniou2005learning,ando2005framework,argyriou2008convex} and typically involves specific regularization~\citep{evgeniou2004regularized,zhang2010learning,maurer2013sparse}. Provable statistical guarantees on the benefit of sharing feature maps between tasks have been established~\citep{maurer2016benefit,tripuraneni2021provable}. \\
MTL-CMO differs in three respects. First, each task-specific object is an operator between function spaces rather than a scalar- or vector-valued predictor. Second, it is two-sided; the method learns shared function spaces on both the input and the output sides of these operators. Third, although the functions themselves are shared across tasks, their associated synthesis maps are defined in the task-specific spaces $\lxk$ and $\lyk$, whose inner products depend on the marginals $\mu_k$ and $\nu_k$.

\noindent
\textbf{Conditional mean operators and their estimation.}
Conditional mean operators have long been studied; for instance, early work addressed the problem of estimating multiple regression and correlation~\citep{breiman1985estimating}. 
A prominent body of work, known as conditional mean embedding (CME), focus on the estimation of the operator whose action is restricted to predefined Reproducible Kernel Hilbert Space (RKHS)~\citep{song2009hilbert,muandet2017kernel}. with classical estimators expressed as regularized kernel ridge regression, whose consistency and statistical rates have been subsequently studied~\citep{grunewalder2012conditional,mollenhauer2020nonparametric,hertel2026verifiable}. More recently, a hybrid CME estimator combining neural networks to model the output space with a predefined RKHS to model the input space was proposed in~\citep{shimizu2024neural}.
CME approaches have also been extensively used to study dynamical systems, although the choice of function space can affect the recovered spectral information~\citep{kostic2023sharp}. However, CMEs are less well-suited to uncertainty quantification, since evaluating conditional expectations is only possible for observables in the RKHS. To alleviate this issue, Neural Conditional Probability (NCP) instead learns a truncated singular representation of a compact CMO, allowing its action to be evaluated on a broader class of observables~\citep{kostic2024ncp}. MTL-CMO builds on this operator perspective while learning functional representations jointly across related tasks.

\noindent
\textbf{Multi-task learning for uncertainty quantification.}
The core idea of multi-task approaches for uncertainty quantification is to jointly learn conditional density functions or conditional cumulative distribution functions.
A baseline approach to share information across conditional distributions is to combine a common neural representation with task-specific mixture density networks~\citep{caruana1997multitask,bishop1994mixture}. Each task then predicts the parameters of its own conditional mixture model. Another approach, DeepJMQR, jointly learns conditional expectations and multiple conditional quantiles to estimate a discrete approximation of one-dimensional conditional transport maps~\citep{rodrigues2020beyond}. \\
These methods directly parameterize a density or selected statistics. MTL-CMO instead jointly learns operators from which different conditional statistics can be approximated by projecting the corresponding observables, see \Cref{eq:cs_approx}.

\noindent
\textbf{Operator-based multi-task learning for dynamical systems.}
For a stochastic dynamical system, the conditional mean operator that maps an observable of a future state to its conditional expectation given the current state is a Koopman operator at the corresponding time lag. 
Several multi-task approaches have been proposed to jointly learn dynamics from observed trajectories and by leveraging Koopman operators~\citep{iwata2021meta,elul2024data,zhang2025koopman,han2026mako}. These methods suppose a shared embedding of the trajectories and impose a task-specific linear evolution of the dynamics in the latent space. Some additionally learn a common~\citep{elul2024data,zhang2025koopman} or task-specific~\citep{han2026mako} decoder from latent coordinates to the state spaces. In particular, \citet{han2026mako} has been proposed for predictive control and therefore includes the control command in the modeling. \\
Unlike MTL-CMO, these methods are primarily trained to reconstruct or predict trajectories, rather than to minimize a Hilbert-Schmidt discrepancy between operators. Their objectives may therefore favor accurate state predictions without ensuring accurate estimates of the operator's action on other observables. MTL-CMO instead minimizes an empirical objective corresponding to the Hilbert-Schmidt error, which directly estimates the conditional expectation of future states for any observables.

\noindent
\textbf{Multi-task conditional mean operator estimation.}
Standard multi-task representation learning shares an input feature map across task-specific predictors~\citep{caruana1997multitask,maurer2016benefit}. MTL-CMO extends this principle to operators by learning input and output function spaces, whose corresponding synthesis maps~\Cref{eq:synthesis_map} act in spaces defined by each task's marginals. This differs from conditional mean embedding, which typically uses a prescribed reproducing kernel Hilbert space~\citep{song2009hilbert,muandet2017kernel}, and from NCP, which estimates a single CMO~\citep{kostic2024ncp}.

%% file: sections/Stat_analysis.tex
\section{Statistical guarantees}
\label{sec:theory}

We study how well \method\ learns the shared representation, and how this depends on the number of tasks. The analysis has two steps. First, we show that the population multi-task loss controls, deterministically, the quality of the learned representation: the error of the best head on each training task, and the error on new tasks that share the same structure (\cref{prp:approx,prp:new_task}). Second, we bound the population loss of the trained model (\cref{sec:theory_rate}). The cost of learning the shared representation is paid with the total number of samples $nK$, while each task only pays for its own head.

\paragraph{Setting.} For every task $k\in[K]$, we observe $n$ i.i.d.\ pairs $\{(x_i^{(k)},y_i^{(k)})\}_{i\in[n]}$ from $\rho_k$, the $K$ datasets being independent. We assume that $\rho_k$ has a density $p_k$ with respect to $\mu_k\otimes\nu_k$ such that $g_k:=p_k-1\in\Lk$. Then $\dmok$ is the integral operator with kernel $g_k$, i.e.\ $[\dmok f](x)=\mdE_{\nu_k}[g_k(x,Y)f(Y)]$, and $\norm[\mathrm{HS}]{\dmok}=\norm[\Lk]{g_k}$. We write $h=(\theta,\{\mbA^{(k)},\mbB^{(k)}\}_{k\in[K]})$ for the parameters of the whole model and
\[
q_{h_k}(x,y):=\Phi_\theta(x)^\top\mbA^{(k)}\mbB^{(k)\top}\Psi_\theta(y)
\]
for the kernel learned for task $k$.

\begin{assump}\label{ass:bounded}
There exists $c\geq1$ such that $\norm{\Phi_\theta(x)}\leq c\sqrt d$ and $\norm{\Psi_\theta(y)}\leq c\sqrt d$ for all $\theta\in\Theta$, $x\in\mcX$ and $y\in\mcY$. The heads belong to $\mcW_\rho:=\{(\mbA,\mbB):\norm[\mathrm F]{\mbA}^2+\norm[\mathrm F]{\mbB}^2\leq2\rho\}$ for some $\rho>0$.
\end{assump}

The bound on the features holds, for instance, when the last layer of the networks has a bounded activation. The constraint on the heads is the constrained counterpart of the ridge penalty of \cref{sec:method}. We believe that an analogous guarantee holds for the penalized estimator used in practice, but a complete proof requires significantly more technical work, which we leave for future work. Under \cref{ass:bounded}, $\norm[\mathrm F]{\mbA\mbB^\top}\leq\rho$ and $\abs{q_{h_k}}\leq Q_\rho:=\rho c^2d$ by Cauchy-Schwartz.

\paragraph{Estimator and population risk.} For task $k$, the model is the operator $\dmokt$ with kernel $q_{h_k}$, and the population loss of \cref{sec:method} is $\mcL^{(k)}(h)=\norm[\mathrm{HS}]{\dmokt-\dmok}^2-\norm[\mathrm{HS}]{\dmok}^2$. Since the last term does not depend on $h$, we measure the performance of the model by
\begin{equation}\label{eq:risk}
\begin{aligned}
&\mcE_k(h):=\norm[\mathrm{HS}]{\dmokt-\dmok}^2=\norm[\Lk]{q_{h_k}-g_k}^2,\\
&\overline{\mcE}(h):=\frac1K\sum_{k=1}^K\mcE_k(h),\qquad
\mcA_K:=\inf_{h\in\Theta\times\mcW_\rho^K}\overline{\mcE}(h),
\end{aligned}
\end{equation}
where the second equality holds because $\dmokt-\dmok$ is the integral operator with kernel $q_{h_k}-g_k$. The empirical loss $\widehat{\mcL}^{(k)}(h)$ of \cref{sec:method} is an unbiased estimator of $\mcL^{(k)}(h)$, so that minimizing the multi-task objective amounts to minimizing an empirical version of $\overline{\mcE}$. We consider any $\widehat h=(\widehat\theta,\{\widehat\mbA^{(k)},\widehat\mbB^{(k)}\}_k)\in\Theta\times\mcW_\rho^K$ such that
\begin{equation}\label{eq:erm}
\frac1K\sum_{k=1}^K\widehat{\mcL}^{(k)}(\widehat h)\leq\inf_{h\in\Theta\times\mcW_\rho^K}\frac1K\sum_{k=1}^K\widehat{\mcL}^{(k)}(h)+\epsilon_{\mathrm{opt}}.
\end{equation}

\paragraph{Quality of a representation.} For a task $k$, let $\Phic:=\Phi_\theta-\mdE_{\mu_k}[\Phi_\theta(X)]$ and $\Psic:=\Psi_\theta-\mdE_{\nu_k}[\Psi_\theta(Y)]$ be the centered dictionaries, and $P^{(k)}_{\Phi_\theta}$, $P^{(k)}_{\Psi_\theta}$ the orthogonal projections in $\lxk$ and $\lyk$ onto their spans. We measure the quality of the representation $\theta$ for task $k$ by the error of the best head on the learned spaces,
\begin{equation}\label{eq:eta}
\eta_k^2(\theta):=\min_{\mbM\in\mdR^{d\times d}}\norm[\mathrm{HS}]{\dmok-\msS_{\Phic}\,\mbM\,\msS_{\Psic}^{*}}^2 .
\end{equation}
This quantity does not depend on the heads, and it is defined in the same way for a new task, with $\dmon$ in place of $\dmok$.

\paragraph{Approximation result.} Let $\tau_k^2:=\sum_{i>r_k}(\sigma_i^{(k)})^2$ be the tail of the singular values of $\dmok$. The next proposition shows that the risk~\eqref{eq:risk} controls the quality of the representation on the training tasks.

\begin{prp}\label{prp:approx}
Let $h=(\theta,\{\mbA^{(k)},\mbB^{(k)}\}_k)\in\Theta\times\mcW_\rho^K$ and $k\in[K]$.
\begin{enumerate}
\item[(i)] \emph{Best head.} $\eta_k^2(\theta)=\norm[\mathrm{HS}]{\dmok-P^{(k)}_{\Phi_\theta}\dmok P^{(k)}_{\Psi_\theta}}^2$, and the minimum in~\eqref{eq:eta} is attained at the T-CMO solution $\mbM^*_k=\mbG_{\Phi_\theta}^{(k)+}\mbC^{(k)}_{\Phi_\theta\Psi_\theta}\mbG_{\Psi_\theta}^{(k)+}$.
\item[(ii)] \emph{Risk controls the representation.} Let $T:=\dmokt$, so that $\mcE_k(h)=\norm[\mathrm{HS}]{T-\dmok}^2$ by~\eqref{eq:risk}. Since $T$ has rank at most $r_k$,
\[
\tau_k^2\ \leq\ \mcE_k(h),\qquad \eta_k^2(\theta)\ \leq\ \mcE_k(h).
\]
In particular, $\mcA_K\geq\frac1K\sum_k\tau_k^2$, with equality if (A2) is realizable, i.e.\ if some $h\in\Theta\times\mcW_\rho^K$ satisfies that $q_{h_k}$ is the kernel of $\dmokr$ for all $k$.
\end{enumerate}
\end{prp}

Here $^+$ denotes the Moore--Penrose pseudo-inverse, which coincides with the inverse when the Gram matrices are invertible; it accounts for possibly redundant dictionaries. Item (i) shows that $\eta_k(\theta)$ is the population error of T-CMO with the frozen representation $\theta$, and item (ii) that the risk of any trained model bounds the quality of its representation.

\paragraph{New tasks.} A new task shares the representation of the training tasks if its operator factorizes through the same functions. We thus consider (A2) in the following form: there exist $\Phi^\star\colon\mcX\to\mdR^{d^\star}$ and $\Psi^\star\colon\mcY\to\mdR^{d^\star}$ such that every task, training or new, satisfies
\begin{equation}\label{eq:shared_repr}
g_k(x,y)=\Phi^\star_{c,k}(x)^\top\mbM_k\,\Psi^\star_{c,k}(y)\qquad\text{for some }\mbM_k\in\mdR^{d^\star\times d^\star},
\end{equation}
where $\Phi^\star_{c,k},\Psi^\star_{c,k}$ are centered with respect to $\mu_k,\nu_k$, and $k=\mathrm{new}$ for the new task. Let $\mbG^{(k)}_{\Phi^\star}$ and $\mbG^{(k)}_{\Psi^\star}$ be the corresponding covariance matrices. Transferring to any new task of the form~\eqref{eq:shared_repr} requires the training tasks to use all shared directions, which we quantify by
\[
\lambda_\Phi:=\lambda_{\min}\Big(\frac1K\sum_{k=1}^K\mbM_k\mbG^{(k)}_{\Psi^\star}\mbM_k^\top\Big),\qquad
\lambda_\Psi:=\lambda_{\min}\Big(\frac1K\sum_{k=1}^K\mbM_k^\top\mbG^{(k)}_{\Phi^\star}\mbM_k\Big).
\]
These quantities are the operator counterpart of the task diversity condition used in multi-task representation learning, where the smallest singular value of the matrix collecting the source heads is required to be bounded away from zero \citep{tripuraneni2020theory,du2021fewshot}; here the heads are matrices and each direction is weighted by the covariance of the shared output features. For instance, $\lambda_\Phi,\lambda_\Psi>0$ as soon as one training task has an invertible head and non-degenerate covariance matrices. For the new task, we let $\sigma_\Phi:=\norm[\mathrm{op}]{\mbM_{\mathrm{new}}\mbG^{(\mathrm{new})}_{\Psi^\star}\mbM_{\mathrm{new}}^\top}$ and $\sigma_\Psi:=\norm[\mathrm{op}]{\mbM_{\mathrm{new}}^\top\mbG^{(\mathrm{new})}_{\Phi^\star}\mbM_{\mathrm{new}}}$. Finally, the representation is learned where the training inputs lie, which we account for with $\beta_\mu:=\max_k\norm[\infty]{\rd\mu_{\mathrm{new}}/\rd\mu_k}$ and $\beta_\nu:=\max_k\norm[\infty]{\rd\nu_{\mathrm{new}}/\rd\nu_k}$; $\beta_\mu=\beta_\nu=1$ when the marginals coincide.

\begin{prp}\label{prp:new_task}
Assume~\eqref{eq:shared_repr} with $\lambda_\Phi,\lambda_\Psi>0$. For every $\theta\in\Theta$,
\begin{equation}\label{eq:new_task}
\eta_{\mathrm{new}}^2(\theta)\ \leq\ C_{\mathrm{new}}\cdot\frac1K\sum_{k=1}^K\eta_k^2(\theta),\qquad C_{\mathrm{new}}:=\frac{\beta_\mu\sigma_\Phi}{\lambda_\Phi}+\frac{\beta_\nu\sigma_\Psi}{\lambda_\Psi}.
\end{equation}
\end{prp}

A representation that is accurate on average over the training tasks is thus accurate for every new task sharing it, and since $\eta_{\mathrm{new}}(\theta)$ is the population error of T-CMO (\cref{prp:approx}(i)), \cref{prp:new_task} quantifies transfer. The constant $C_{\mathrm{new}}$ is small when the training tasks are diverse (large $\lambda_\Phi,\lambda_\Psi$). Both ingredients are necessary: a shared direction that no training task uses cannot be learned, and the learned features are not constrained where no training input lies.

By \cref{prp:approx,prp:new_task}, all these guarantees follow from a bound on the average risk $\overline{\mcE}(\widehat h)$ of the trained model.

\subsection{Statistical rate}
\label{sec:theory_rate}

We measure the complexity of the shared dictionaries on the pooled sample of all tasks. For $m\geq1$, let $X^{(k)}_l\sim\mu_k$ ($k\in[K]$, $l\in[m]$) be independent and $\gamma_{klj}$ i.i.d.\ standard Gaussian variables, and define
\begin{equation}\label{eq:gauss}
\mathfrak G_{m,K}(\mcF_\Phi):=\frac{1}{mK}\,\mdE\sup_{\theta\in\Theta}\sum_{k=1}^K\sum_{l=1}^m\sum_{j=1}^d\gamma_{klj}\,\phi^\theta_j\big(X^{(k)}_l\big),
\end{equation}
and $\mathfrak G_{m,K}(\mcF_\Psi)$ in the same way with $Y^{(k)}_l\sim\nu_k$. For norm-bounded neural networks, $\mathfrak G_{m,K}(\mcF)\leq\mathrm{Comp}(\mcF)/\sqrt{mK}$ with $\mathrm{Comp}(\mcF)$ independent of $m$ and $K$ \citep{golowich2018size}. We use Gaussian rather than Rademacher complexities because the separation between the shared representation and the task-specific heads relies on the chain rule of \citet{maurer2016benefit}, which is based on Gaussian comparison inequalities; the two complexities are equivalent up to a factor $\sqrt{\log(mKd)}$ \citep[Eq.~(4.9)]{ledoux1991probability}. We set
\[
L_\rho:=\rho c\sqrt d\,(1+Q_\rho),\qquad B_\rho:=4Q_\rho(Q_\rho+3).
\]

\begin{thm}\label{thm:rate}
Let \cref{ass:bounded} hold and $n\geq2$ be even. There exists an absolute constant $C>0$ such that, for any $\delta\in(0,1)$, with probability at least $1-\delta$,
\begin{equation}\label{eq:rate}
\overline{\mcE}(\widehat h)\ \leq\ \mcA_K+\epsilon_{\mathrm{opt}}
 +C\,L_\rho\Big[\underbrace{\mathfrak G_{n/2,K}(\mcF_\Phi)+\mathfrak G_{n/2,K}(\mcF_\Psi)}_{\text{shared representation}}+\underbrace{c\sqrt{d/n}}_{\text{heads}}\Big]
+2B_\rho\sqrt{\frac{\log(2\delta^{-1})}{2nK}} .
\end{equation}
\end{thm}

For neural networks, the statistical error is of order
\[
L_\rho\,\frac{\mathrm{Comp}(\mcF_\Phi)+\mathrm{Comp}(\mcF_\Psi)}{\sqrt{nK}}+L_\rho\,c\sqrt{\frac dn}.
\]
The complexity of the networks, which is typically the dominant term, is paid with the $nK$ samples of all tasks, whereas each head only costs $\sqrt{d/n}$. Learning a single task ($K=1$) gives $\mathrm{Comp}(\mcF)/\sqrt n$: sharing the representation across $K$ tasks divides the dominant term by $\sqrt K$. This is the operator counterpart of the benefit of multi-task representation learning of \citet{maurer2016benefit}. The price for sharing is the approximation error $\mcA_K$, which is small only when the tasks share their dominant singular spaces, i.e.\ under (A3). As in \citet{kostic2024ncp}, the optimization error is not analyzed.

Combining \cref{thm:rate} with \cref{prp:approx,prp:new_task} gives the transfer guarantee: with probability at least $1-\delta$,
\[
\eta_{\mathrm{new}}^2(\widehat\theta)\leq C_{\mathrm{new}}\cdot\mathrm{RHS}\eqref{eq:rate},
\]
where $\mathrm{RHS}\eqref{eq:rate}$ is the right-hand side of~\eqref{eq:rate}. The error of T-CMO on a new task thus decreases with the total number of training samples $nK$, up to the approximation terms.

\subsection{Proof of \cref{prp:approx,prp:new_task}}
 
\begin{lmm}\label{lmm:projection}
Let $\msS_1\colon\mdR^d\to\lxk$ and $\msS_2\colon\mdR^d\to\lyk$ be linear, with $P_1,P_2$ the orthogonal projections onto their ranges. For every Hilbert--Schmidt operator $\msD$,
$\min_{\mbM}\norm[\mathrm{HS}]{\msD-\msS_1\mbM\msS_2^*}=\norm[\mathrm{HS}]{\msD-P_1\msD P_2}$, attained at $\mbM^*=(\msS_1^*\msS_1)^+\msS_1^*\msD\msS_2(\msS_2^*\msS_2)^+$. Moreover,
\[
\norm[\mathrm{HS}]{\msD-P_1\msD P_2}^2=\norm[\mathrm{HS}]{(I-P_1)\msD}^2+\norm[\mathrm{HS}]{P_1\msD(I-P_2)}^2=\norm[\mathrm{HS}]{\msD(I-P_2)}^2+\norm[\mathrm{HS}]{(I-P_1)\msD P_2}^2 .
\]
\end{lmm}
\begin{proof}
Since $P_i=\msS_i(\msS_i^*\msS_i)^+\msS_i^*$, the set $\{\msS_1\mbM\msS_2^*\}$ equals $\{T:T=P_1TP_2\}$. The map $T\mapsto P_1TP_2$ is idempotent and self-adjoint for the Hilbert--Schmidt inner product, hence it is the orthogonal projection onto this set, and the closest point to $\msD$ is $P_1\msD P_2$, which corresponds to $\mbM^*$. The last identities follow from $\msD-P_1\msD P_2=(I-P_1)\msD+P_1\msD(I-P_2)=\msD(I-P_2)+(I-P_1)\msD P_2$, where in both cases the two terms are orthogonal for the Hilbert--Schmidt inner product.
\end{proof}
 
\begin{proof}[Proof of \cref{prp:approx}]
(i) Apply \cref{lmm:projection} with $\msS_1=\msS_{\Phic}$ and $\msS_2=\msS_{\Psic}$, for which $\msS_1^*\msS_1=\mbG^{(k)}_{\Phi_\theta}$, $\msS_2^*\msS_2=\mbG^{(k)}_{\Psi_\theta}$ and $\msS_1^*\dmok\msS_2=\mbC^{(k)}_{\Phi_\theta\Psi_\theta}$.
 
(ii) Let $T:=\dmokt$, so that $\mcE_k(h)=\norm[\mathrm{HS}]{T-\dmok}^2$ by~\eqref{eq:risk}. Since $T$ has rank at most $r_k$, the Eckart--Young--Mirsky theorem gives $\mcE_k(h)\geq\tau_k^2$, with equality when $q_{h_k}$ is the kernel of $\dmokr$. Moreover, let $\Pi$ denote the projection onto the orthogonal of the constant functions (in $\lxk$ or $\lyk$). Since $\dmok=\Pi\dmok\Pi$, we have $\norm[\mathrm{HS}]{T-\dmok}\geq\norm[\mathrm{HS}]{\Pi T\Pi-\dmok}$, and $\Pi T\Pi=\msS_{\Phic}\mbA^{(k)}\mbB^{(k)\top}\msS_{\Psic}^*$ is of the form in~\eqref{eq:eta}. Hence $\eta_k^2(\theta)\leq\mcE_k(h)$.
\end{proof}

\begin{proof}[Proof of \cref{prp:new_task}]

Fix $\theta$. For a task $k$ (training or new), let $P_k$ and $P'_k$ be the projections of \cref{prp:approx} in $\lxk$ and $\lyk$, and define the positive semi-definite matrices
\[
\mbE_k:=\msS_{\Phi^\star_{c,k}}^*(I-P_k)\msS_{\Phi^\star_{c,k}},\qquad \mbF_k:=\msS_{\Psi^\star_{c,k}}^*(I-P'_k)\msS_{\Psi^\star_{c,k}}\in\mdR^{d^\star\times d^\star},
\]
so that $\mba^\top\mbE_k\mba=\norm[\lxk]{(I-P_k)\,\mba^\top\Phi^\star_{c,k}}^2$ is the residual of the regression onto the learned dictionary of $\mba^\top\Phi^\star_{c,k}$, the centered version under $\mu_k$ of the shared function $\mba^\top\Phi^\star$, and similarly for $\mbF_k$. Since the centered features $\Phi^{(k)}_{\theta,c}$ are orthogonal to constants, $P_k$ coincides on centered functions with the projection onto $\spn\{\Phi_\theta,1\}$; as $\mba^\top\Phi^\star$ and $\mba^\top\Phi^\star_{c,k}$ differ by a constant,
\[
\mba^\top\mbE_k\mba=\mathrm{dist}_{\lxk}\big(\mba^\top\Phi^\star,\,\spn\{\Phi_\theta,1\}\big)^2,
\]
where neither the function nor the subspace depends on $k$; only the norm does. The same holds for $\mbF_k$.

\emph{Step 1 (training tasks).} By~\eqref{eq:shared_repr}, $\dmok=\msS_{\Phi^\star_{c,k}}\mbM_k\msS_{\Psi^\star_{c,k}}^*$, so that
\[
\norm[\mathrm{HS}]{(I-P_k)\dmok}^2=\Tr\big(\mbE_k\,\mbM_k\mbG^{(k)}_{\Psi^\star}\mbM_k^\top\big),\qquad \norm[\mathrm{HS}]{\dmok(I-P'_k)}^2=\Tr\big(\mbF_k\,\mbM_k^\top\mbG^{(k)}_{\Phi^\star}\mbM_k\big),
\]
and both are at most $\eta_k^2(\theta)$ by \cref{lmm:projection} and \cref{prp:approx}(i).
 
\emph{Step 2 (change of task).} For every $f$, the projection of $f$ in $\lxk$ belongs to $\spn\{\Phi_\theta,1\}$, so its distance to this space in $L^2_{\mu_{\mathrm{new}}}$ is at most $\sqrt{\beta_\mu}$ times its distance in $\lxk$. Hence $\mbE_{\mathrm{new}}\preceq\beta_\mu\mbE_k$ for every $k$, and similarly $\mbF_{\mathrm{new}}\preceq\beta_\nu\mbF_k$. Using the definition of $\lambda_\Phi$ and Step 1,
\[
\lambda_\Phi\Tr(\mbE_{\mathrm{new}})\leq\frac1K\sum_k\Tr\big(\mbE_{\mathrm{new}}\mbM_k\mbG^{(k)}_{\Psi^\star}\mbM_k^\top\big)\leq\frac{\beta_\mu}K\sum_k\Tr\big(\mbE_k\mbM_k\mbG^{(k)}_{\Psi^\star}\mbM_k^\top\big)\leq\frac{\beta_\mu}K\sum_k\eta_k^2(\theta),
\]
and similarly $\lambda_\Psi\Tr(\mbF_{\mathrm{new}})\leq\frac{\beta_\nu}K\sum_k\eta_k^2(\theta)$.

\emph{Step 3 (new task).} By \cref{prp:approx}(i) and \cref{lmm:projection},
$\eta^2_{\mathrm{new}}(\theta)\leq\norm[\mathrm{HS}]{(I-P_{\mathrm{new}})\dmon}^2+\norm[\mathrm{HS}]{\dmon(I-P'_{\mathrm{new}})}^2$, and, as in Step 1,
\[
\norm[\mathrm{HS}]{(I-P_{\mathrm{new}})\dmon}^2=\Tr\big(\mbE_{\mathrm{new}}\mbM_{\mathrm{new}}\mbG^{(\mathrm{new})}_{\Psi^\star}\mbM_{\mathrm{new}}^\top\big)\leq\sigma_\Phi\Tr(\mbE_{\mathrm{new}}),
\]
and $\norm[\mathrm{HS}]{\dmon(I-P'_{\mathrm{new}})}^2\leq\sigma_\Psi\Tr(\mbF_{\mathrm{new}})$. We use Step 2 to conclude.

\end{proof}

\subsection{Proof of \cref{thm:rate}}

We use two tools. The first is the chain rule of \citet{maurer2016benefit}, which separates the complexity of a shared representation from the one of the task-specific functions.

\begin{lmm}[\citealp{maurer2016benefit}, Theorem~13]\label{lmm:maurer}
Let $\mcH$ be a class of maps $\mcZ\to\mdR^D$ with $0\in\mcH$ and with values in the ball of radius $R$, and let $\mcG$ be a class of functions $\mdR^D\to\mdR$ with Lipschitz constant at most $L$ on that ball and such that $g(0)=0$ for all $g\in\mcG$. For $k\in[K]$, let $\zeta_{k1},\ldots,\zeta_{km}$ be i.i.d.\ from $P_k$, independent across $k$, and write $\bar\zeta:=(\zeta_{kl})_{k,l}$,
\[
G\big(\mcH(\bar\zeta)\big):=\mdE_\gamma\sup_{h\in\mcH}\sum_{k,l}\innerp{\gamma_{kl}}{h(\zeta_{kl})},\qquad
Q(\mcG):=\sup_{\mby\neq\mby'\in(\mdR^D)^m}\frac{\mdE\sup_{g\in\mcG}\sum_{l=1}^m\gamma_l\big(g(y_l)-g(y'_l)\big)}{\norm{\mby-\mby'}},
\]
where the $\gamma_{kl}$ are i.i.d.\ standard Gaussian vectors of $\mdR^D$ and the $\gamma_l$ are i.i.d.\ standard Gaussian variables. Then there exist absolute constants $c_1,c_2>0$ such that
\begin{align*}
&\mdE\sup_{h\in\mcH,\,g_1,\ldots,g_K\in\mcG}\frac1K\sum_{k=1}^K\Big(\mdE_{\zeta\sim P_k}\big[g_k(h(\zeta))\big]-\frac1m\sum_{l=1}^mg_k\big(h(\zeta_{kl})\big)\Big)\\
&\hspace{6cm}\leq\frac{c_1L\,\mdE\,G\big(\mcH(\bar\zeta)\big)}{mK}+\frac{c_2\,Q(\mcG)\,\sup_{h\in\mcH}\norm{h(\bar\zeta)}}{m\sqrt K},
\end{align*}
where $\norm{h(\bar\zeta)}$ denotes the Euclidean norm of the vector $\big(h(\zeta_{kl})\big)_{k,l}$.
\end{lmm}

\begin{rmk}\label{rmk:maurer_zero}
The assumptions $0\in\mcH$ and $g(0)=0$ make the term $G(\mcG(y_0))$ of the chain rule of \citet[Theorem~12]{maurer2016benefit} vanish at $y_0=0$, which is how \cref{lmm:maurer} is stated. They are not essential (see Remark~1 of \citealp{maurer2016benefit}): for general classes one simply adds $\inf_{y_0\in\mcH(\bar\zeta)}\mdE\,G(\mcG(y_0))/(mK)$ to the bound. In the proof of \cref{thm:rate} below, $g(0)=0$ holds for every head, and $0\in\mcH$ amounts to assuming that $\Phi_\theta\equiv0$ and $\Psi_\theta\equiv0$ for some $\theta\in\Theta$.
\end{rmk}

The following elementary lemma collects the two quantities that \cref{lmm:maurer} requires for the bilinear heads of \method.

\begin{lmm}\label{lmm:heads}
Let $r,\rho>0$ and $B_r:=\{\mba\in\mdR^d:\norm{\mba}\leq r\}$. For $\norm[\mathrm F]{\mbW}\leq\rho$ set $p_{\mbW}(\mba,\mbb):=\mba^\top\mbW\mbb$, so that $\abs{p_{\mbW}}\leq\rho r^2$ on $B_r\times B_r$. On $B_r\times B_r$, the classes
\[
\mcG_{\mathrm{lin}}:=\big\{p_{\mbW}:\norm[\mathrm F]{\mbW}\leq\rho\big\},\qquad
\mcG_{\mathrm{sq}}:=\big\{p_{\mbW}^2:\norm[\mathrm F]{\mbW}\leq\rho\big\}
\]
have Lipschitz constant and $Q$ (as defined in \cref{lmm:maurer}) at most $\sqrt2\rho r$ and $2\sqrt2\rho^2r^3$, respectively. Moreover, if a class $\mcG$ on $B_r^{4}$ is obtained by averaging two members of such a class over disjoint pairs of coordinates, i.e.\ $g(\mba,\mbb,\mba',\mbb')=\tfrac12\big(\phi(\mba,\mbb)+\phi(\mba',\mbb')\big)$ with $\phi$ ranging over the class, then the same bounds hold for $\mcG$.
\end{lmm}
\begin{proof}
\emph{Lipschitz constants.} For $\norm{\mba},\norm{\mbb},\norm{\tilde\mba},\norm{\tilde\mbb}\leq r$,
\[
\abs{p_{\mbW}(\mba,\mbb)-p_{\mbW}(\tilde\mba,\tilde\mbb)}\leq\abs{(\mba-\tilde\mba)^\top\mbW\mbb}+\abs{\tilde\mba^\top\mbW(\mbb-\tilde\mbb)}\leq\rho r\big(\norm{\mba-\tilde\mba}+\norm{\mbb-\tilde\mbb}\big)\leq\sqrt2\rho r\norm{(\mba,\mbb)-(\tilde\mba,\tilde\mbb)},
\]
using $\norm[\mathrm{op}]{\mbW}\leq\norm[\mathrm F]{\mbW}$ and Cauchy--Schwarz. Since $\abs{p_{\mbW}}\leq\rho r^2$, we get $\abs{p_{\mbW}^2-p_{\mbW}^2(\tilde\cdot)}\leq2\rho r^2\abs{p_{\mbW}-p_{\mbW}(\tilde\cdot)}\leq2\sqrt2\rho^2r^3\norm{(\mba,\mbb)-(\tilde\mba,\tilde\mbb)}$.

\emph{The quantities $Q$.} Both classes are linear in a matrix parameter:
\[
p_{\mbW}(\mba,\mbb)=\innerp{\mbW}{\mba\mbb^\top},\qquad p_{\mbW}(\mba,\mbb)^2=\innerp{\mbW\otimes\mbW}{(\mba\mbb^\top)\otimes(\mba\mbb^\top)},
\]
with $\norm[\mathrm F]{\mbW\otimes\mbW}=\norm[\mathrm F]{\mbW}^2\leq\rho^2$. For any matrices $\mbM_1,\ldots,\mbM_m$ and any $v>0$, Jensen's inequality gives
\begin{equation}\label{eq:linear_gauss}
\mdE\sup_{\norm[\mathrm F]{\mbV}\leq v}\sum_{l=1}^m\gamma_l\innerp{\mbV}{\mbM_l}\leq v\,\mdE\norm[\mathrm F]{\textstyle\sum_l\gamma_l\mbM_l}\leq v\Big(\sum_{l=1}^m\norm[\mathrm F]{\mbM_l}^2\Big)^{1/2}.
\end{equation}
Applying~\eqref{eq:linear_gauss} with $\mbM_l=\mba_l\mbb_l^\top-\tilde\mba_l\tilde\mbb_l^\top$ and $v=\rho$, and using
$\norm[\mathrm F]{\mba\mbb^\top-\tilde\mba\tilde\mbb^\top}\leq r\big(\norm{\mba-\tilde\mba}+\norm{\mbb-\tilde\mbb}\big)\leq\sqrt2r\norm{(\mba,\mbb)-(\tilde\mba,\tilde\mbb)}$,
gives $Q(\mcG_{\mathrm{lin}})\leq\sqrt2\rho r$. For $\mcG_{\mathrm{sq}}$, write $\mbu:=\mba\otimes\mbb$, so that $\norm{\mbu}\leq r^2$ and $\norm{\mbu-\tilde\mbu}\leq\sqrt2r\norm{(\mba,\mbb)-(\tilde\mba,\tilde\mbb)}$; then
\[
\norm[\mathrm F]{\mbu\mbu^\top-\tilde\mbu\tilde\mbu^\top}\leq\big(\norm{\mbu}+\norm{\tilde\mbu}\big)\norm{\mbu-\tilde\mbu}\leq2\sqrt2r^3\norm{(\mba,\mbb)-(\tilde\mba,\tilde\mbb)},
\]
and~\eqref{eq:linear_gauss} with $v=\rho^2$ gives $Q(\mcG_{\mathrm{sq}})\leq2\sqrt2\rho^2r^3$.

\emph{Averaging two blocks.} If $\phi$ has Lipschitz constant $L_0$ and the class of such $\phi$ has $Q\leq Q_0$, then, writing $\mbu=(\mbu_1,\mbu_2)$ for the two blocks, $\abs{g(\mbu)-g(\tilde\mbu)}\leq\frac{L_0}{2}(\norm{\mbu_1-\tilde\mbu_1}+\norm{\mbu_2-\tilde\mbu_2})\leq\frac{L_0}{\sqrt2}\norm{\mbu-\tilde\mbu}$, and likewise
$\mdE\sup_g\sum_l\gamma_l\big(g(\mbu_l)-g(\tilde\mbu_l)\big)\leq\frac{Q_0}{2}\big(\norm{\mbu^{(1)}-\tilde\mbu^{(1)}}+\norm{\mbu^{(2)}-\tilde\mbu^{(2)}}\big)\leq\frac{Q_0}{\sqrt2}\norm{\mbu-\tilde\mbu}$,
where $\mbu^{(j)}$ collects the $j$-th blocks. Both constants therefore only decrease.
\end{proof}

The last tool is Hoeffding's representation of a U-statistic as an average of averages of i.i.d.\ terms \citep[see, e.g.,][Chapter~4]{delapena1999decoupling}. Let $Z_1,\ldots,Z_n$ be i.i.d., let $s$ be a symmetric kernel, and let $\mdS_n$ be the set of permutations of $[n]$. Then
\begin{equation}\label{eq:hoeffding_rep}
\frac{1}{n(n-1)}\sum_{i\neq j}s\big(Z_i,Z_j\big)=\frac{1}{\abs{\mdS_n}}\sum_{\pi\in\mdS_n}\frac2n\sum_{l=1}^{n/2}s\big(Z_{\pi(2l-1)},Z_{\pi(2l)}\big),
\end{equation}
where, for each fixed $\pi$, the $n/2$ terms of the inner average are i.i.d.

\begin{proof}[Proof of \cref{thm:rate}]

\emph{Step 1 (reduction to a uniform deviation).} Let $\bar h$ be a fixed element of $\Theta\times\mcW_\rho^K$, and set
\[
\widehat{\overline{\mcL}}:=\frac1K\sum_{k=1}^K\widehat{\mcL}^{(k)},\qquad
\overline{\mcL}:=\frac1K\sum_{k=1}^K\mcL^{(k)},\qquad
Z:=\sup_{h\in\Theta\times\mcW_\rho^K}\big(\overline{\mcL}(h)-\widehat{\overline{\mcL}}(h)\big).
\]
Since $\mcE_k$ and $\mcL^{(k)}$ differ by a quantity that does not depend on $h$, \eqref{eq:erm} gives
\[
\overline{\mcE}(\widehat h)-\overline{\mcE}(\bar h)=\overline{\mcL}(\widehat h)-\overline{\mcL}(\bar h)\leq Z+\epsilon_{\mathrm{opt}}+\big(\widehat{\overline{\mcL}}(\bar h)-\overline{\mcL}(\bar h)\big).
\]
It remains to bound $Z$ and the deviation at the fixed point $\bar h$.

\emph{Step 2 (Concentration argument).} With $q=q_{h_k}$ and $q_{ij}=q\big(x_i^{(k)},y_j^{(k)}\big)$, the empirical loss of \cref{sec:method} reads
\begin{equation}\label{eq:loss_expansion}
\widehat{\mcL}^{(k)}(h)=\frac{1}{n(n-1)}\sum_{i\neq j}\big(q_{ij}^2+2q_{ij}\big)-\frac2n\sum_{i=1}^nq_{ii},
\end{equation}
whose expectation is
\[
\mdE_{\mu_k\otimes\nu_k}[q^2]+2\mdE_{\mu_k\otimes\nu_k}[q]-2\mdE_{\rho_k}[q]=\norm[\Lk]{q-g_k}^2-\norm[\Lk]{g_k}^2=\mcL^{(k)}(h),
\]
where we used $\mdE_{\rho_k}[q]=\mdE_{\mu_k\otimes\nu_k}[p_kq]$. By \cref{ass:bounded}, $\abs{q_{h_k}}\leq Q_\rho$ for every $h$ and every $(x,y)$, so replacing one sample of task $k$ modifies at most $2(n-1)$ terms of the U-statistic in~\eqref{eq:loss_expansion}, each by at most $2(Q_\rho^2+2Q_\rho)$, and one diagonal term, by at most $2Q_\rho$; hence it modifies $\widehat{\overline{\mcL}}(h)$ by at most $B_\rho/(nK)$, uniformly in $h$. Both $Z$ and $\widehat{\overline{\mcL}}(\bar h)$ are therefore functions of the $nK$ independent samples with bounded differences $B_\rho/(nK)$, and McDiarmid's inequality together with a union bound gives that, with probability at least $1-\delta$,
\[
Z\leq\mdE Z+B_\rho\sqrt{\frac{\log(2\delta^{-1})}{2nK}}
\qquad\text{and}\qquad
\widehat{\overline{\mcL}}(\bar h)-\overline{\mcL}(\bar h)\leq B_\rho\sqrt{\frac{\log(2\delta^{-1})}{2nK}}.
\]

\emph{Step 3 (splitting the empirical process).} By~\eqref{eq:loss_expansion}, $\overline{\mcL}(h)-\widehat{\overline{\mcL}}(h)=\Delta_U(h)+2\Delta_V(h)$, where
\[
\Delta_U(h):=\frac1K\sum_{k=1}^K\big(\mdE U^{(k)}_h-U^{(k)}_h\big),\qquad
\Delta_V(h):=\frac1K\sum_{k=1}^K\big(V^{(k)}_h-\mdE V^{(k)}_h\big),
\]
with $V^{(k)}_h:=\frac1n\sum_iq\big(x^{(k)}_i,y^{(k)}_i\big)$ and $U^{(k)}_h$ the U-statistic of task $k$ with symmetric kernel
\[
s_h\big((x,y),(x',y')\big):=\tfrac12\Big(q(x,y')^2+2q(x,y')+q(x',y)^2+2q(x',y)\Big).
\]
Hence $\mdE Z\leq\mdE\sup_h\Delta_U(h)+2\,\mdE\sup_h\Delta_V(h)$, and we bound the two suprema separately.

\emph{Step 4 (the U-part).} By~\eqref{eq:hoeffding_rep} applied to each task, $\Delta_U(h)$ is, for every $h$, an average over the $K$-tuples of permutations $(\pi_1,\ldots,\pi_K)$ of
\[
\frac{2}{nK}\sum_{k=1}^K\sum_{l=1}^{n/2}\Big(\mdE\big[s_h(\zeta^{\pi_k}_{kl})\big]-s_h(\zeta^{\pi_k}_{kl})\Big),\qquad
\zeta^{\pi}_{kl}:=\big(Z^{(k)}_{\pi(2l-1)},Z^{(k)}_{\pi(2l)}\big).
\]
The supremum over $h$ of an average is at most the average of the suprema, and each term of the latter has, by exchangeability within tasks and independence across tasks, the same distribution as the term obtained with the identity permutations. Hence $\mdE\sup_h\Delta_U(h)$ is at most the expected supremum of the same process over the $nK/2$ independent blocks $\zeta_{kl}:=\big(Z^{(k)}_{2l-1},Z^{(k)}_{2l}\big)$, $l\in[n/2]$, to which we apply \cref{lmm:maurer} with $m=n/2$, the shared representation
\[
h_\theta(\zeta):=\big(\Phi_\theta(x),\Psi_\theta(y'),\Phi_\theta(x'),\Psi_\theta(y)\big)\in\mdR^{4d},\qquad \zeta=\big((x,y),(x',y')\big),
\]
and the task-specific functions
\[
g_{\mbW}(\mba,\mbb,\mba',\mbb'):=\tfrac12\Big(\varphi(\mba^\top\mbW\mbb)+\varphi(\mba'^\top\mbW\mbb')\Big),\qquad \varphi(t):=t^2+2t,\qquad \norm[\mathrm F]{\mbW}\leq\rho,
\]
so that $s_h=g_{\mbW_k}\circ h_\theta$ with $\mbW_k=\mbA^{(k)}\mbB^{(k)\top}$ and $g_{\mbW}(0)=0$. It remains to bound the three quantities entering \cref{lmm:maurer}.

\emph{(a) Range.} By \cref{ass:bounded}, $h_\theta$ takes values in the ball of radius $R=2c\sqrt d$, and each of the four blocks of $h_\theta(\zeta)$ lies in the ball of radius $r:=c\sqrt d$ of $\mdR^d$. Consequently $\sup_h\norm{h(\bar\zeta)}\leq R\sqrt{nK/2}$.

\emph{(b) Lipschitz constant and $Q$.} Since $\varphi(t)=t^2+2t$, \cref{lmm:heads} with $r=c\sqrt d$ gives, for the Lipschitz constant of $g_{\mbW}$ and for $Q(\mcG)$, the common bound
\[
2\sqrt2\rho^2c^3d^{3/2}+2\cdot\sqrt2\rho c\sqrt d=2\sqrt2\,\rho c\sqrt d\,(1+Q_\rho)=2\sqrt2\,L_\rho .
\]

\emph{(c) Gaussian average.} Splitting the Gaussian vector $\gamma_{kl}\in\mdR^{4d}$ into its four blocks, $G(\mcH(\bar\zeta))$ is at most the sum of four Gaussian averages, two of them over the $nK/2$ inputs $x$ drawn from the $\mu_k$ and two over the $nK/2$ outputs $y$ drawn from the $\nu_k$, $n/2$ per task. By the definition~\eqref{eq:gauss},
\[
\mdE\,G\big(\mcH(\bar\zeta)\big)\leq2\cdot\frac{nK}{2}\Big(\mathfrak G_{n/2,K}(\mcF_\Phi)+\mathfrak G_{n/2,K}(\mcF_\Psi)\Big).
\]

Substituting (a), (b) and (c) into \cref{lmm:maurer} with $m=n/2$ yields
\[
\mdE\sup_h\Delta_U(h)\leq4\sqrt2\,c_1L_\rho\Big(\mathfrak G_{n/2,K}(\mcF_\Phi)+\mathfrak G_{n/2,K}(\mcF_\Psi)\Big)+8c_2L_\rho\,c\sqrt{\frac dn}.
\]

\emph{Step 5 (the V-part).} Here the summands are already i.i.d.\ within each task, so we apply \cref{lmm:maurer} directly with $m=n$, the representation $h_\theta(x,y)=\big(\Phi_\theta(x),\Psi_\theta(y)\big)\in\mdR^{2d}$, which by \cref{ass:bounded} takes values in the ball of radius $R=\sqrt2c\sqrt d$, and the task functions $g_{\mbW}(\mba,\mbb)=-\mba^\top\mbW\mbb$, for which $g_{\mbW}(0)=0$ and, by \cref{lmm:heads} with $r=c\sqrt d$, both the Lipschitz constant and $Q$ are at most $\sqrt2\rho c\sqrt d\leq\sqrt2L_\rho$. As in Step 4(c), $\mdE\,G(\mcH(\bar\zeta))\leq nK\big(\mathfrak G_{n,K}(\mcF_\Phi)+\mathfrak G_{n,K}(\mcF_\Psi)\big)$ and $\sup_h\norm{h(\bar\zeta)}\leq R\sqrt{nK}$. Since a Gaussian average over a sample of size $n$ per task is at most the one over either half of it, $\mathfrak G_{n,K}\leq\mathfrak G_{n/2,K}$, and we obtain
\[
\mdE\sup_h\Delta_V(h)\leq\sqrt2\,c_1L_\rho\Big(\mathfrak G_{n/2,K}(\mcF_\Phi)+\mathfrak G_{n/2,K}(\mcF_\Psi)\Big)+2c_2L_\rho\,c\sqrt{\frac dn}.
\]

\emph{Step 6 (conclusion).} Steps 3--5 give $\mdE Z\leq C L_\rho\big[\mathfrak G_{n/2,K}(\mcF_\Phi)+\mathfrak G_{n/2,K}(\mcF_\Psi)+c\sqrt{d/n}\big]$ for an absolute constant $C$. Combining with Steps 1 and 2, and taking $\bar h$ such that $\overline{\mcE}(\bar h)$ is arbitrarily close to $\mcA_K$ (the remaining arbitrarily small term being absorbed in $\epsilon_{\mathrm{opt}}$), we obtain~\eqref{eq:rate}.

\end{proof}